\documentclass[11pt]{article}
\usepackage[hidelinks]{hyperref}
\usepackage{float}
\usepackage[version = 4]{mhchem}
\usepackage{pdfpages}
\usepackage{mathrsfs}
\usepackage{longtable}
\usepackage{listings}
\usepackage{textcomp}
\usepackage{cancel}
\usepackage{enumerate}
\usepackage{amsmath}
\usepackage{amsthm}
\usepackage{amssymb}
\usepackage{booktabs}
\usepackage{colortbl}
\usepackage{graphicx}
\usepackage{subfigure}
\usepackage{ulem}
\usepackage{siunitx}
\usepackage{appendix}
\usepackage{wrapfig}
\usepackage{geometry}
\usepackage{indentfirst}
\usepackage{blkarray}
\usepackage{multirow} 
\usepackage{amsmath}
\usepackage{amssymb}

\usepackage{natbib}
\usepackage{JI_MathCourse_Notations}

\newtheorem{cor}{Corollary}[section]
\newtheorem{defn}{Definition}[section]
\newtheorem{thm}{Theorem}[section]
\newtheorem{lemma}{Lemma}[section]
\newtheorem{prop}{Proposition}[section]

\theoremstyle{definition}
\newtheorem{remark}{Remark}[section]

\newcommand{\ta}{\Tilde{a}}
\newcommand{\tb}{\Tilde{b}}

\newcommand{\tf}{\Tilde{f}}

\newcommand{\tw}{\Tilde{w}}

\newcommand{\tx}{\Tilde{x}}

\newcommand{\tz}{\Tilde{z}}
\newcommand{\ttheta}{\Tilde{\theta}}
\newcommand{\tsigma}{\Tilde{\sigma}}

\title{\textbf{Linear Independence of Generalized Neurons and Related Functions}}
\author{Leyang Zhang}
\date{}

\begin{document}

\maketitle

\begin{abstract}
    The linear independence of neurons plays a significant role in theoretical analysis of neural networks. Specifically, given neurons $H_1, ..., H_n: \bR^N \times \bR^d \to \bR$, we are interested in the following question: when are $\{H_1(\theta_1, \cdot), ..., H_n(\theta_n, \cdot)\}$ are linearly independent as the parameters $\theta_1, ..., \theta_n$ of these functions vary over $\bR^N$. Previous works give a complete characterization of two-layer neurons without bias, for generic smooth activation functions. In this paper, we study the problem for neurons with arbitrary layers and widths, giving a simple but complete characterization for generic analytic activation functions. 
\end{abstract}

\section{Introduction}\label{Section Intro}

The study of linear independence of neurons has played an important part in not only the detailed analysis of shallow neural networks (NN) -- e.g., two-layer networks \citep{LZhangCrit, LZhangGlobal}, but also some analysis of deep NN, such as Emebdding Principle proposed by \citet{YZhangEBDD, YZhangEBDD2}. In this paper we study the following problem: 

\begin{center}
   \textit{ Given parameter space $\bR^N$ and input space $\bR^n$. For any $m \in \bN$, under what conditions of parameters $\theta_1, ..., \theta_n \in \bR^N$ are $H(\theta_1, \cdot), ..., H(\theta_m, \cdot)$ linearly dependent? }
\end{center} 

The simplest possible cases (1) \label{simplest case of linear dependence} are as follows: 
\begin{itemize}
    \item [(a)] When $H(\theta_k, \cdot)$ is constant zero for some $1 \le k \le m$. 
    \item [(b)] When $H(\theta_k, \cdot), H(\theta_j, \cdot)$ are both constant for some distinct $k, j \in \{1, ..., m\}$. 
    \item [(c)] When $H(\theta_k, \cdot) = H(\theta_j, \cdot)$ for some distinct $k, j \in \{1, ..., n\}$. 
\end{itemize}
For two-layer neurons of the form $\sigma(wx)$, $w, x \in \bR^d$, it has been shown that for generic activations, these are the \textit{only} cases in which they linearly dependent \citep{KFukumizu, BSimsek, RSun, LZhangGlobal}. However, due to the non-linearity of activation function, it becomes challenging to analyze neurons of arbitrary width and depth. In this paper, we demonstrate that the same result for two-layer neurons holds for neurons with arbitrary width and depth, given generic activation functions. \\

Traditionally, the linear independence of a set of functions is often proved using their Wronskian (provided that these functions are sufficiently smooth). This is exactly the idea behind the mainstream methods for showing linear independence of two-layer neurons without bias, see e.g. \cite{BSimsek, YIto}. Unfortunately, due to the great complexity of deep NN, this method is almost impossible to perform on neurons with more layers -- in fact, it is even not obvious how it works for two-layer neurons with bias. Thus, we use alternative ways to solve this problem. \\

In this work we investigate the linearity of neurons in both general and specific settings. For the general case, we introduce the concept of generalized neurons and generalized neural networks (Definition \ref{Defn Genealized neurons and generalized NN}) which includes usual neurons and neural networks. For such objects, we characterize their linear independence for specific activation functions based on their growth rates (Proposition \ref{Prop Functions of ordered growth I} and \ref{Prop Functions of ordered growth II}). Then we parametrically deform the activation function to deduce the generality of such characterizations, with the help of analytic bump functions (Section \ref{Subsection Analytic bump functions}). By detailly studying the minimal zero set $\calZ$ (Definition \ref{Defn Minimal zero set}) of a neural network, we show that on any bounded subset of parameter space, for ``most" analytic activations, Embedding Principle fully characterize the linear independence. For the specific case, we study two-layer neurons (with and without bias) and three-layer neurons. We identify several classes of activations, which include commonly used activations (Sigmoid, Tanh, Swich, etc.) where, again, the linear independence is fully characterized by Embedding Principle. \\

Figure \ref{Fig Overview} gives an overview of content and structure of this paper.\\

\begin{figure}\label{Fig Overview}
    \centering
    \includegraphics[width = \textwidth]{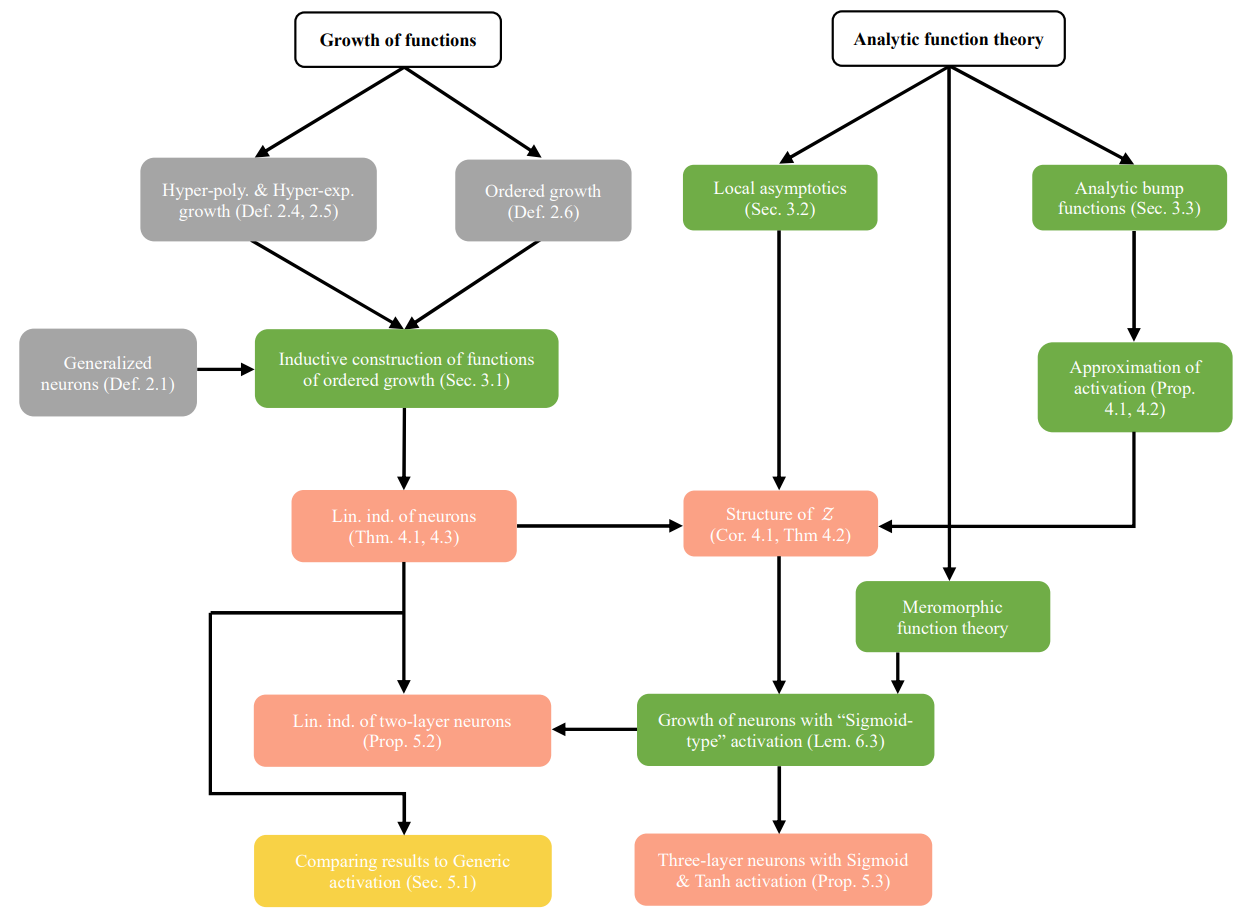}
    \caption{Overview and structure of this paper. }
    \label{Figure Overview of paper}
\end{figure}

\section{Notations and Assumptions}\label{Section Notations and assumptions}

In this section we define the notations and assumptions which will be used later. Let $\bN:= \{1, 2, 3, ...\}$ and let $\bF = \bR$ or $\bC$. We consider $\bF$-valued functions defined on $\bF^d := \Pi_{i=1}^d \bF$. Such a function $f$ is called \textit{analytic} if it is real analytic on an open subset of $\bR^d$ or complex analytic on an open subset of $\bC^d$. Given a multi-index $\alpha \in \bN^s$ for some $s \in \bN$, the $\alpha$-th derivative of $f$ is denoted by $\partial^\alpha f$ or $\partial_z^\alpha f(z)$ if we want to emphasize the variable $z$; in particular, the set of partial derivatives of $f$ can be written as 
\[
    \{\partial^\alpha f: \alpha \in \cup_{s=1}^\infty \{1, ..., d\}^s \}. 
\]
For a \textit{curve} we mean a smooth function $\gamma$ from an interval $I \siq \bR$ to $\bF^N$ for some $N \in \bN$, i.e. a smooth function $\gamma: I \to \bF^N$. We also use $\gamma$ to denote the image of this function. So for example we will simply write $\gamma \siq \bF^N$ instead of $\gamma(I) \siq \bF^N$. Then, given $I$ of the form $[\alpha, \beta]$ or $I = [\alpha, \infty)$ we define 
\[
    l_\gamma(t) := \int_\alpha^t |\gamma'(s)| ds
\]
as the length function of $\gamma$. Specifically, when $I = [\alpha, \infty)$ we write $l_\gamma(\infty) = \int_\alpha^\infty |\gamma'(s)| ds$ as the total length of $\gamma$. 
\\

Given an open set $\Omega \siq \bF$, we write $\sigma: \Omega \to \bF$ for our activation function, unless specified explicitly. When $\Omega$ is not specified, $\sigma$ is defined on the whole real line, and, whenever possible, on the whole $\bC$. For parameterized functions, we mean functions of the form $H: \Omega \times \bF^d \to \bF$, where $\Omega \siq \bF^N$ is open. Any $\theta \in \Omega$ is called a parameter of $H$, and any $z \in \bF^d$ the input of $H$. Two class of parameterized functions we are interested in are the generalized neurons and generalized neural networks (NN). 

\begin{defn}[Generalized neurons and generalized NN]\label{Defn Genealized neurons and generalized NN}
    Let $d, m, N \in \bN$ and $O \siq \bR^N$, $\Omega \siq \bF^d$ be open. Given parametrized functions $f_1, ..., f_m: O \times \Omega \to \bF^d$. A generalized neuron is a parametrized function having the form 
    \[
        \sigma\left( \sum_{k=1}^m w_k f_k(\theta', x) + b \right) 
    \]
    or 
    \[
        \partial_{\theta}^\alpha \sigma\left( \sum_{k=1}^m w_k f_k(\theta', x) + b \right), \quad \alpha \in \cup_{s=1}^\infty \bN^s
    \]
    or
    \[
        \partial_{\theta'}^\alpha \sigma\left( \sum_{k=1}^m w_k f_k(\theta', x) + b \right), \quad \alpha \in \cup_{s=1}^\infty \bN^s, 
    \]
    where its parameter is $\theta := (w_1, ..., w_m, b) \in \bF^{m+1}$ and $\theta'$, provided that $\theta$ is chosen such that it is well-defined. Any linear combination of generalized neurons is called an \textit{generalized neural networks}, or \textit{generalized NN}. 
\end{defn}
\begin{remark}
    Note that while we consider parametrized functions $f_1, ..., f_m$, the requirement actually includes the usual functions. Indeed, given any function $f: \Omega \to \bF$, it induces a trivial parametrized function $\tf:O \times \Omega \to \bF$ defined by $\tf(\theta', x) = f(x)$ for all $\theta' \in O$. In this case, $\partial_{\theta'}^\alpha \tf$'s are all constant-zero functions, which are clearly linear combinations of generalized neurons, whence generalized NN. \\
    
    The generalized NN form a large class of objects to study. In fact, as we show in the examples below, they include the ``classical" fully-connected neural networks (see Definition \ref{Defn fully-connected NN}) as well as linear combinations of the partial derivatives of them. In this paper we focus on the generalized NN of the form 
    \[
        \sigma\left( \sum_{k=1}^m w_k f_k(z) + b \right)
    \]
    and their linear combinations. Brief discussion will be made about how to extend our results to the generalized NN throughout the article. 
\end{remark}

\begin{defn}[fully-connected neural network]\label{Defn fully-connected NN}
    Given $L \in \bN$, $d=:m_0, m_1, ..., m_L=1 \in \bN$ and a function $\sigma: \bF \to \bF$, an $L$-layer fully-connected neural network with input dimension $d$, widths $\{m_l\}_{l=1}^L$ and activation $\sigma$ is a parameterized function $H: \bR^{\sum_{l=1}^L m_l(m_{l-1} + 1)} \times \bF\to \bF$ defined inductively on $l$ as follows: 
    \begin{equation}
    \begin{aligned}
        &H^{(0)}: \bR^{\sum_{l=1}^L m_l \times (m_{l-1} + 1)} \times \bF^d \to \bF^d, \,\,H^{(0)}(\theta, z) = z \\
        &\vdots \\ 
        &H^{(l)}: \bR^{\sum_{l=1}^L m_l \times (m_{l-1} + 1)} \times \bF^d \to \bF^{m_l}, \,\,H^{(l)}(\theta, z) = \left(\sigma\left(w_k^{(l)} H^{(l-1)}(\theta, z) + b_k^{(l)}\right) \right)_{k=1}^{m_1} \\
        &\vdots \\ 
        &H^{(L)}: \bR^{\sum_{l=1}^L m_l \times (m_{l-1} + 1)} \times \bF^d \to \bF, \,H^{(L)}(\theta, z) = \sum_{j=1}^{m_{L-1}} a_j H^{(L-1)}(\theta, z) + b \\ 
        &H(\theta, z) = H^{(L)}(\theta, z). 
    \end{aligned}
    \end{equation}
    Here $\theta := (\theta^{(l)})_{l=1}^L$, where 
    \begin{align*}
        \theta^{(L)} &= \left( (a_j)_{j=1}^{m_L}, b \right) \in \bR^{m_L} \times \bR \approx \bR^{m_L + 1}, \\
        \theta^{(l)} &= \left( (w_k^{(l)})_{k=1}^{m_l}, (b_k^{(l)})_{k=1}^{m_l} \right) \in \bR^{m_l m_{l-1}} \times \bR^{m_l} \approx \bR^{m_l (m_{l-1} + 1)} , \quad 1 \le l < L. 
    \end{align*}
    For simplicity, we say $H$ is a (fully-connected) NN with network structure $\{m_l\}_{l=1}^L$ and activation $\sigma$. Furthermore, we make the following definitions: 
    \begin{itemize}
        \item [(a)] We call $\{m_l\}_{l=1}^L$ the network structure of $H$, and for simplicity we denote $N := \sum_{l=1}^L m_l \times (m_{l-1} + 1)$ whenever we deal with a fully-connected NN. 

        \item [(b)] The $w_k^{(l)}$'s are often called weights, and the $b$ and $b_k^{(l)}$'s are often called biases. 
        
        \item [(b)] When $b$ and all the $b_k^{(l)}$'s are forced to be always zero, we say $H$ is an NN without bias. In this case, we will still write $\theta := (\theta^{(L)}, ..., \theta^{(1)})$, but 
        \begin{align*}
            \theta^{(L)} &= (a_j)_{j=1}^{m_L} \in \bR^{m_L}, \\ 
            \theta^{(l)} &= (w_k^{(l)})_{k=1}^{m_l} \in \bR^{m_l m_{l-1}}, \quad 1 \le l < L
        \end{align*}
        To distinguish it from a general fully-connected NN, we may sometimes call the latter one NN with bias. 
    \end{itemize}
\end{defn}
\begin{remark}\label{Rmk Comments on Defn NN}
     For each $0 \le l \le L$, $H^{(l)}$ is independent of $\theta^{(l')}$'s for all $l' > l$. Thus, as an abuse of notation we may also view it as a map $\bR^{\sum_{l'=1}^l m_{l'} \times (m_{l'-1} + 1)} \times \bF^d \to \bF^{m_l}$ but still write $H^{(l)}(\theta^{(l)}, ..., \theta^{(1)}, z) = H^{(l)}(\theta, z)$ as in the definition. Second, when we emphasize the dependence of $H$ on activation $\sigma$, we may write $H_\sigma$ or $H(\sigma, \cdot, \cdot)$, namely, given a set $\calF$ of functions from $\bF$ to $\bF$, we can view it as a map 
    \begin{equation}\label{eq for neuron structure}
        H: \calF \times \bR^{\sum_{l=1}^L m_l (m_{l-1} + 1)} \times \bF\to \bF. 
    \end{equation}
\end{remark}
~\\

The simplest non-trivial example of fully-connected NNs are perhaps two-layer ones. Using the notations in Definition \ref{Defn fully-connected NN}, such networks take the form 
\[
    H_\sigma: \bR^{(m_1 + 1) + m_1 (d+1)} \times \bF \to \bF, \quad H_\sigma(\theta) = \sum_{k=1}^{m_1} a_k \sigma\left(w_k^{(1)}z + b_k^{(1)}\right) + b, 
\]
and a two-layer NN without bias takes the form 
\[
    H_\sigma: \bR^{m_1 + m_1 d} \times \bF \to \bF, \quad H_\sigma(\theta) = \sum_{k=1}^{m_1} a_k \sigma\left(w_k^{(1)}z\right). \\
\]

Then we give some important examples to illustrate how the concept of generalized NN and classical neural networks (including fully-connected NNs) are related to one another. 
\begin{itemize}
    \item [(a)] A generalized neuron is a generalized NN.
    
    \item [(b)] A fully-connected two layer NN is a generalized NN. Indeed, given $\sigma: \bR \to \bR$ and $\theta = (w, b) \in \bR^d \times \bR = \bR^{d+1}$, a generalized neuron is an element in the set 
    \[
        \{\sigma(wx+b), \sigma^{(s)}(wx + b), \sigma^{(s)}(wx + b)\Pi_{l=1}^s x_{i_l} \}_{s=1}^\infty.  
    \]
    Therefore, using our notations in Definition \ref{Defn Genealized neurons and generalized NN}, we can set $m = d$ and $f_t(\theta', x) = x_t$, the trivially induced $t$-th coordinate map on $\bR^d$, $1 \le t \le d$. Then any two-layer NN is just 
    \[
        \sum_{k=1}^{m'} a_k \sigma\left(\sum_{t=1}^d (w_k)_t f_t(\theta', x) + b_k \right)
    \]
    for some $(w_1, b_1), ..., (w_m, b_m) \in \bR^{d+1}$. 
    \item [(c)] More generally, for any $L \in \bN$, any fully connected $L$-layer neuron is a generalized neuron. Indeed, we may set $\Omega = \bR^d$, $m = \text{its $(l-1)$-th layer width}$, and each $f_k$ to be an $(l-1)$-layer neuron of it. Similarly, the partial derivatives of the parameters of $L$-layer neurons are generalized neurons. Therefore, any linear combination of $L$-layer neurons, as well as the partial derivatives of its parameters, are generalized NN. 

    \item [(d)] A Res-Net neural network is also an generalized NN, because it can be viewed as a linear combination of an $L$-layer network and its lower-level layers. 
\end{itemize}

Given any $m \in \bN$ and generalized neurons $H_1, ..., H_m$, our goal in this paper is to investigate what parameters $\theta_1, ..., \theta_m$ will guarantee that $H_1(\theta_1, \cdot) , ..., H_m(\theta_m, \cdot)$ are linearly dependent/independent. Obviously, the answer to this questions depends on $\sigma$. For example, when $\sigma$ is constant, then such $\theta$ does not exist at all; on the other hand, as shown by \cite{BSimsek, LZhangCrit}, for generic analytic $\sigma: \bR \to \bR$ and $H_k(\theta_k, x)= \sigma(\theta_k x)$ where $\theta_k, x \in \bR^d$, this holds if and only if $\theta_1, ..., \theta_m$ are distinct. We will establish a similar result for generalized neurons with generic activations $\sigma$. To do this we will first construct special activations, then carefully deform it to other ones. The construction relies on the growth of functions which we define below. First we introduce the standard notations on function asymptotics. 

\begin{defn}[Function asymptotics]
    Let $\Omega \siq \bF$ be open and $\gamma: [0, \infty) \to \Omega$ a curve. Given $f, g: \Omega \to \bF$, we denote 
    \begin{itemize}
        \item [(a)] $f(\gamma(t)) = O(g(\gamma(t))$ as $t \to \infty$ if there is a $C > 0$ with $|f(\gamma(t))| \le C|g(\gamma(t))|$; 
        \item [(b)] $f(\gamma(t)) = o(\gamma(t))$ as $t \to \infty$ if there for any $C > 0$ we have $|f(\gamma(t))| \le C|g(\gamma(t))|$; 
        \item [(c)] $f(\gamma(t)) = \Omega(g(\gamma(t))$ as $t \to \infty$  if there is a $C > 0$ with $|f(\gamma(t))| \ge C|g(\gamma(t))|$; 
        \item [(d)] $f(\gamma(t)) = \Theta(g(\gamma(t))$ as $t \to \infty$ if $f(\gamma(t)) = O(g(\gamma(t)))$ and $f(\gamma(t)) = \Omega(g(\gamma(t))$. 
    \end{itemize}
    Specifically, for $f, g: \bR \to \bF$, we write $f(x) = O(g(x))$, $f(x) = o(g(x))$, $f(x) = \Omega(g(x))$ and $f(x) = \Theta(g(x))$ as $x \to \infty\,(-\infty, \pm\infty)$ when the curve is just $[0, \infty)$ (or $(-\infty, 0]$ or $\bR$). 
\end{defn}

\begin{defn}[Hyper-polynomial growth]\label{Defn Hyper-polynomial growth}
    Let $\Omega \siq \bF$ be open and $f: \Omega \to \bF$. Given a curve $\gamma: [0, \infty) \to \Omega$, we say $\limftyt \re{f(\gamma(t))} = \infty (-\infty)$ at hyper-polynomial rate (along $\gamma$), or simply \textit{$f$ grows at hyper-polynomial rate (along $\gamma$)} if $\re f(\gamma(t_1)) = o(\re f(\gamma(t_2)))$ whenever 
    \begin{itemize}
        \item [(a)] $\frac{1}{l_\gamma(t_2) - l_\gamma(\infty)} - \frac{1}{l_\gamma(t_1) - l_\gamma(\infty)} \to \infty$ as $t_1, t_2 \to \infty$ if $l_\gamma(\infty)$ is finite, or 
        \item [(b)] $ l_\gamma(t_2) - l_\gamma(t_1) \to \infty$ as $t_1, t_2 \to \infty$ if $l_\gamma(\infty)$ is infinite. 
    \end{itemize}
\end{defn}

\begin{defn}[Hyper-exponential growth]\label{Defn Hyper-exponential growth}
    Let $\Omega \siq \bF$ be open and $f: \Omega \to \bF$. Given a curve $\gamma: [0, \infty) \to \Omega$, we say $\limftyt \re{f(\gamma(t))} = \infty (-\infty)$ at hyper-exponential rate (along $\gamma$), or simply \textit{$f$ grows at hyper-exponential rate (along $\gamma$)} if $\re f(\gamma(t_1)) = o(\re f(\gamma(t_2))$ whenever $t_1, t_2 \to \infty$ such that $l_\gamma(t_2) - l_\gamma(t_1)$ bounded below by a strictly positive number for all large $t_1, t_2$.  
\end{defn}

Below are examples and counter-examples of functions which grow at hyper-polynomial or hyper-exponential rates. 
\begin{itemize}
    \item [(a)] $f(z) = e^z$ and $\gamma(t) = t + ic$ for any $c\in \bR$ with $\cos(b) \ne 0$. Indeed, $\limftyt |\re e^{t+ic}| = |\cos(b)| e^t = \infty$ and we have 
    \[
        \frac{\re e^{t_1 + ic}}{\re e^{t_2 + ic}} = e^{-(t_2 - t_1)} \to 0
    \] 
    as $t_2 - t_1 \to \infty$. Thus, $e^{t_1 + ic} = o(e^{t_2 + ic})$ as $t_1, t_2, (t_2 - t-1) \to \infty$ at the same time. Thus, $\limftyt \re f(\gamma(t)) = \infty$ at hyper-polynomial rate. However, since 
    \[
        \frac{\re e^{t + ic}}{\re e^{t + b + ic}} = e^{-b}
    \]
    for any $b \in \bR$, the exponential function does not grow at hyper-exponential rate along $\gamma$. 

    \item [(c)] $f(x) = e^{x^2}$, $x \in \bR$ and $\gamma(t) = t$. Then $\limftyt f(\gamma(t)) = e^{t^2} = \infty$ and we have 
    \[
        \frac{e^{t^2}}{e^{(t+b)^2}} = e^{-(2bt + b^2)} \to 0
    \]
    as $t \to \infty$, for any $b > 0$. Thus, $\limftyt f(\gamma(t)) = \infty$ at hyper-exponential rate. 

    \item [(d)] If $\limftyt \re f(\gamma(t)) = \infty$ at hyper-exponential rate then $\limftyt \re f(\gamma(t)) = \infty$ at hyper-polynomial rate. To illustrate this let's consider the case for $l_\gamma(\infty) = \infty$. Since $f$ grows at hyper-exponential rate, for any $n \in \bN$ we can find $t, t' \to \infty$ such that $\re f(\gamma(t)) = o(\re f(\gamma(t')))$ and $l_\gamma(t'{}^{(n)}) \ge l_\gamma(t) + n$. Here $t'{}^{(n)}$ can be viewed as a function in $t$. Thus, we can find a sequence $\{\delta_n\}_{n=1}^\infty$ decreasing to 0, and a sequence $\{M_n\}_{n=1}^\infty$ increasing to $\infty$, such that for any $n$, 
    \[
        |\re f(\gamma(t))| \le \delta_n |\re f(\gamma(t'{}^{(n)}))| 
    \]
    whenever $t \ge M_n$. Now define $t_1, t_2$ (as functions in $t$) by $t_1 = t$ and $t_2 = t'{}^{(n)}$ for $t \in [M_n, M_{n+1})$. Then it is clear that $t_1, t_2 \to \infty$, $l_\gamma(t_2) - l_\gamma(t_1) \to \infty$, and $\re f(\gamma(t_1)) = o(\re f(\gamma(t_2)))$. 

    \item [(e)] When $|\dot{\gamma}|$ is bounded below over $[0, \infty)$ (so in particular $l_\gamma(\infty) = \infty$), then any function $f$ which grows faster than any polynomial grows at hyper-polynomial rate along $\gamma$, i.e., if for any $s \in \bN$ there is some $T > 0$ such that $f(\gamma(t)) > l_\gamma(t)^s$ whenever $t \ge T$, then $f$ grows at hyper-polynomial rate. 
\end{itemize}

\begin{defn}[Ordered growth]\label{Defn Ordered growth}
    Let $\Omega \siq \bF$ be open and $\gamma: \bR \to \Omega$ a curve. Given a finite collection of functions $f_1, ..., f_m: \Omega \to \bR$, we say $f_1, ..., f_m$ have ordered growth (along $\gamma$) if there is a permutation $\pi \in S_m$ such that, as $t \to \infty$ we have $f_{\pi(k)}(\gamma(t)) = o(f_{\pi(k-1)}(\gamma(t)))$ for all $2 \le k \le m$. 
\end{defn}

Below are examples of functions which have ordered growth. 
\begin{itemize}
    \item [(a)] Given $w_1, ..., w_m \in \bR$ such that $w_1 > ... > w_m$, then $e^{w_1 x}, ..., e^{w_m x}$ have ordered growth along $\gamma(t) = t$. 

    \item [(b)] More generally, given any function $f: \Omega \to \bF$ and a curve $\varphi: [0, \infty) \to \Omega$ such that $f$ grows at hyper-polynomial rate along $\varphi$. For any $m \in \bN$ and any $\tau_1, ..., \tau_m: [0, \infty) \to \bR$ such that $\tau_k(t) - \tau_{k-1}(t) \to \infty$ as $t \to \infty$, the functions 
    \[
        f_k := f\circ \varphi \circ \tau_k, \quad 1 \le k \le m
    \] 
    have ordered growth along $\gamma(t) = t$. This follows from the definition of Hyper-polynomial growth of a function. 
\end{itemize}

\section{Preparing Lemmas and Propositions}\label{Section Preparing Lemmas and Propositions}

This section is the main part of preparation of our results, which can be found in Section \ref{Section Theory of general neurons}. An important part of our general theory is to order the growth of a finite collection of functions. To obtain such results, we will apply the preparations in \ref{Subsection Functions of Ordered Growth}, where various results to help inductively establish ordered growth of generalized neurons are presented. Another part is about approximating some given functions, which relies on the preparations in \ref{Subsection Local asymptotics of parameterized functions} and \ref{Subsection Analytic bump functions}. In \ref{Subsection Local asymptotics of parameterized functions}, we study how the asymptotics of a parameterized function depends on its parameters. In \ref{Subsection Analytic bump functions}, we introduce a class of analytic functions similar to bump functions; this will be used in the constructions in Section \ref{Section Theory of general neurons}. But first, we show that the study of linear independence of generalized neurons of arbitrary input dimension $d \in \bN$ can be reduced to input dimension $d = 1$. 

\begin{lemma}[dimension reduction]\label{Lem Dimension reduction}
    Fix $m \in \bN$. Given distinct $w_1, ..., w_m \in \bR^d$. Then there is a $v \in \partial B(0,1) \siq \bR^d$ such that $\< w_1, v\>, ..., \<w_m, v\>$ are distinct. Moreover, if $w_k, w_j$ are multiples to one another, then for any $v \in \partial B(0,1)$, $\<w_k, v\>, \<w_j, v\>$ are multiples to one another. 
\end{lemma}

\subsection{Functions of Ordered Growth}\label{Subsection Functions of Ordered Growth}

Now we consider functions of ordered growth. One important result is that these functions are linearly independent. This is exactly presented in the following proposition. 

\begin{prop}[functions of ordered growth are linearly independent]\label{Prop Ordered growth implies linear independence}
    Let $m \in \bN$ and $\Omega \siq \bF$ be open. Suppose that $f_1, ..., f_m: \Omega \to \bR$ have ordered growth along a curve $\gamma: [0, \infty) \to \Omega$ and none of $f_k \circ \gamma$ is eventually constant zero, i.e., $\sup\{t: f_k(\gamma(t)) \ne 0\} = \infty$. Then 
    \begin{itemize}
        \item [(a)] $f_1, ..., f_m$ are linearly independent. 

        \item [(b)] $\frac{1}{f_1}, ..., \frac{1}{f_m}$ are linearly independent. 
    \end{itemize}
    Thus, if $\sum_{k=1}^m a_k f_k$ or $\sum_{k=1}^m a_k \frac{1}{f_k}$ is constant zero, then $a_1 = ... = a_m = 0$. 
\end{prop}
\begin{proof} 
By rearranging the indices if necessary, we can assume that in the proof for (a) and (b) below, $f_k(\gamma(t)) = o(f_{k-1}(\gamma(t)))$ as $t \to \infty$, for all $2 \le k \le m$. 

\begin{itemize}
    \item [(a)] Let $a_1, ..., a_m \in \bF$ be constants such that $\sum_{k=1}^m a_k f_k \equiv 0$. By our assumption, there is a sequence $\{t_j\}_{j=1}^\infty$ diverging to $\infty$ such that $f_1(\gamma(t_j)) \ne 0$. Moreover, we have $f_k(\gamma(t_j)) = o(f_1(\gamma(t_j)))$ as $j \to \infty$, for all $2 \le k \le m$. It follows that 
    \[
        a_1 = - \limftyj \sum_{k=2}^m a_k \frac{f_k(\gamma(t_j))}{f_1(\gamma(t_j))} = 0. 
    \]
    This means we can instead consider $\sum_{k=2}^m a_k f_k$, which is also constant zero. By repeating this argument, we will eventually show that $a_k = 0$ for all $k$, thus proving the linear independence of the functions $f_1, ..., f_m$. 

    \item [(b)] The proof is similar to (a). Let $a_1, ..., a_m \in \bF$ be constants such that $\sum_{k=1}^m \frac{a_k}{f_k} \equiv 0$. Let $\{t_j\}_{j=1}^\infty$ diverging to $\infty$ be such that $f_m(\gamma(t_j)) \ne 0$ for all $j \in \bN$. It follows that 
    \begin{align*}
        a_m 
        &= - \limftyj \sum_{k=1}^{m-1} a_k \frac{1 / f_k(\gamma(t_j)}{1 / f_m(\gamma(t_j)} \\ 
        &= - \limftyj \sum_{k=1}^{m-1} a_k \frac{f_m(\gamma(t_j))}{f_k(\gamma(t_j))} = 0. 
    \end{align*}
    This means we can instead consider $\sum_{k=1}^{m-1} \frac{a_k}{f_k}$, which is also constant zero. By repeating this argument we will eventually show that $a_k = 0$ for all $k$, thus proving the linear independence of the functions $\frac{1}{f_1}, ..., \frac{1}{f_m}$. 
\end{itemize}
\end{proof}
\begin{remark}
    In fact, our proof for (a) implies that as long as $f_k(\gamma(t)) = o(f_j(\gamma(t)))$ as $t \to \infty$ for all $k \ne j$, then the constant-zero function $\sum_{k=1}^m a_k f_k$ must give $a_j = 0$. This does not fully use the fact that $f_1, ..., f_m$ have ordered growth. The ordered growth property simply helps us restrict the functions to a single curve $\gamma$ to establish their linear independence. A similar remark holds for (b). 
\end{remark}

The next proposition works for $f_1, ..., f_m$'s with ordered growth. They will be used to deduce complete characterizations of linear independence of multi-layer neurons with certain types of activations. 

\begin{prop}[Functions of ordered growth -- I]\label{Prop Functions of ordered growth I}
    Let $f_1, ..., f_m: \Omega \to \bR$ be functions of ordered growth along a curve $\gamma: [0, \infty) \to \Omega$, such that $\limftyt f_k(\gamma(t)) = \infty$ for all $1 \le k \le m$. Suppose that $\sigma: \bR \to \bR$ satisfies 
    \begin{itemize}
        \item [(a)] $\lim_{z \to \infty} \sigma(z) = \infty$ and $\lim_{z \to -\infty} \sigma(z) = \infty$, both at hyper-polynomial rate. 
        \item [(b)] For all $w > 0$ and $\tw < 0$, either $\sigma(\tw f_{\pi(1)}(\gamma(t))) = o(\sigma(w f_{\pi(m)}(\gamma(t))))$ or $\sigma(w f_{\pi(1)}(\gamma(t))) = o(\sigma(\tw f_{\pi(m)}(\gamma(t))))$ as $t \to \infty$. Here $\pi$ is the permutation defined as in Definition \ref{Defn Ordered growth}. 
    \end{itemize}
    Given $n \in \bN$ and parameters $\{(w_{jk})_{k=1}^m\}_{j=1}^n \siq \bR^m$. Define for each $j$ a function 
    \begin{equation}
        H_j(z) := \sigma \left (\sum_{k=1}^m w_{jk} f_k(z) \right). 
    \end{equation}
    Then $H_1, ..., H_n: \Omega \to \bR$ have ordered growth (along $\gamma$) and $\limftyt H_j(\gamma(t)) = \infty$ for all $1 \le j \le n$ if and only if all the $w_j := (w_{jk})_{k=1}^m$ are distinct and $\ne 0$. 
\end{prop}
\begin{proof}
    It is clear that if $w_{j_1} = w_{j_2}$, then $H_1, ..., H_m$ cannot have ordered growth along $\gamma$, and if $w_j = 0$ for some $j$, then $\limftyt H_j(\gamma(t)) = \sigma(0) < \infty$. So it remains to deal with the other direction. By rearranging the indices of $f_1, ..., f_m$, we may assume that $f_k(\gamma(t)) = o(f_{k-1}(\gamma(t)))$ as $t \to \infty$ for all $2 \le k \le m$. \\
    
    First, let's prove that $\limftyt H_j(\gamma(t)) = \infty$ whenever $w_j \ne 0$. So fix any $j$. Because $w_j \ne 0$, there is a smallest $k \in \{1, ..., m\}$ such that $w_{jk} \ne 0$. Then 
    \[
        \limftyt \left| w_{jk} f_k(\gamma(t)) \right| = |w_{jk}| \limftyt f_k(\gamma(t)) = \infty. 
    \]
    Moreover, our assumption implies that $f_{k'}(\gamma(t)) = o(f_k(\gamma(t)))$ for all $k' > k$. This, together with the fact that $w_{j1} = ... = w_{j(k-1)} = 0$, yields 
    \begin{equation}\label{eq 1 of Prop Functions of ordered growth I}
        \sum_{k'=1}^m w_{jk'} f_{k'}(\gamma(t)) \sim w_{jk} f_k(\gamma(t)) 
    \end{equation}
    as $t \to \infty$. In particular, 
    \[
        \limftyt \left| \sum_{k'=1}^m w_{jk'} f_{k'}(\gamma(t)) \right| = \infty. 
    \]
    It follows from property (a) of $\sigma$ that $\limftyt H_j(\gamma(t)) = \infty$. \\

    Then we investigate the ordered growth property these $H_j$'s. For this we assume without loss of generality that $\sigma(\tw f_1(\gamma(t))) = o( \sigma(w f_m(\gamma(t))))$ for any $w > 0$ and $\tw < 0$. We claim that this implies $\sigma(\tw f_k (\gamma(t))) = o(\sigma(w f_j(\gamma(t))))$ for all $k, j \in \{1, ..., m\}$, whenever $w > 0$ and $\tw < 0$. Indeed, given $k \ne 1$ and/or $j \ne m$, we have 
    \begin{align*}
        f_k(\gamma(t)) = o(f_1(\gamma(t))), \text{and/or } f_m(\gamma(t)) = o(f_j(\gamma(t))). 
    \end{align*}
    Thus,
    \begin{equation}\label{eq 2 of Prop Functions of ordered growth I}
    \begin{aligned}
        \tw f_k(\gamma(t)) = o(f_1(\gamma(t))), \text{and/or } w f_m(\gamma(t)) = o(f_j(\gamma(t))). 
    \end{aligned}
    \end{equation}
    Since $w$ and $\tw$ are non-zero, we also have 
    \begin{align*}
        \tw f_k(\gamma(t)), \tw f_1(\gamma(t)) &\to -\infty, \\
        w f_m(\gamma(t)), w f_j(\gamma(t)) &\to \infty
    \end{align*}
    as $t \to \infty$. Combining these with property (a) of $\sigma$, we can easily see that 
    \[
        \sigma(\tw f_k(\gamma(t)) \ll \sigma(\tw f_k(\gamma(t))) \ll \sigma(w f_m(\gamma(t)))) \ll \sigma(w f_j(\gamma(t))). 
    \]
    Similarly, property (a) of $\sigma$ also implies that $\sigma(w f_k(\gamma(t)) = o( \sigma(\tw f_j(\gamma(t))))$ whenever $k > j$ and $w, \tw$ are both (strictly) positive/negative. Moreover, note that for any distinct $j_1, j_2$, because $w_{j_1} \ne w_{j_2}$, there is a smallest $k \in \{1, ..., m\}$ with $w_{j_1 k} \ne w_{j_2 k}$. By our proof above, this implies that 
    \[
        \limftyt \left| \sum_{k'=1}^m (w_{j_1 k'} - w_{j_2 k'}) f_{k'}(\gamma(t)) \right| = \limftyt |w_{j_1 k} - w_{j_2 k}| f_k(\gamma(t)) = \infty. 
    \] 

    Now fix distinct $j_1, j_2 \in \{1, ..., n\}$. Let $k_1, k_2 \in \{1, ..., m\}$ be the smallest numbers with $w_{j_1 k_1}, w_{j_2 k_2} \ne 0$, respectively, and let $k_1', k_2' \in \{1, ..., m\}$ be the smallest numbers with $w_{j_1 k_1'} \ne w_{j_2 k_2'}$. Consider the following possible cases. If $w_{j_1 k_1} > 0$ and $w_{j_2 k_2} < 0$, then by (\ref{eq 2 of Prop Functions of ordered growth I}) above, we have $\sigma(w_{j_2 k_2} f_{k_2}(\gamma(t))) = o(\sigma(w_{j_1 k_1} f_{k_1}(\gamma(t)))$ as $t \to \infty$. By (\ref{eq 1 of Prop Functions of ordered growth I}), 
    \begin{align*}
        \sum_{k=1}^m w_{j_2 k} f_k(\gamma(t)) &\sim w_{j_2 k_2} f_{k_2}(\gamma(t)), \\
        \sum_{k=1}^m w_{j_1 k} f_k(\gamma(t)) &\sim w_{j_1 k} f_k(\gamma(t)). 
    \end{align*}
    It follows that $H_{j_2}(\gamma(t)) = o(H_{j_1}(\gamma(t)))$. Similarly, if $w_{j_1 k_1} < 0$ and $w_{j_2 k_2} > 0$ we can show that $H_{j_1}(\gamma(t)) = o(H_{j_2}(\gamma(t)))$. Then suppose $w_{j_1 k_1}$ and $w_{j_2 k_2}$ are both positive or negative. This means $\sum_{k=1}^m w_{j_1k} f_k(\gamma(t))$ and $\sum_{k=1}^m w_{j_2k} f_k(\gamma(t))$ both diverge to $\infty$ or $-\infty$. Moreover, if $w_{j_1 k_1'} - w_{j_2 k_2'} > 0$ then $\sum_{k=1}^m (w_{j_1 k} - w_{j_2 k}) f_k(\gamma(t)) \to \infty$, while $w_{j_1 k_1'} - w_{j_2 k_2'} < 0$ implies $\sum_{k=1}^m (w_{j_1 k} - w_{j_2 k}) f_k(\gamma(t)) \to -\infty$. Using property (a) of $\sigma$, for the first case, we have $H_{j_2}(\gamma(t)) = o(H_{j_1}(\gamma(t))$, while for the second case $H_{j_1}(\gamma(t)) = o(H_{j_2}(\gamma(t)))$. \\

    By dealing with the three cases above, we show that for any distinct $j_1, j_2$ we must have $H_{j_1}(\gamma(t)) = o(H_{j_2}(\gamma(t)))$ or vice versa. The proof is completed by noting that this result is equivalent to $H_1, ..., H_n$ having ordered growth along $\gamma$. 
\end{proof}

\begin{prop}[Functions of ordered growth -- II]\label{Prop Functions of ordered growth II}
    Let $f_1, ..., f_m: \Omega \to \bR$ be functions of ordered growth along a curve $\gamma: [0, \infty) \to \Omega$, such that $\limftyt f_k(\gamma(t)) = \infty$ for all $1 \le k \le m$. Suppose that $\sigma: \bR \to \bR$ satisfies 
    \begin{itemize}
        \item [(a)] $\limftyz \sigma(z) = \infty$ and $\lim_{z \to -\infty} \sigma(z) = \infty$, both at hyper-exponential rate. 
        \item [(b)] For all $w > 0$ and $\tw < 0$, either $\sigma(\tw f_{\pi(1)}(\gamma(t))) = o(\sigma(w f_{\pi(m)}(\gamma(t))))$ or $\sigma(w f_{\pi(1)}(\gamma(t))) = o(\sigma(\tw f_{\pi(m)}(\gamma(t))))$ as $t \to \infty$. Here $\pi$ is the permutation defined as in Definition \ref{Defn Ordered growth}. 
    \end{itemize}
    Given $n \in \bN$ and parameters $\left\{ \left( (w_{jk})_{k=1}^m, b_j\right) \right\} \siq \bR^{m+1}$. Define for each $j$ a function 
    \begin{equation}\label{eq defn H_j with bias}
        H_j(z) := \sigma\left( \sum_{k=1}^m w_{jk}f_k(z) + b_j \right). 
    \end{equation} 
    Then $H_1, ..., H_n: \Omega \to \bR$ have ordered growth (along $\gamma$) and $\limftyt H_j(\gamma(t)) = \infty$ for all $1 \le j \le n$ if and only if all the $(w_j, b_j)$'s are distinct and none of $w_j = 0$. 
\end{prop}
\begin{proof}
    It is clear that if $w_{j_1} = w_{j_2}$, then $H_1, ..., H_m$ cannot have ordered growth along $\gamma$, and if $w_j = 0$ for some $j$, then $\limftyt H_j(\gamma(t)) = \sigma(b_j) < \infty$. So it remains to prove the other direction. Again, we will show that for any distinct $j_1, j_2 \in \{1, ..., m\}$, either $H_{j_1}(\gamma(t)) = o(H_{j_2}(\gamma(t)))$ or vice versa, as $t \to \infty$. The proof is identical to the one for Proposition \ref{Prop Functions of ordered growth I} above, except that we need to deal with one more case: $w_{j_1} = w_{j_2}$ and $b_{j_1} \ne b_{j_2}$ for some $j_1, j_2 \in \{1, ..., n\}$. Note that as $t \to \infty$, 
    \[
        \sum_{k=1}^m w_{j_1k} f_k(\gamma(t)) + b_{j_1},\, \sum_{k=1}^m w_{j_2k} f_k(\gamma(t)) + b_{j_2} \to \infty (-\infty) 
    \]
    simultaneously; in particular, they have the same signs. Thus, by property (a) of $\sigma$, as $t \to \infty$:  
    \begin{itemize}
        \item [(a)] If they are both positive (diverge to $\infty$) and $b_{j_1} > b_{j_2}$, then $H_{j_2}(\gamma(t)) = o(H_{j_1}(\gamma(t)))$. 
        \item [(b)] If they are both positive (diverge to $\infty$) and $b_{j_1} < b_{j_2}$, then $H_{j_1}(\gamma(t)) = o(H_{j_2}(\gamma(t)))$. 
        \item [(c)] If they are both negative (diverge to $-\infty$) and $b_{j_1} > b_{j_2}$, then $H_{j_1}(\gamma(t)) = o(H_{j_2}(\gamma(t)))$. 
        \item [(d)] If they are both positive (diverge to $-\infty$) and $b_{j_1} < b_{j_2}$, then $H_{j_2}(\gamma(t)) = o(H_{j_1}(\gamma(t)))$. 
    \end{itemize}
    This completes the proof. 
\end{proof}

Below we give some examples about $\sigma$'s which make Proposition \ref{Prop Functions of ordered growth I} and/or \ref{Prop Functions of ordered growth II} hold. The idea is to construct $\sigma$ as $\sigma(z) = a(z) + \ta(-z)$, where $a, \ta$ grows at different orders and $\lim_{z\to-\infty} a(z)$, $\lim_{z\to-\infty} \ta(z)$ are bounded. For simplicity, consider $m = 2$ and $f_1(x) = x^2$, $f_2(x) = x$, $x \in \bR$. Clearly, $\limftyx f_1(x) = \limftyt f_2(x) = \infty$ and $f_2(x) = o(f_1(x))$ as $x \to \infty$, i.e., they have ordered growth. 
\begin{itemize}
    \item [(a)] For Proposition \ref{Prop Functions of ordered growth I}, define $\sigma(z) = e^{z^3} + e^{-z}$. Then for any $w > 0$ and $\tw < 0$, we have $\sigma(\tw f_1(x)) \sim e^{\tw x^2}$, while $\sigma(w f_2(x)) \sim e^{w^3x^3}$. Thus, the hypotheses in Proposition \ref{Prop Functions of ordered growth I} are satisfied. 

    \item [(b)] For proposition \ref{Prop Functions of ordered growth II}, define $\sigma(z) = e^{z^7} + e^{-z^3}$. In the same way as (a) we can check that the hypotheses in Proposition \ref{Prop Functions of ordered growth II} are satisfied. Clearly this $\sigma$ also satisfies the hypotheses in Proposition \ref{Prop Functions of ordered growth I}. 
    \end{itemize} 
As we can see from these examples, the ``local behavior" of $\sigma$ is not important -- we care about the behavior of $\sigma(z)$ as $z \to \pm\infty$. In fact, the idea is to play with its growth. Note that this depends on the $f_k$'s. In general, if we are dealing with $L$-layer generalized neurons, we need to take every layer into consideration. Also note that the $\sigma$ we construct differs much from what is used in practice. However, as we shall see in Section \ref{Section Theory of general neurons}, we can carefully deform it to approximate a commonly used activation function very well. 

\begin{remark}
    Proposition \ref{Prop Functions of ordered growth I} and Proposition \ref{Prop Functions of ordered growth II} both deal with generalized neurons of the form $\sigma\left(\sum_{k=1}^m w_k f_k + b_k\right)$ (a.k.a. the $H_j$'s defined in (\ref{eq defn H_j with bias})). Obviously they focus on neurons of multiple layers. However, we can apply the idea of them to deal with more kinds of generalized neurons. The result below is an example about how to extend them. It deals with $H_j$'s and their partial derivatives against parameters. 
\end{remark}

\begin{prop}\label{Prop Functions of ordered growth III}
    Let $f_1, ..., f_m: \Omega \to \bR$ be functions of ordered growth along a curve $\gamma: [0, \infty) \to \Omega$, such that $\limftyt f_k(\gamma(t)) = \infty$ for all $1 \le k \le m$. Suppose that $\sigma: \bR \to \bR$ satisfies 
    \begin{itemize}
        \item [(a)] $\limftyz \sigma(z), \sigma'(z) = \infty$ and $\lim_{z\to-\infty} \sigma(z), \sigma'(z) = \infty$, such that $\sigma(z) = o(\sigma'(z))$ as $z \to \pm\infty$. 

        \item [(b)] We have $\sigma'(w f_{\pi(k)})f_{\pi(1)}(\gamma(t)) = o(\sigma(\tw f_{\pi(k-1)}(\gamma(t))))$ whenever $w, \tw$ are both strictly positive or negative. 

        \item [(c)] For all $w > 0$ and $\tw < 0$, either $\sigma'(\tw f_{\pi(1)}(\gamma(t)))f_{\pi(1)}(\gamma(t)) = o(\sigma(w f_{\pi(m)}(\gamma(t))))$, or\\ $\sigma'(w f_{\pi(1)}(\gamma(t)))f_{\pi(1)}(\gamma(t)) = o(\sigma(\tw f_{\pi(m)}(\gamma(t))))$ as $t \to \infty$. Here $\pi$ is the permutation defined as in Definition \ref{Defn Ordered growth}. 
    \end{itemize}
    Let $n \in \bN$ and for each $j \in \bN$ define $H_j$ as in formula (\ref{eq defn H_j with bias}). Then, if and only if the $(w_j, b_j)$'s are distinct and $w_j \ne 0$ for all $j$, the functions 
    \[
        \left\{ H_j, \partial_{w_{jk}} H_j, \partial_{b_j} H_j \right\}_{j,k=1}^{n,m} 
    \]
    have ordered growth along $\gamma$, and 
    \[
        \limftyt H_j(\gamma(t)) = \limftyt \partial_{w_{jk}} H_j(\gamma(t)) = \limftyt \partial_{b_j} H_j(\gamma(t)) = \infty 
    \]
    for every $1 \le j \le n$ and $1 \le k \le m$. 
\end{prop}
\begin{proof}
    Without loss of generality, we assume that $f_k(\gamma(t)) = o(f_{k-1}(\gamma(t)))$ for all $2 \le k \le m$ and  $\sigma'(\tw f_1(\gamma(t)))f_1(\gamma(t)) = o(\sigma(w f_m(\gamma(t))))$ as $t \to \infty$, for all $w > 0$ and $\tw < 0$. Similar as in the proof for Proposition \ref{Prop Functions of ordered growth I}, for any $k,j \in \{1, ..., m\}$ and any $w > 0$ and $\tw < 0$, we have
    \[
        \sigma'(\tw f_k(\gamma(t)))f_k(\gamma(t)) = o(\sigma(w f_j(\gamma(t)))) 
    \]
    as $t \to \infty$. Combining this with property (a) and (b) of $\sigma$, we conclude that for any $w_1, ..., w_m > 0$ and $\tw_1, ..., \tw_m < 0$, we have the following ordering: 
    \begin{equation}\label{eq 1 of Prop Functions of ordered growth III}
    \begin{aligned}
        \sigma(\tw_m f_m(\gamma(t)) &\ll ...\\
        &\ll \sigma(\tw_1 f_1(\gamma(t))) \ll \sigma'(\tw_1 f_1(\gamma(t))) \ll \sigma'(\tw_1 f_1(\gamma(t)))f_1(\gamma(t)) \\
        &\ll \sigma(w_m f_m(\gamma(t))) \ll \sigma'(w_m f_m(\gamma(t))) \ll \sigma'(w_m f_m(\gamma(t))) f_1(\gamma(t)) \\
        &\ll ... \\
        &\ll \sigma'(w_1 f_1(\gamma(t))) f_1(\gamma(t)). 
    \end{aligned}
    \end{equation}
    
    Given $1 \le j \le n$, by our definition of $H_j$ we have 
    \begin{align*}
        \partial_{w_{jk}} H_j(z) &= \sigma'\left( \sum_{k'=1}^m w_{jk'} f_{k'}(z) + b_{k'} \right) f_k(z) \\ 
        \partial_{b_j} H_j(z) &= \sigma'\left( \sum_{k=1}^m w_{jk} f_k(z) + b_k \right). 
    \end{align*}
    These are clearly generalized neurons. As mentioned in previous proofs, for each $j$ there is some $k(j) \in \{1, ..., m\}$ such that  
    \[
        \sum_{k=1}^m w_{jk} f_k(\gamma(t)) \sim w_{j k(j)} f_{k(j)} (\gamma(t)) 
    \]
    as $t \to \infty$. Fix $j_1, j_2 \in \{1, ..., n\}$. We would like to establish an ordering of the functions $\left\{ H_{j_i}, \partial_{w_{j_i k}}, \partial_{b_{j_i}}H_{j_i} \right\}_{k,i=1}^{m,2}$. Consider the following cases. Note that in (\ref{eq 1 of Prop Functions of ordered growth III}), the $w_j$ and $\tw_j$'s are arbitrary. 
    \begin{itemize}
        \item [(a)] Suppose that $w_{j_1 k(j_1)}$ and $w_{j_2 k(j_2)}$ have different signs; without loss of generality assume that $w_{j_1 k(j_1)} > 0$ and $w_{j_2 k(j_2)} < 0$. Then by (\ref{eq 1 of Prop Functions of ordered growth III}), 
        \begin{align*}
            &\,\,\,\,\,\,\,H_{j_2}(\gamma(t)) \ll \partial_{b_{j_2}} H_{j_2}(\gamma(t)) \ll \partial_{w_{j_2 1}} H_{j_2}(\gamma(t)) \\
            &\ll H_{j_1}(\gamma(t)) \ll \partial_{b_{j_1}} H_{j_1}(\gamma(t)) \ll \partial_{w_{j_1 1}} H_{j_1}(\gamma(t)). 
        \end{align*}
        Moreover, using the ordered growth of $f_1, ..., f_m$ (recall the formulas for $\partial_{w_jk} H_j$'s), 
        \[
            \partial_{b_{j_2}} H_{j_2}(\gamma(t)) \ll \partial_{w_{j_2 m}} H_{j_2}(\gamma(t)) \ll ... \ll \partial_{w_{j_2 1}} H_{j_2}(\gamma(t))
        \]
        and a similar inequality holds for $j_1$. Thus, we establish an ordering for all the functions. 

        \item [(b)] Suppose that $w_{j_1 k(j_1)}$ and $w_{j_2 k(j_2)}$ are both positive or negative. Let $k' \in \{1, ..., m\}$ is the smallest number such that $w_{j_1 k'} \ne w_{j_2 k'}$. Then 
        \[
            \sum_{k=1}^m (w_{j_1 k} - w_{j_2 k}) f_k(\gamma(t)) \sim (w_{j_1 k'} - w_{j_2 k'}) f_k(\gamma(t)) 
        \]
        as $t \to \infty$. If $w_{j_2 k'} > w_{j_1 k'}$, then by property (a) and (b) of $\sigma$ or inequality (\ref{eq 1 of Prop Functions of ordered growth III}), we have 
        \begin{align*}
            H_{j_1}(\gamma(t)) &\ll \partial_{b_{j_1}} H_{j_1}(\gamma(t)) \ll \partial_{w_{j_1 1}} H_{j_1}(\gamma(t)) \\ 
                               &\ll \partial_{b_{j_2}} H_{j_2}(\gamma(t)) \ll \partial_{w_{j_2 1}} H_{j_1}(\gamma(t)). 
        \end{align*}
        If $w_{j_2 k'} < w_{j_1 k'}$ we simply argue in the same way to obtain 
        \begin{align*}
            H_{j_2}(\gamma(t)) &\ll \partial_{b_{j_2}} H_{j_2}(\gamma(t)) \ll \partial_{w_{j_2 1}} H_{j_1}(\gamma(t)) \\ 
                               &\ll \partial_{b_{j_1}} H_{j_1}(\gamma(t)) \ll \partial_{w_{j_1 1}} H_{j_1}(\gamma(t)). 
        \end{align*}
        Finally, using the ordered growth of $f_1, ..., f_m$ we also have 
        \[
            \partial_{b_{j_1}} H_{j_1}(\gamma(t)) \ll \partial_{w_{j_1 m}} H_{j_1}(\gamma(t)) \ll ... \ll \partial_{w_{j_1 1}} H_{j_1}(\gamma(t))
        \]
        and similarly for $j_2$. Thus, we can also establish an ordering for all the functions. 
    \end{itemize}
\end{proof}

Let's consider the example (b) above again. We claim that the activation function $\sigma(z) = e^{z^7} + e^{-z^3}$ satisfies the hypotheses in Proposition \ref{Prop Functions of ordered growth III}. First note that $\sigma'(z) = 7z^6 e^{z^7} - 3z^2 e^{-z^3}$, whence $\sigma'(z) \sim 7z^6 e^{z^7}$ as $z \to \infty$, and $\sigma'(z) \sim -3z^2 e^{-z^3}$ as $z \to -\infty$. Thus, $\sigma$ clearly have property (a). For (b), given $w > 0$ and $\tw < 0$, we have 
\begin{align*}
    \sigma'(\tw x^2) \sim -3\tw^2 x^4 e^{-\tw^3 x^6}, \quad \sigma'(w x) \sim 7w^6 x^6 e^{x^7}
\end{align*}
as $t \to \infty$. Finally, (c) is mentioned in Proposition \ref{Prop Functions of ordered growth II}, so we have proved that (c) holds as well. This shows that $\sigma$ satisfies the hypotheses in Proposition \ref{Prop Functions of ordered growth III}. 

\begin{remark}
    In the same way we can formalize and prove a similar result concerning finitely many $L$-layer neurons and all the partial derivatives against parameters of them. 
\end{remark}

\subsection{Local Asymptotics of Parameterized Functions}\label{Subsection Local asymptotics of parameterized functions}

Consider a parameterized function $H: \bF^N \times \bF \to \bF$, where $\theta \in \bF^N$ is called the parameters of $H$ and $z \in \bF$ the input of $H$ (see Section \ref{Section Notations and assumptions}). The level sets, the asymptotics of $H(\theta, \cdot)$ near any $x_0 \in \bF$ all depend on $\theta$: for example, when $H$ is analytic, we have $H(\theta, x) \sim c(x-x_0)^s$ for some $c$ and $s$ determined by $\theta$. Below we study how these properties depend on $\theta$. We will first prove the most general case. Then we will focus specifically on definable functions (Defintion \ref{Defn Definable function}) which gives us a stronger result. \\

Before we present the results, let's first make the following definition. 

\begin{defn}
    Define $\calA(\bF)$ as the set of real analytic functions $f: \bF \to \bF$. For each $S \in \bN \cup \{0\}$ define $C^S(\bF)$ the set of functions $f: \bF \to \bF$ whose $s$-th derivatives are all continuous, $0 \le s \le S$. Given such a function space $\calF := \calA(\bF)$ or $C^S(\bF)$, define the following topologies on it: 
    \begin{itemize}
        \item [(a)] (Uniform convergence) Given $E \siq \bF$, let $\norm{f}_{\infty, E} := \sup_{x \in E} |f(x)|$ for $f \in \calF$. Specifically, we write $\norm{f}_\infty$ when $E = \bF$. 

        \item [(b)] ($S$-order uniform convergence). Given $E \siq \bF$, let $\norm{f}_{S, E} := \sup_{x \in E} \sum_{s=0}^S |f^{(s)}(x)|$ for $f \in \calF$, where $f^{(s)}$ is the $s$-th derivative of $f$. Specifically, we write $\norm{f}_S$ when $E = \bF$. 

        \item [(b)] (local $S$-order uniform convergence) An open ball around some $f \in \calF$ has the form 
        \[
            \left\{ g \in \calF: \sum_{s=0}^S \norm{g^{(s)} - f^{(s)}}_{\infty, \overline{B(z^*, r)}} < \delta \right\} 
        \]
        where $\overline{B(z^*, r)} := \{z \in \bF: |z - z^*| \le r\}$. 
    \end{itemize}
\end{defn}
\begin{remark}
    When $S = 0$, we write $C(\bF)$ for $C^0(F)$; this is just the set of continuous functions on $\bF$. Also note that for $S, S' \in \bN$, the $S'$-order and local $S'$-order uniform convergence topologies on $C^S(\bF)$ are well-defined only when $S' \le S$. \\
\end{remark}

The following lemma is so straightforward that we omit the proof for it. 

\begin{lemma}\label{Lem Bound of approximating function}
    Let $E$ be a subset of $\bF^N$. Suppose that $\calL: \calA(\bF) \times \bF^N \to [0, \infty)$ is a function and there are constants $L, R > 0$ such that for any $\sigma, \tsigma \in \calA(\bF)$, we have $\sup_{\theta \in E} \norm{\calL_\sigma - \calL_{\tsigma}} < L \norm{\sigma - \tsigma}_{\infty, [-R,R]}$. Then for any $\sigma, \tsigma \in \calA(\bF)$, $\calL_{\tsigma}^{-1}(0)$ is contained in $\calL_{\sigma}^{-1}\left[0, L \norm{\sigma - \tsigma}_{\infty, [-R,R]} \right)$.
\end{lemma}

Another way to state this lemma is that, given $\sigma$ and $\vep > 0$, for any $\tsigma \in \calA(\bF)$ with $\norm{\sigma - \tsigma}_{\infty, [-R,R]} < \frac{\vep}{L}$ we have $\calL_{\tsigma}^{-1}(0) \siq \calL_{\sigma}^{-1}[0, \vep)$, so locally $\calL_{\tsigma}^{-1}(0)$ and $\calL_{\sigma}^{-1}(0)$ ``looks" similar. We would like to first find a $\sigma$ and a function $\calL$ with good properties, then carefully change $\sigma$ to $\tsigma$. This operation is contained in the proof for Proposition \ref{Prop Structure of calC I}. Then we apply Lemma \ref{Lem Bound of approximating function} to show that $\calL_\sigma^{-1}(0)$ and $\calL_{\tsigma}^{-1}(0)$ ``looks similar", at least locally. \\

Under the setting of neural networks, this simple result says more. For example, consider $\calL(\sigma, \theta) := \int_0^1 \left( \sum_{k=1}^m a_k \sigma(w_k x) \right)^2 dx$. By the existing results linear independence of (unbiased) two-layer neurons \cite{BSimsek}, when $\sigma$ is a non-polynomial analytic function whose derivatives are all not odd or even, then $\calL(\sigma, \theta) = 0$ if and only if $\theta = (a_k, w_k)_{k=1}^m$ satisfies: 
\begin{itemize}
    \item [(a)] the $w_k$'s are not all distinct, 
    \item [(b)] the sum of $a_k$'s whose corresponding $w_k$'s are the same equals zero. 
\end{itemize}
Then Lemma \ref{Lem Bound of approximating function} implies that for $\tsigma \in \calA(\bR)$ sufficiently close to $\sigma$, $\calL(\tsigma, \theta) \ne 0$ whenever the $w_k$'s have a ``gap" to one another, i.e., for any $k \ne j$ we have $|w_k - w_j| > \vep$ for some $\vep > 0$. Since $\sum_{k=1}^m a_k \tsigma(w_k x)$ is analytic, these neurons $\{\tsigma(w_k x)\}_{k=1}^m$ are linearly independent. \\

The next result we will use frequently is Corollary \ref{Cor Limiting asymp of dim-1 parameterized function near 0}. To prove it we need Lemma \ref{Lem finite analytic functions kill zeros locally} and Lemma \ref{Lem finite definable functions kills zeros globally}, which can be deduced from Lemma \ref{Lem analytic function zero set stratification}. Note that it is a summary of Theorem 6.3.3 in \cite{SKrantz}. The original theorem gives a more detailed characterization of the zero set of an analytic function, but we do not need that much. 

\begin{lemma}[stratification of zero set]\label{Lem analytic function zero set stratification}
    Given $N \in \bN$. Let $\Omega \siq \bF^N$ be open. Then for any analytic function $f: \Omega \to \bF$, $f^{-1}(0)$ is a locally finite union of analytic submanifolds of $\Omega$ (or $\bF^N$), i.e., for any compact $K \siq \Omega$, $f^{-1}(0) \cap K$ is a union of finitely many manifolds.
\end{lemma}

To prove Lemma \ref{Lem finite analytic functions kill zeros locally} and Lemma \ref{Lem finite definable functions kills zeros globally}, we first introduce the idea of o-minimal sets and definable functions \citep{AGhorbani, AWilkie}. 

\begin{defn}[o-minimal structure, rephrased from \cite{AWilkie}]\label{Defn Ominimal structure}
    We define a structure $\calS$ on $\bR$ to be a sequence $\{S_n\}_{n=0}^\infty$ satisfying the following axioms for each $n \in \bN \cup \{0\}$.
    \begin{itemize}
        \item [(a)] $S_n$ is an algebra of subsets of $\bR^n$. 
        \item [(b)] If $A \in S_n$, both $A \times \bR$ and $\bR \times A \in S_{n+1}$. 
        \item [(c)] All semi-algebraic sets in $\bR^n$ are in $S_n$. 
        \item [(d)] For any $E \siq S_{n+1}$, $Pr_n(E) \siq S_n$ where $Pr_n: \bR^{n+1} \to \bR^n$ is the projection map onto the first $n$ coordinates of elements in $\bR^{n+1}$. 
    \end{itemize}
    For any $n$ and any $E \siq S_n$, we call $E$ a definable set. A structure is said to be o-minimal if the definable sets in $S_1$ are exactly finite unions of points and intervals. 
\end{defn}

\begin{defn}[definable functions, Definition 3.2 from \cite{AGhorbani}]\label{Defn Definable function}
    Given definable sets $A \siq \bR^m$ and $B \siq \bR^n$. A function $f: A \to B$ is called definable, if its graph ${(x, f(x)): x\in A} \siq \bR^m \times \bR^n$ is definable.  
\end{defn}

In \cite{AGhorbani}, it is said that Wilkie proved that the smallest structure containing the graph of $\exp(\cdot)$ is an o-minimal structure, whence the algebraic operations on $\exp(\cdot)$ are all definable functions. In particular, this include
\begin{itemize}
    \item [(a)] All commonly seen analytic activations, such as Sigmoid, Tanh, Softmax and Swish activations. 

    \item [(b)] All functions of the form $\sigma(x) = e^{x^\alpha} + e^{x^\beta}$, where $\alpha, \beta \in \bN$. 

    \item [(c)] The generalized neurons and neural networks with activations in (a) and/or (b). 

    \item [(d)] The derivatives (against both input and parameter) of functions in (c). 
\end{itemize}

An important result about definable functions is as follows. It is useful to us in Section \ref{Subsection Zsigma for activations vanishing at 0} and in Corollary \ref{Cor Limiting asymp of dim-1 parameterized function near 0}. 

\begin{thm}\label{Thm Zero set of definable function has finite components}
    The zero set of a definable function has finitely many components; in particular the zero sets of functions in the examples (a), (b), (c) and (d) above all have finitely many connected components. 
\end{thm}
\begin{proof}
    By \cite{AGhorbani}, every definable set has finitely many connected components. Thus, it suffices to show that the zero set of a definable function $f: A \to B$ is a definable set, where $A \siq \bR^m$ and $B \siq \bR^n$ for some $m, n \in \bN$. This is because the functions in examples (a), (b), (c) and (d) are all definable. By Definition \ref{Defn Definable function}, the graph of $f$, $\{(x, f(x)): x \in A\}$ is definable. Let $p: \bR^n \times \bR^m \to \bR$ be the polynomial $p(x, y) = |y|^2$. Since
    \begin{align*}
        f^{-1}(0) = \{(x, f(x)): x \in A\} \cap \{(x, 0): x \in \bR^m \} = \{(x, f(x)): x \in A\} \cap p^{-1}(0), 
    \end{align*}
    $f^{-1}(0)$ is definable by Definition \ref{Defn Ominimal structure}. \\
\end{proof}

\begin{lemma}\label{Lem finite analytic functions kill zeros locally}
    Given $N \in \bN$. Let $\{c_s\}_{s=0}^\infty$ are (real or complex) analytic functions defined on an open subset $\Omega$ of $\bF^N$. Denote their common zero set by $\calZ$. Given a bounded open subset $\Omega'$ of $\Omega \cut \calZ$, there is some $S \in \bN$ such that for any $\theta \in \Omega'$, we can find some $s \in \{0,1,..., S\}$ with $c_s(\theta)\ne 0$. 
\end{lemma}
\begin{proof}
    It suffices to prove that when $\calZ \ne \bF^N$, there is some $S \in \bN$ with $\left(\cap_{s=0}^S c_s^{-1}(0) \right) \cap B = \calZ \cap B$, where $B$ is any compact ball of $\bF^N$. For this we argue in an inductive way. First fix any $\theta_1 \in B\cut\calZ$. By hypothesis there is some $s_1 \in \bN$ with $c_{s_1}(\theta_1) \ne 0$. Thus, Lemma \ref{Lem analytic function zero set stratification} implies that $c_{s_1}^{-1}(0) \siq B$ is contained in $\calZ$ union a countable, locally finite collection of connected analytic submanifolds of $\bF^N$, each one (except $\calZ$) having dimension at most $N-1$; in particular, finitely many of them intersect $B$. \\

    Suppose that we have found integers $s_1, ..., s_{l(k)} \in \bN$ such that $c_{s_1}^{-1}(0) \cap ... \cap c_{s_{l(k)}}^{-1}(0)$ is $\calZ$ union a countable collection of connected analytic submanifolds of $\bF^N$ having dimension at most $N-k$, and finitely many of them intersect $B$, say
    \[
        c_{s_1}^{-1}(0) \cap ... \cap c_{s_{l(k)}}^{-1}(0) \cap B \siq \calZ \cup \left(\calM_1 \cup ... \cup \calM_{n(k)}\right)
    \]
    Here $n(k)$ and $l(k)$ are integers depending on $k$. We may assume that $\calM_j \ne \calZ$ for all $1 \le j \le n(k)$. For any $\calM_j$ and any $\theta \in \calM_j \cut \calZ$, there is a $s_{l(k) + j} \in \bN \cup\{0\}$ with $c_{s_{l(k)+j}}(\theta) \ne 0$. Since $\calM_j$ is an analytic manifold, this implies that $c_{s_{l(k)+j}}^{-1}(0) \cap \calM_j$ is contained in a countable union of submanifolds of $\calM_j$, each one having dimension $N-(k+1)$. Do this for each $1 \le j \le n(k)$ and set $l(k+1) = l(k) + n(k)$. Then clearly $c_{s_1}^{-1}(0) \cap ... \cap c_{s_{l(k+1)}}^{-1}(0)$ is contained in $\calZ$ union a countable collection of analytic submanifolds of $\bF^N$ having dimension at most $N - (k+1)$ . Define $f(\theta) = \sum_{j=1}^{l(k+1)} c_{s_j}(\theta)$, then 
    \[
        f^{-1}(0) = c_{s_1}^{-1}(0) \cap ... \cap c_{s_{l(k+1)}}^{-1}(0). 
    \]
    By Lemma \ref{Lem analytic function zero set stratification}, $f^{-1}(0)$ is a locally finite union of analytic submanifolds of $\bF^N$, whence our proof above simply says that each such manifold (except $\calZ$) in the union has dimension at most $N - (k+1)$, or finitely many intersect $B$. This completes the induction step. 
\end{proof}

This result can be strengthened when $\{c_s\}_{s=0}^\infty$ are definable functions. 

\begin{lemma}\label{Lem finite definable functions kills zeros globally}
    Given $N \in \bN$. Let $\{c_s\}_{s=0}^\infty$ are (real) definable functions on $\bF$. Denote their common zero set by $\calZ$. Then there is some $S \in \bN$ such that for any $\theta \in \bR^N \cut \calZ$, we can find some $s \in \{1, ..., S\}$ with $c_s(\theta) \ne 0$. 
\end{lemma}
\begin{proof}
    We first note there is some $S'\in \bN$ such that $c_0^{-1}(0) \cap ... \cap c_{S'}^{-1}(0) \cut \calZ$ is a (possibly) empty collection of discrete points. Indeed, if this is not true then for any $S' \in \bN$, $c_0^{-1}(0) \cap ... \cap c_{S'}^{-1}(0) \cut \calZ$ contains an analytic set $\calM$ of dimension at least 1. In other words, $\calM \siq \cap_{s=0}^\infty c_s^{-1}(0)$, which implies that $\calM \siq \calZ$, a contradiction. \\

    Fix this $S'$. Define $c:\bR^N \to \bR$ by $c(\theta) = \sum_{s=0}^{S'} |c_s(\theta)|^2$. Since $c$ is obtained by composing a definable function with definable functions, it is also definable, and clearly 
    \[
        c^{-1}(0) = c_0^{-1}(0) \cap ... \cap c_{S'}^{-1}(0). 
    \]
    Thus, by Theorem \ref{Thm Zero set of definable function has finite components}, $c^{-1}(0)$ has finitely many components, implying that $c_0^{-1}(0) \cap c_{S'}^{-1}(0) \cut \calZ$ is a finite set which we may denote by $\{p_1, ..., p_l\}$. For any $1 \le i \le l$, there is some $S_i \in \bN$ with $c_{S_i}(p_i) \ne 0$. Thus, by setting 
    \[
        S := \max\{S', S_1, ..., S_l\}, 
    \]
    we can see that $\left(\cap_{s=0}^S c_s^{-1}(0)\right) \cut \calZ = \emptyset$. 
\end{proof}

This result has many immediate corollaries about analytic parameterized functions, which we list below. They will be used in the proof for Proposition \ref{Prop Structure of calC II}. 

\begin{cor}\label{Cor Limiting asymp of dim-1 parameterized function near 0}
    Let $f: \bF^N \times \bF \to \bF$ be an analytic function which is not constant zero. Let $\calZ \siq \bF^N$ denote the set of parameters $\theta \in \bF^N$ such that $f(\theta, \cdot) \equiv 0$. Then any point $\theta^* \in \bF^N$ has a neighborhood $U(\theta^*) \siq \bF^N$ and a $S \in \bN$ such that for any $\theta \in U(\theta^*) \cut \calZ$, there are some $s \in \{0,1,...,S\}$ and $c_s(\theta) \ne 0$ with $f(\theta, x) \sim c_s(\theta)x^s$ as $x \to 0$. In particular, for any bounded open subset $\Omega \siq \bF^N$ we can find some $S \in \bN$ making the statement hold, and if $f$ is a definable function such that its derivatives are all definable, this holds for $\Omega = \bF^N \cut \calZ$.
\end{cor}
\begin{proof}
    Fix $\theta^* \in \bF^N$. Note that on a bounded product neighborhood $U(\theta^*) \times V(\theta^*) \in \bF^N \times \bF$ of $(\theta^*, 0)$ we can write $f(\theta, x)$ in its series expansion in $x$, namely 
    \[
        f(\theta, x) := \sum_{s=0}^\infty c_s(\theta) x^s, \quad (\theta, x) \in U(\theta^*) \times V(\theta^*). 
    \]
    Here the $c_s$'s are coefficients in the series expansion, each one being analytic in $\theta$. If $\theta \in \calZ$, $c_s(\theta) = 0$ for all $s \in \bN \cup \{0\}$; if $\theta \notin \calZ$, since $f(\theta, \cdot)$ is not constant zero, there must be some $s \in \bN \cup \{0\}$ with $c_s(\theta) \ne 0$. Therefore, the common zero set (restricted to $U(\theta^*)$) of $c_0, c_1, ...$ is just $\calZ \cap U(\theta^*)$. By Lemma \ref{Lem finite analytic functions kill zeros locally} above, we have some $S \in \bN$ such that for any $\theta \in U(\theta^*) \cut \calZ$, there is a smallest $s \in \{0,1,..., S\}$ with $c_s(\theta) \ne 0$. Then $f(\theta, x) \sim c_s(\theta)x^s$ as $x \to 0$. This proves the first part of the corollary. \\

    Now fix a bounded open $\Omega \siq \bF^N$. Note that $\overline{\Omega}$ is compact. For every point $\theta^* \in \overline{\Omega}$, we can apply the first part of this corollary to find a $U(\theta^*)$ and $S(\theta^*)$. Since finitely many $U(\theta^*)$'s cover $\overline{\Omega}$ (and thus $\Omega$), say $\overline{\Omega} \siq U(\theta_1^*) \cup ... \cup U(\theta_n^*)$. Set $S := \max\{S(\theta_1^*), ..., S(\theta_n^*)\}$. Given $\theta \in \Omega$, $\theta \in U(\theta_j^*)$ for some $1 \le j \le n$. Then by definition of $S(\theta_j^*)$, there are some $s \le S(\theta_j^*) \le S$ and $c_s(\theta) \ne 0$ with $f(\theta, x) \sim c_s(\theta) x^s$ as $x \to 0$. \\

    Finally, assume that $f$ and its derivatives are all definable functions. Define $c_s(\theta) := \frac{\partial^s f}{\partial x^s}(\theta, x) |_{x=0}$. Then $\{c_s\}_{s=1}^\infty$ are definable functions. Therefore, by Lemma \ref{Lem finite definable functions kills zeros globally} there is some $S \in \bN$ such that for any $\theta \in \bR^N \cut \calZ$, we can find some $s \in \{1, ..., S\}$ with $c_s(\theta) \ne 0$. But then we obtain again $f(\theta, x) \sim c_s(\theta)x^s$ as $x \to 0$. \\
\end{proof}

\subsection{Analytic Bump Functions}\label{Subsection Analytic bump functions}

Recall that a bump function $\xi: \bR \to [0, \infty)$ is a function such that $\xi = \text{const.}$ on an interval $[a,b]$ and $\xi = 0$ on $\bR\cut[a',b']$, where $[a',b']$ is some interval containing $[a,b]$ properly. Such functions can be made smooth \cite{JLeeSmooth}, but never analytic. Our goal in this part is to construct a $\xi$ such that $\xi \approx 1$ on $[a,b]$ and $\xi \approx 0$ on $\bR\cut[a',b']$, where the approximation can be arbitrarily precise. The key to this construction is the following lemma. 

\begin{lemma}\label{Lem Analytic Bump Function}
    Let $\rho: [0, \infty) \to (0, \infty)$ be a function which is continuous, non-decreasing and has finite intersection with any linear function on $\bR$. then there are $\lambda > 0$ and an interval $I \siq [0, \infty)$ containing the origin, such that the sequence of functions 
    \[
        \left\{ f_n: f_1 = 1 + \lambda \rho, f_n = 1 + \lambda \rho(f_{n-1}) \right\}_{n=1}^\infty
    \]
    satisfies the following properties: 
    \begin{itemize}
        \item [(a)] $\limftyn f_n(x)$ exists for every $x \in I$, and this limit is strictly greater than $1$. Moreover, if $\rho$ is strictly convex and continuously differentiable, there is some $L > 1$ such that $\{f_n\}_{n=1}^\infty$ converges uniformly to the constant function $x \mapsto L$ on $I$. 
        
        \item [(b)] Conversely, assume that $x = o(\rho(x))$ as $x \to \infty$. Then $\limftyn f_n(x) = \infty$ for all $x \in [0,\infty) \cut I$. Moreover, if $\rho$ is strictly increasing on $[0, \infty)$, $\{f_n\}_{n=1}^\infty$ diverges to $\infty$ uniformly on $[0,\infty)\cut I'$, where $I'$ is any interval properly containing $I$. 
    \end{itemize}
\end{lemma}
\begin{proof} We will first find intervals for (a) and (b) separately. As the intervals are the same, the desired result follows. 
\begin{itemize}
    \item [(a)] Fix $x^* \ge 0$. Suppose that the sequence $\{f_n(x*)\}_{n=1}^\infty$ converges, then using $\limftyn f_{n+1}(x^*) = \limftyn f_n(x^*)$ and the continuity of $\rho$, we obtain
    \[
        1 + \lambda \rho(\limftyn f_n(x^*)) = \limftyn f_n(x^*). 
    \]
    It follows that 
    \[
        \lambda = \frac{\limftyn f_n(x^*) - 1}{\rho(\limftyn f_n(x^*))}. 
    \]
    Thus, our goal is to find a $\lambda > 0$ and then a closed interval $I$ satisfying (a). Since $\rho$ is positive on $[0, \infty)$, the map $x \mapsto \frac{x-1}{\rho(x)}$ is well-defined on $[0, \infty)$, negative when $x < 1$ and positive when $x > 1$. For any positive number $\lambda < \sup_{x \ge 0}\frac{x-1}{\rho(x)}$ we have an equation in $x$
    \begin{equation}\label{eq 1 of Lem Analytic Bump Function}
         \frac{x-1}{\rho(x)} = \lambda \Longleftrightarrow \frac{1}{\lambda} (x - 1) = \rho(x). 
    \end{equation}
    Since $\rho$ has finite intersection with any linear function on $\bR$, there are finitely many roots for the equation above, say $L_1 < .... < L_n$. Clearly, $L_k > 1$ for all $k$. We would like to show that $\limftyn f_n(x^*)$ exists whenever $x^* \in [0, L_n]$. For this, consider the following cases: 
    \begin{itemize}
        \item [i)] $x^* \in [L_k, L_{k+1}]$ for some $1 \le k < n$. If $\frac{x^* - 1}{\rho(x^*)} \le \lambda$, then rearranging this inequality we obtain $x^* \le 1 + \lambda \rho(x^*)$. Moreover, since $\rho$ is non-decreasing, $\rho(x^*) \le \rho(L_{k+1})$ and thus (because $L_{k+1} > 1$)
        \[
            1 + \lambda \rho(x^*) = 1 + \frac{L_{k+1} - 1}{\rho(L_{k+1})} \rho(x^*) \le L_{k+1}. 
        \]
        Similarly, if $\frac{x^* - 1}{\rho(x^*)} \ge \lambda$, then $x^* \ge 1 + \lambda \rho(x^*)$ and $1 + \lambda \rho(x^*) \ge L_k$. By repeating this argument, we can immediately see that the sequence $\{f_n(x^*)\}_{n=1}^\infty$ is monotonic and thus converges. In fact, when $x^* \in (L_k, L_{k+1})$, our definition of $L_k$'s implies that all the inequalities above are strict, and in particular 
        \[
            \limftyn f_n(x^*) = \left\{
                                \begin{aligned}
                                    &L_{k+1}, &\frac{x^* - 1}{\rho(x^*)} < \lambda \\
                                    &L_k,     &\frac{x^* - 1}{\rho(x^*)} > \lambda
                                \end{aligned}\right..
        \]
        
        \item [ii)] $x^* \in (1, L_1)$. Then by definition of $L_1$ we must have $\frac{x^* - 1}{\rho(x^*)} < \lambda$. Argue in the same way as in i) we obtain $x^* < 1 + \lambda \rho(x^*)$ and $1 + \lambda \rho(x^*) < L_1$. Thus, by repeating this we see that $\{f_n(x^*)\}_{n=1}^\infty \siq (1, L_1)$ and is increasing. Therefore, it converges (to $L_1$). 

        \item [iii)]$x^* \in [0, 1]$. In ii) we have shown that $\{f_n(x^*)\}_{n=1}^\infty$ converges whenever $x \in (1, L_1]$. Therefore, as $\{f_n\}_{n=1}^\infty$ is defined inductively, it is enough to show $1 + \lambda \rho(x^*) \in (1, L_1]$. Since $\lambda > 0$ and $\rho(x^*) > 0$, we clearly have $1 + \lambda \rho(x^*) > 1$. On the other hand, since $\rho$ is increasing, $\rho(x^*) \le \rho(L_1)$ and thus 
        \[
            1 + \lambda \rho(x^*) = 1 + \frac{L_1 - 1}{\rho(L_1)}\rho(x^*) \le L_1. 
        \]
        This completes the proof. 
    \end{itemize}
    
    Now further assume that $\rho$ is strictly convex and continuously differentiable. Since $\rho'$ is strictly increasing and thus for some sufficiently small $\lambda$, the graph of $\frac{1}{\lambda}(x-1)$ is tangent to the graph of $\rho$. For this $\lambda$, equation (\ref{eq 1 of Lem Analytic Bump Function}) has precisely one root, say $L$. Then by ii) and iii) above, for any $x^* \in [0, L]$, we have $\limftyn f_n(x^*) = L$. We then show that $f_n(0) \le f_n(x)$ for all $x \in [0, L]$, so that the convergence of $\{f_n\}_{n=1}^\infty$ is uniform. This is clear for $n = 1$. Assume it holds for some $n \in \bN$, then 
    \begin{align*}
        f_{n+1}(x) 
        &= 1 + \lambda \rho(x) \\
        &\ge 1 + \lambda \rho(L') \\ 
        &= f_{n+1}(L'). 
    \end{align*}
    This completes the induction step. \\

    Therefore, given $\lambda$ such that equation (\ref{eq 1 of Lem Analytic Bump Function}) has roots, the interval $I$ for (a) can be $I = [0, L]$, where $L > 1$ is the largest root of (\ref{Lem Analytic Bump Function}). 

    \item [(b)] Let $\lambda$ and $L > 1$ be defined as in the last paragraph of the proof for (a). We would like to show that for any $x^* > L$, $\limftyn f_n(x^*) = \infty$. Since $x = o(\rho(x))$ as $x \to \infty$ and since $L$ is the largest root of equation (\ref{eq 1 of Lem Analytic Bump Function}), we must have $\limftyx \frac{x - 1}{\rho(x)} = 0$, which implies that $\frac{x-1}{\rho(x)} < \lambda$ for all $x > L$. Equivalently, similar as in (a)-i), 
    \[
        1 + \lambda \rho(x^*) > x^*. 
    \]
    By repeating this argument we can see that $\{f_n(x^*)\}_{n=1}^\infty$ is strictly increasing. Assume that its supremum is $c \in (L, \infty)$. Let 
    \[
        \delta := \frac{L-1}{\rho(L)} - \sup_{x \ge (L+c)/2} \frac{x - 1}{\rho(x)}
    \]
    and let 
    \[
        \vep := \inf_{x \ge 0} \rho(x). 
    \]
    By the continuity of $\rho$, both $\delta$ and $\vep$ are positive. For any large $n$ with $f_n(x^*) > \frac{L+c}{2}$ and $|f_n(x^*) - c| < \delta \vep$ we have 
    \begin{align*}
        f_{n+1}(x^*) - f_n(x^*) 
        &= 1 + \lambda \rho(f_n(x^*)) - f_n(x^*) \\
        &= \frac{L-1}{\rho(L)} \rho(f_n(x^*)) - (f_n(x^*) - 1) \\ 
        &= \left[ \frac{L-1}{\rho(L)} - \frac{f_n(x^*)-1}{\rho(f_n(x^*))} \right] \rho(f_n(x^*)) \\
        &\geq \delta \vep, 
    \end{align*}
    It follows that $f_{n+1}(x^*) > c$, a contradiction. Therefore, $\{f_n(x^*)\}_{n=1}^\infty$ must diverge to $\infty$ as $n \to \infty$. \\ 

    Now further assume that $\rho$ is increasing on $[0, \infty)$. 
    Arguing in the same way as in (a), we can see that for any $L' > L$, $f_n(x) \ge f_n(L')$ for all $x \ge L'$. Since $\limftyn f_n(L') = \infty$, $\{f_n\}_{n=1}^\infty$ must diverge to $\infty$ uniformly on $[L', \infty)$. \\

    Therefore, the interval $I = [0, L]$ (for prescribed $\lambda$) in (a) also works for (b). 
\end{itemize}
\end{proof}

It is easy to construct examples for Lemma \ref{Lem Analytic Bump Function}. We give two below. 
\begin{itemize}
    \item [(a)] Consider $\rho(x) = e^x$. Then clearly $\rho$ is continuous, increasing, strictly convex, continuously differentiable and satisfies $x = o(\rho(x))$ as $x \to \infty$. To find the interval $I$, we shall work with the derivative of $q(x) := \frac{x-1}{\rho(x)}$: we set 
    \[
        q'(x) = e^{-x}(2-x) = 0, \quad (x > 1). 
    \]
    Then $x = 2$. This suggests us to set $\lambda = e^{-2}$ and $I = [0, 2]$. We can also see from Figure \ref{Figure Illustration example (a): rho(x) = exp} that $\{f_n\}_{n=1}^\infty$ indeed converges on $I$ and diverges on $[0, \infty) \cut I$. 

    \item [(b)] Consider $\rho(x) = e^{x^2}$. Again, $\rho$ is continuous, increasing, strictly convex, continuously differentiable and satisfies $x = o(\rho(x))$ as $x \to \infty$. Following the same procedure as in (a) we can find 
    \[
        \lambda = \exp\left(-\frac{2+\sqrt{3}}{2}\right) \frac{\sqrt{3}-1}{2}, \quad I = \left[0, \frac{1 + \sqrt{3}}{2} \right]. 
    \]
\end{itemize}
Note that the sequence $\{f_n\}_{n=1}^\infty$ for both (a) and (b) can be extended to $\bR$, simply by setting $f_n(-x) = f_n(x)$ for all $x < 0$. But in this way, the extension for (a) is only analytic on $\bR\cut\{0\}$ and piecewise analytic on $\bR$, while the extension for (b) is analytic on the whole real line. In general, if $\rho$ an analytic function (i.e., it is analytic on a neighborhood of $[0, \infty)$) with these properties, we can define $\Tilde{\rho}: \bR \to \bR$ by $\Tilde{\rho}(x) = \rho(x^2)$. Then $\Tilde{\rho}$ is analytic on $\bR$ and also satisfies these properties. \\

\begin{figure}[H]
    \centering
    \includegraphics[width=0.65\textwidth]{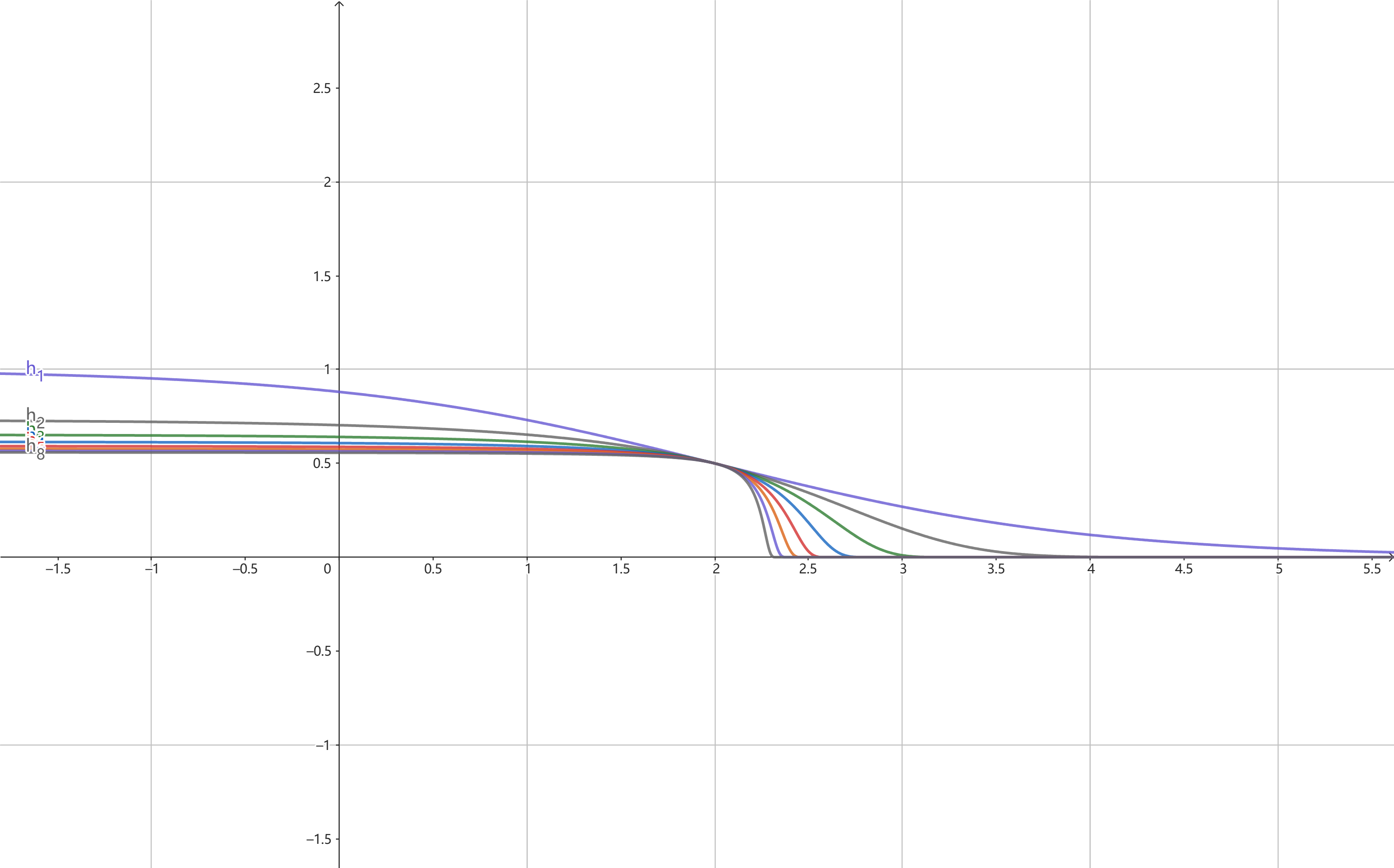}
    \caption{Illustration of example (a): how to construct the function sequence $\{f_n\}_{n=1}^\infty$ for $\rho(x) = e^x$.}
    \label{Figure Illustration example (a): rho(x) = exp}
\end{figure}

\begin{cor}[Analytic bump function]\label{Cor Analytic bump function}
    Given any $\vep > 0$ and $a, a', b, b' \in \bR$ with $a' < a < b < b'$, there is a positive $\xi \in \calA(\bR)$ such that $\norm{\xi - 1}_{\infty, [a, b]} < \vep$, $\norm{\xi}_{\infty, \bR\cut[a',b']} < \vep$ and $\norm{\xi}_\infty < 1 + \vep$. Moreover, for any $n \in \bN$, $\xi$ can be chosen such that 
    \[
        \xi^{(s)}(x) = o\left(\underbrace{\exp(-\exp(...\exp(|x|)...))}_{\text{$n$ ``$\exp$"'s}}\right) 
    \]
    as $x \to \pm\infty$, and 
    \[
        \norm{\xi^{(s)}}_{\infty, [a,b]} < \vep 
    \]
    for all $0 \le s \le S$. 
\end{cor}
\begin{proof}
    By Lemma \ref{Lem Analytic Bump Function} and the remark above, there is some $\rho: \bR \to \bR$ which is analytic, non-decreasing, strictly convex, and $x = o(\rho(x))$ as $x \to \infty$. For this $\rho$ there are some $\lambda$ and some $L > 1$ such that
    \begin{itemize}
        \item [(a)] $\limftyn f_n = L$ uniformly on $[-L,L]$. 
        \item [(b)] $\limftyn f_n(x) = \infty$ for all $x \in \bR\cut[-L,L]$, and $\{f_n\}_{n=1}^\infty$ diverges to $\infty$ uniformly on $\bR\cut I'$, where $I'$ is any interval properly containing $[-L,L]$. 
        \item [(c)] For each $n \in \bN$, $f_n$ is bounded below by $f_n(0)$. 
    \end{itemize}
    Here the sequence $\{f_n\}_{n=1}^\infty$ is defined as in Lemma \ref{Lem Analytic Bump Function}: 
    \[
        f_1(x) = 1 + \lambda \rho(x), \quad f_{n+1}(x) = 1 + \lambda \rho(f_n(x))\,\,\,\, \forall\, n\in \bN. 
    \]
    Define a sequence of functions 
    \[
        g_n(x) := \frac{1}{L} f_n\left(\frac{b - a}{2L} x + a\right), \quad n \in \bN
    \]
    and then a sequence of functions 
    \[
        \xi_n(x) := \frac{1}{g_n(x)}, \quad n \in \bN. 
    \]
    Clearly, each $g_n$ is just obtained using a translation and a stretch from $f_n$. Thus, for $\{\xi_n\}_{n=1}^\infty$ we have 
    \begin{itemize}
        \item [(a)] $\limftyn \xi_n = 1$ uniformly on $[a, b]$. 
        \item [(b)] $\limftyn \xi_n(x) = 0$ for all $x \in \bR\cut[a,b]$, and $\{\xi_n\}_{n=1}^\infty$ converges uniformly on $\bR\cut[a', b']$. 
        \item [(c)] $\xi_n > 0$ for all $n \in \bN$. 
    \end{itemize}
    This means for sufficiently large $n$ we have, by setting $\xi := \xi_n$, both $\norm{\xi - 1}_{\infty, [a,b]} < \vep$ and $\norm{\xi}_{\infty, \bR\cut[a',b']} < \vep$, proving the first part of the statement. \\

    For the second part, we simply set $\rho(x) = e^{x^2}$. Then it is clear that for each $n \in \bN$ and $s \in \bN$, 
    \[
        \xi_n^{(s)}(x) = o\left(\underbrace{\exp(-\exp(...\exp(|x|)...))}_{\text{$n$ ``$\exp$"'s}}\right)
    \]
    as $x \to \pm \infty$, and by Lemma \ref{Lem Analytic Bump Function} (b), $\limftyn \norm{\xi_n^{(s)}}_{\infty, \bR\cut [a', b']} = 0$. Then we prove that for any $s \in \bN$ we have $\limftyn \norm{\xi_n^{(s)}}_{\infty, [a,b]} = 0$. By looking at our construction, it suffices to prove that (see also example (b) above)
    \[
        \limftyn \norm{f_n^{(s)}}_{\infty, [-\tx^*, \tx^*]} = 0, \quad \tx^* \in [0, x^*],
    \]
    where $x^* = \frac{1 + \sqrt{3}}{2}$. To illustrate this we only prove the case for $s = 1$; the general case can be proved in a similar way. Note that each $f_n'$ is odd and increasing on $[0, \infty)$, because each $f_n$ is even and 
    \[
        f_n'= \lambda \rho'(f_{n-1}) \cdot ... \cdot \lambda \rho'. 
    \]
    Fix $\tx^* \in [0, x^*]$ and assume that there is a sequence of indices $\{n_k\}_{k=1}^\infty$ such that $\limftyk f_{n_k}'(\tx^*) \ge C$ for some $C > 0$. Since each $f_{n_k}'$ is increasing on $[0, \infty)$, for any $x \in [\tx^*, x^*)$ we must have $\limftyk f_{n_k}'(x) \ge C$ as well. But then 
    \begin{align*}
        \varliminf_{k \to \infty} f_{n_k}(x^*) 
        &\ge \varliminf_{k \to \infty} f_{n_k}(\tx^*) + \varliminf_{k \to \infty} \int_{\tx^*}^{x^*} f_{n_k}'(x) dx \\
        &= L + C(x^* - \tx^*) > L, 
    \end{align*}
    contradicting $\limftyn f_n(x^*) = L$ by Lemma \ref{Lem Analytic Bump Function}. Therefore, $\limftyn f_n'(\tx^*) = 0$. Since each $f_n'$ is odd and increasing on $[0, \infty)$, $\limftyn \norm{f_n'}_{\infty, [-\tx^*, \tx^*]} = 0$ and thus for sufficiently large $n$, $\norm{\xi_n'}_{\infty, [a,b] \cup \bR\cut[a',b']} < \vep$ must hold. \\
\end{proof}

\begin{prop}\label{Prop function conca and S-order approx}
    Let $\sigma, \sigma_0 \in \calA(\bR)$ such that there are some $S, n \in \bN \cup \{0\}$ with 
    \[
        \sigma^{(s)}(x), \sigma_0^{(s)}(x) = o\left(\underbrace{\exp(\exp(...\exp(|x|)...))}_{\text{$n$ ``$\exp$"'s}}\right)
    \]
    for all $s \in \{0,1,...,S\}$, as $x \to \pm\infty$. Given real numbers  $a,a', b,b'$ with $a' < a < b < b'$, for any $\vep > 0$ there is an analytic function $\tsigma \in \calA(\bR)$ having the following properties: 
    \begin{itemize}
        \item [(a)] $\norm{\tsigma - \sigma}_{S [a,b]} < \vep$ and $\norm{\tsigma - \sigma_0}_{S, \bR\cut[a',b']} < \vep$ for all $s \in \{0,1,...,S\}$. 

        \item [(b)] Fix $s \in \{0,1,...,S\}$. If we further assume that $\norm{\sigma - \sigma_0}_{S, [a', b']} < \vep$, then there is some $C_s > 0$ such that $\norm{\sigma - \sigma_0}_S < C_s \vep$. \\
    \end{itemize}
\end{prop}
\begin{proof}
\begin{itemize}
    \item [(a)] By Corollary \ref{Cor Analytic bump function}, there is an analytic bump function $\xi \in \calA(\bR)$ such that $\norm{\xi - 1}_{\infty, [a,b]} < \vep'$, $\norm{\xi}_{\infty, \bR\cut [a',b']} < \vep'$ and for any $s \in \{0,1,...,S\}$, $\norm{\xi^{(s)}}_{\infty, [a,b] \cup \bR\cut[a',b']} < \vep'$ and
    \[
        \xi^{(s)}(x) = o\left(\underbrace{\exp(-\exp(...\exp(|x|)...))}_{\text{$n+1$ ``$\exp$"'s}}\right), 
    \]
    where $\vep' > 0$ is a small number to be determined. Define $\tsigma$ by 
    \begin{equation}
        \tsigma(x) := \xi(x) \sigma(x) + (1 - \xi(x)) \sigma_0(x). 
    \end{equation}
    Fix $s \in \{0,1,...,S\}$. By inductively applying Leibniz rule for differentiation we have 
    \begin{align*}
        \tsigma^{(s)}(x) 
        &= \sum_{k=0}^s \binom{s}{k} \xi^{(k)}(x) \sigma^{(s-k)}(x) + \sum_{k=1}^s \binom{s}{k} (1 - \xi(x))^{(k)} \sigma_0^{(s-k)}(x) \\ 
        &= \xi(x) \sigma^{(s)}(x) + (1 - \xi(x)) \sigma_0^{(s)}(x) + \sum_{k=1}^s \binom{s}{k} \xi^{(k)}(x) \left[ \sigma^{(s-k)}(x) - \sigma_0^{(s-k)}(x) \right]. 
    \end{align*}
    For any $x \in \bR$ we have 
    \begin{align*}
        |\tsigma^{(s)}(x) - \sigma^{(s)}(x)| &\le |1 - \xi(x)| |\sigma^{(s)}(x) - \sigma_0^{(s)}(x)| + \sum_{k=1}^s \binom{s}{k} \xi^{(k)}(x) |\sigma^{(s-k)}(x) - \sigma_0^{(s-k)}(x)| \\
        |\tsigma^{(s)}(x) - \sigma_0^{(s)}(x)| &\le \xi(x) |\sigma^{(s)}(x) - \sigma_0^{(s)}(x)| + \sum_{k=1}^s \binom{s}{k} \xi^{(k)}(x) |\sigma^{(s-k)}(x) - \sigma_0^{(s-k)}(x)|. 
    \end{align*}
    Therefore, by our construction of $\xi$, we have, for any $x \in [a,b]$ 
    \begin{align*}
        |\tsigma^{(s)}(x) - \sigma^{(s)}(x)| 
        &< \vep' |\sigma^{(s)}(x) - \sigma_0^{(s)}(x)| + \vep' \sum_{k=1}^s \binom{s}{k} |\sigma^{(s-k)}(x) - \sigma_0^{(s-k)}(x)| \\ 
        &\le \vep' \norm{\sigma - \sigma_0}_{S, [a,b]} + \vep' 2^S \norm{\sigma - \sigma_0}_{S, [a,b]}. 
    \end{align*}
    Second, by hypothesis of $\sigma^{(s)}, \sigma_0^{(s)}$ and our construction of $\xi$ there is some $R > 0$ such that whenever $|x| > R$, 
    \begin{align*}
        |\tsigma^{(s)}(x) - \sigma_0^{(s)}(x)| 
        &< e^{-u} u + \sum_{k=1}^s \binom{s}{k} e^{-u} u \\ 
        &\le (1 + 2^S) e^{-u} u \\
        &< \vep, 
    \end{align*}
    where 
    \[
        u := \underbrace{\exp(\exp(...\exp(|x|)...))}_{\text{$n$ ``$\exp$"'s}},  
    \]
    and for any $x \in [-R,R] \cut [a',b']$, 
    \begin{align*}
        |\tsigma^{(s)}(x) - \sigma_0^{(s)}(x)| 
        &< \vep' \norm{\sigma - \sigma_0}_{S, [-R,R]} + \vep' 2^S \norm{\sigma - \sigma_0}_{S, [-R,R]}. 
    \end{align*}
    Thus, by setting $\vep' = \vep \left( S(1 + 2^S) \norm{\sigma - \sigma}_{S, [a, b] \cup [-R,R] } \right)^{-1}$, we obtain the desired estimates for both $\norm{\tsigma - \sigma}_{S, [a,b]}$ and $\norm{\tsigma - \sigma_0}_{S, \bR\cut[a',b']}$. 

    \item [(b)] Let $\xi$ be an analytic bump function defined in (a) and $\tsigma(x) := \xi(x) \sigma(x) + (1 - \xi(x)) \sigma_0(x)$. Fix $s \in \{0,1,...,S\}$. By (a), we need only investigate $|\tsigma^{(s)}(x) - \sigma_0^{(s)}(x)|$ for $x \in [a',b']$. Since the derivatives of$\xi$ are bounded on compact subsets of $\bR$, for any such $x$ we have 
    \begin{align*}
        |\sigma^{(s)} - \sigma_0^{(s)}| 
        &< \norm{\xi}_{\infty, [a', b']} \vep + \sum_{k=1}^s \binom{s}{k} \norm{\xi^{(k)}}_{\infty, [a', b']} \vep \\ 
        &\le (1 + 2^S) \norm{\xi}_{S, [a', b']} \vep. 
    \end{align*}
    This completes the proof. 
\end{itemize}
\end{proof}

\begin{figure}[H]
    \centering
    \includegraphics{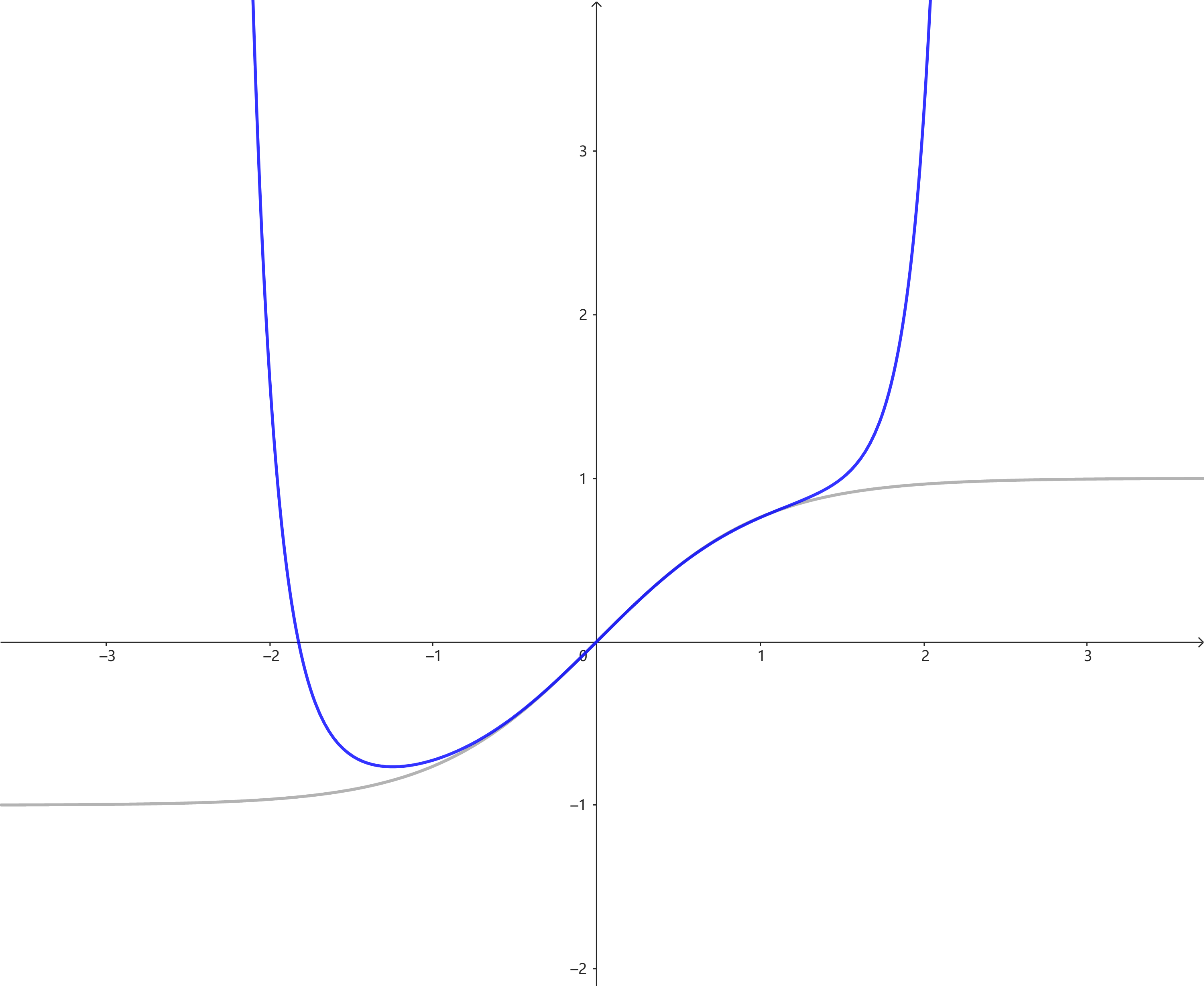}
    \caption{Construction of an analytic function $\tsigma$ that approximates Tanh activation on an interval around 0, following Proposition \ref{Prop function conca and S-order approx} (a). Here we use $\sigma(x) = e^{x^2}$ and $\zeta_4$ defined as in Corollary \ref{Cor Analytic bump function}, with ``base function" $f(x) = e^{|x|}$.}
    \label{Figure Good Analytic Function for Tanh}
\end{figure}

\begin{figure}[H]
    \centering
    \includegraphics{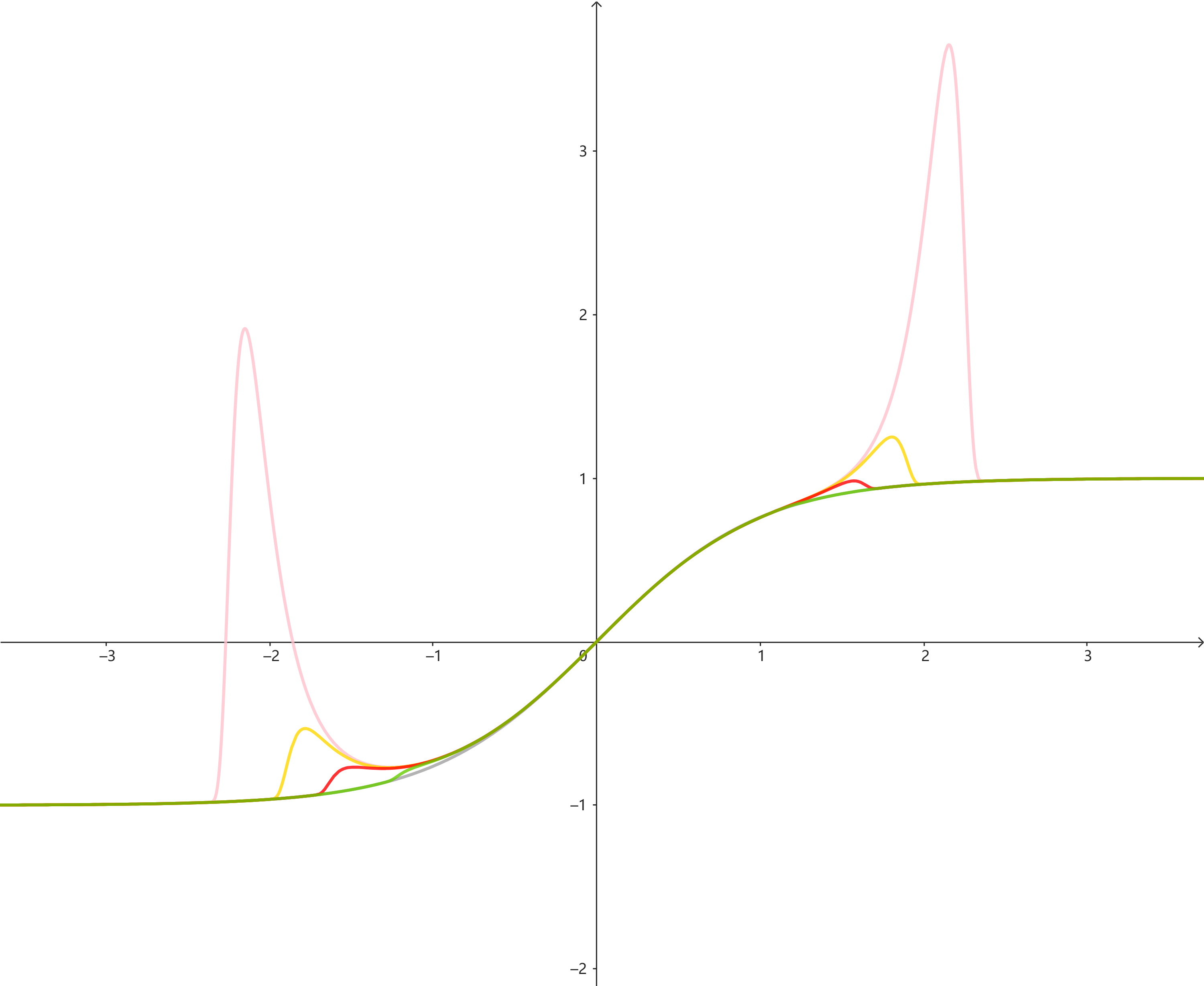}
    \caption{Construction of an analytic function $\tsigma$ that approximates Tanh activation globally on $\bR$, following Proposition \ref{Prop function conca and S-order approx} (b). In the construction, we use $\sigma$ constructed as in Figure \ref{Figure Good Analytic Function for Tanh}. $\zeta$ is a scaling of an analytic bump function $\zeta_5$ defined in Corollary \ref{Cor Analytic bump function} with ``base function" $f(x) = e^{|x|}$. Precisely, each function takes the form $\sigma(x) = \zeta_5(\alpha x) [\zeta_4(x) \tanh(x) + (1-\zeta_4(x))\tanh(x)] + (1 - \zeta_5(\alpha x)) e^{x^2}$, where $\alpha = 1.1, 1.3, 1.5, 2$ for the \cha{pink}, \textcolor{yellow}{yellow}, \textcolor{orange}{orange}, and \green{green} curves, respectively. As we can see, the approximation is almost indistinguishable from Tanh when $\alpha = 2$.}
    \label{Figure Good Approximate Function for Tanh}
\end{figure}

\section{Theory of General Neurons}\label{Section Theory of general neurons}

This section investigates the linear independence / linear dependence of neurons with general structures, i.e., arbitrary layer and widths. As mentioned in Section \ref{Section Intro}, this depends on both the parameters and choice of activation functions. We will first show the existence of activation function for which the set of parameters (which we denote by $\calZ_\sigma$, see Definition \ref{Defn Minimal zero set}) making an NN constant zero is ``minimal". In Corollary \ref{Cor Structure of Z I}, we analyze the geometric structure of $\calZ_\sigma$ for given $\sigma$ and network structure. In general, this set varies as the activation function $\sigma$ and network structure changes; however, they still admit some geometries which we characterize. in Corollary \ref{Cor Structure of Z I}. We also notice that the requirement $\sigma(0) = 0$ offers us extra good structure of $\calZ_\sigma$, so we further discuss networks with such activation functions in Section \ref{Subsection Zsigma for activations vanishing at 0}. \\

Recall that a fully-connected neural network $H$ with network structure $\{m_l\}_{l=1}^L$ has the form 
\begin{align*}
    H(\theta, z) = H^{(L)}(\theta, z) 
    &= \sum_{j=1}^{m_{L-1}} a_j H^{(L-1)}(\theta, z) + b \\ 
    &= \sum_{j=1}^{m_{L-1}} a_j \sigma\left( w_j^{(L-1)} H^{(L-2)}(\theta, z) + b_j^{(L-1)}\right) + b. 
\end{align*}
If we further assume that for some $m \le m_{L-2}$, $H_1^{(L-2)}(\theta, \cdot), ..., H_m^{(L-2)}(\theta, \cdot)$ are the non-constant components of $H^{(L-2)}$, then we can rewrite each component of $H^{(L-1)}$ as 
\begin{equation}\label{eq for rewritting neuron}
    H_j^{(L-1)}(\theta, z) = \sigma\left( \sum_{k=1}^m w_{jk}^{(L-1)} H_k^{(L-2)}(\theta, z) + \left[ \sum_{k>m} w_{jk}^{(L-1)} H_k^{(L-2)}(\theta, z) + b_j^{(L-1)} \right] \right). 
\end{equation}
For more notations and definitions about neural networks, see Section \ref{Section Notations and assumptions}. \\

\begin{thm}[Linear independence of fully-connected neurons]\label{Thm Lin ind of full-connected neurons I}
    There is a $\sigma \in \calA(\bR)$ with the following property: given $d, L \in \bN$ and $m_1, ..., m_L \in \bN$, consider the fully-connected NN, $H$, with network structure $\{m_l\}_{l=1}^L$. Then the generalized neurons $\{H_k^{(L-1)}(\theta,\cdot)\}_{k=1}^{m_{L-1}}$ are linearly independent if and only if 
    \begin{itemize}
        \item [(a)] The vectors 
        \[
            \left( \left(w_{jk}^{(L-1)} \right)_{k=1}^m, \sum_{k > m} w_{jk}^{(L-1)} H_k^{(L-2)}(\theta, z) + b_j^{(L-1)} \right), \quad 1 \le j \le m_{L-1}
        \]
        are distinct, and 
        \item [(b)] We have $\left( w_{jk}^{(L-1)} \right)_{k=1}^m = 0$ for at most one $j \in \{1, ..., m_{L-1}\}$. Moreover, for this $j$, $H_j^{(L-1)}(\theta, \cdot)$ is not constant zero. 
    \end{itemize}
    Furthermore, the same result holds for NN without bias. 
\end{thm}
\begin{proof}
    For simplicity, for each $1 \le j \le m_{L-1}$ we introduce the notations 
    \[
        \tw_j = (\tw_{jk})_{k=1}^m := \left(w_{jk}^{(L-1)} \right)_{k=1}^m, \quad \tb_j := \sum_{k > m} w_{jk}^{(L-1)} H_k^{(L-2)}(\theta, z) + b_j^{(L-1)}. 
    \]
    With these notations, we can write each generalized neuron $H_j^{(L-2)}$ as 
    \[
        H_j^{(L-1)}(\theta, z) = \sigma\left( \sum_{k=1}^m \tw_{jk} H_k^{(L-2)}(\theta, z) + \tb_j \right). 
    \]
    This clearly relates us to the theory about functions of ordered growth. In fact, we will use Proposition \ref{Prop Functions of ordered growth II} to complete the proof, arguing in an inductive way. For this, consider any activation $\sigma \in \calA(\bR)$ with $\sigma(x) = \Theta\left( \exp(e^{x^5}) \right)$ as $x \to \infty$ and $\sigma(x) = \Theta\left( \exp(e^{-x^3}) \right)$ as $x \to -\infty$. Notice that $\sigma$ has the following properties: 
    \begin{itemize}
        \item [i)] $\limftyz \sigma(z) = \lim_{z \to -\infty} \sigma(z) = \infty$, both at hyper-exponential rate. Moreover, we also have $\limftyz \log \sigma(z) = \lim_{z\to-\infty} \log \sigma(z) = \infty$. 

        \item [ii)] Given $f_1, f_2: \bR \to [0, \infty)$ diverging to $\infty$ of ordered growth such that $f_2(x) = o(f_1(x))$ as $x \to \infty$, we have, for any $w > 0$ and $\tw < 0$, $\sigma\left(w e^{f_2(x)}\right) = o\left( \sigma\left( \tw e^{f_1(x)}\right) \right)$ as $x \to \infty$. Indeed, by the construction of $\sigma$, there are $C_1^+, C_2^+, C_1^- C_2^- > 0$ such that 
        \begin{align*}
            C_1^+ \exp(e^{z^5}) \le &\,\sigma(z) \le C_2^+ \exp(e^{z^5})\\
            C_1^- \exp(e^{-z^3}) \le &\,\sigma(z)\le C_2^- \exp(e^{-z^3})
        \end{align*}
        as $z \to \infty$ and $z \to -\infty$, respectively. Therefore, as $x \to \infty$
        \begin{align*}
            \sigma\left( w e^{f_2(x)}\right) 
            &\le C_2^+ \exp\left( \exp\left( w^5e^{5f_2(x)} \right)\right) \\
            &\ll  C_1^- \exp\left( \exp\left( -\tw^3 e^{3f_1(x)} \right)\right) \le \sigma\left( \tw e^{f_1(x)} \right), 
        \end{align*}
        where we used $\limftyx (3f_1(x) - 5f_2(x)) = \infty$ and (a) to do the approximations. 
    \end{itemize}

    First we show that given $(\tw_j, \tb_j)_{j=1}^{m_{L-1}}$ satisfying (a) and (b) of this theorem, the generalized neurons $H_1^{(L-1)}, ..., H_{m_{L-1}}^{(L-1)}$ are linearly independent. Without loss of generality, assume that $d = 1$. The case for $L=1$ is trivial. For $L = 2$, given distinct $(w_k, b_k)_{k=1}^{m_1}$ such that at most one of the $w_k$ is $0$, say $w_{m_1} = 0$. Then $\{\sigma(w_k x + b_k)\}_{k=1}^{m_1-1}$ all diverge to $\infty$ and have ordered growth (along the real line), while $\sigma(w_{m_1}x + b_{m_1}) \equiv \sigma(b_{m_1})$. To see this, fix any distinct $k_1, k_2 \in \{1, ..., m_1-1\}$ and consider the following cases: 
    \begin{itemize}
        \item [i)] $w_{k_1}, w_{k_2}$ are both positive or negative. If $w_{k_1} > w_{k_2} > 0$ then
        \begin{align*}
            \sigma(w_{k_2}x + b_{k_2}) 
            &\le C_2^+ \exp\left( e^{(w_{k_2}x + b_{k_2})^5} \right) \\
            &\ll  C_1^+ \exp\left( e^{(w_{k_1}x + b_{k_1})^5} \right) \le \sigma(w_{k_1}x + b_{k_1}), 
        \end{align*}
        or $\sigma(w_{k_2}x + b_{k_2}) = o\left( \sigma(w_{k_1}x + b_{k_1}) \right)$ as $x \to \infty$. If $w_{k_1} < w_{k_2} < 0$, then 
        \begin{align*}
            \sigma(w_{k_2}x + b_{k_2}) 
            &\le C_2^- \exp\left( e^{-(w_{k_2}x + b_{k_2})^3} \right) \\
            &\ll  C_1^-\exp\left( e^{-(w_{k_1}x + b_{k_1})^3} \right) \le \sigma(w_{k_1}x + b_{k_1}), 
        \end{align*}
        or $\sigma(w_{k_2}x + b_{k_2}) = o\left( \sigma(w_{k_1}x + b_{k_1}) \right)$ as $x \to \infty$. The same argument works for $w_{k_2} > w_{k_1} > 0$ and $w_{k_2} < w_{k_1} < 0$. 

        \item [ii)] $w_{k_1}, w_{k_2}$ have different signs. If $w_{k_1} > 0$ and $w_{k_2} < 0$, then 
        \begin{align*}
            \sigma(w_{k_2}x + b_{k_2}) 
            &\le C_2^- \exp\left( e^{-(w_{k_2}x + b_{k_2})^3} \right) \\
            &\ll  C_1^+ \exp\left( e^{(w_{k_1}x + b_{k_1})^5} \right) \le \sigma(w_{k_1}x + b_{k_1}), 
        \end{align*}
        or $\sigma(w_{k_2}x + b_{k_2}) = o\left( \sigma(w_{k_1}x + b_{k_1}) \right)$ as $x \to \infty$. The same argument works for $w_{k_1} < 0$ and $w_{k_2} > 0$. 

        \item [iii)] $w_{k_1} = w_{k_2}$. Then we must have $b_{k_1} \ne b_{k_2}$. Again, using property (a) of $\sigma$ above, we can easily deduce that these two neurons have ordered growth. 
    \end{itemize}
    This proves the case for $L=1$-layer neurons. Moreover, using the same argument we can show that the functions $\{\log\left( \sigma(w_k x+ b_k)\right)\}_{k=1}^{m_1-1}$ all diverge to $\infty$ and have ordered growth. \\

    Now consider $L \ge 3$. Suppose that $\log H_1^{(L-2)}(\theta, \cdot), ..., \log H_m^{(L-2)}(\theta, \cdot)$ all diverge to $\infty$ and have ordered growth such that $\log H_k^{(L-2)}(\theta, \cdot) = o(\log H_{k-1}^{(L-2)}(\theta, \cdot))$ for all $2 \le k \le m$, and the functions $\{ \log H_k^{(L-2)}(\theta, \cdot)\}_{k=m+1}^{m_{L-2}}$ are constant. Assume that $(\tw_j, \tb_j)$'s satisfy requirement (a) and (b) above, in particular, $(\tw_j, \tb_j)$'s are distinct. \\
    
    If $\tw_j$'s are all non-zero, then by applying property (b) of $\sigma$ to $f_1 = \log H_1^{(L-2)}(\theta, \cdot)$ and $f_2 = \log H_m^{(L-2)}(\theta, \cdot)$, we can see that both requirements (a) and (b) in Proposition \ref{Prop Functions of ordered growth II} are satisfied, thus the $H_j^{(L-1)}(\theta, \cdot)$'s, $1 \le j \le m_{L-1}$ all diverge to $\infty$ and have ordered growth, and so are the $\log H_j^{(L-1)}(\theta, \cdot)$'s. If there is one $\tw_j = 0$, we may, by rearranging the indices if necessary, assume that $\tw_{m_{L-1}} = 0$. Then $H_{m_{L-1}}^{(L-1)}(\theta, \cdot)$ and $\log H_{m_{L-1}}^{(L-1)}(\theta, \cdot)$ are constant functions, while by applying Proposition \ref{Prop Functions of ordered growth II} we see that the $H_j^{(L-1)}(\theta, \cdot)$'s, $1 \le j \le m_{L-1}-1$ all diverge to $\infty$ and have ordered growth, and so are the $\log H_j^{(L-1)}(\theta, \cdot)$'s. Then it follows from Proposition \ref{Prop Ordered growth implies linear independence} (and the fact that at most one $H_j^{(L-1)}(\theta, \cdot)$ is constant) that $H_j^{(L-1)}(\theta, \cdot), ..., H_{m_{L-1}-1}^{(L-1)}(\theta, \cdot)$ are linearly independent. \\

    Conversely, if $(\tw_j, \tb_j)_{j=1}^{m_{L-1}}$ are not distinct, we clearly have two identical generalized neurons, and if there are distinct $j_1, j_2 \in \{1, ..., m_{L-1}\}$ such that $\tw_{j_1} = \tw_{j_2} = 0$, then $H_{j_1}^{(L-1)}(\theta, \cdot)$ and $H_{j_2}^{(L-1)}(\theta, \cdot)$ are both constant functions. Neither cases give linearly independent neurons. This completes the proof. 
\end{proof}
\begin{remark}
    Several remarks on this result. First, note that if $\left(w_j^{(L-1)}, b_j^{(L-1)} \right)_{j=1}^{m_{L-1}}$ does not satisfy (a) or (b), then the $H_j^{(L-1)}$'s are \textit{linearly dependent for all $\sigma$}. Second, the choice of activation function can be extended to any $\sigma$ with 
    \[
        \sigma(z) = \Theta\left(\exp(e^{z^p})\right), \quad \sigma(z) = \Theta\left(\exp(e^{-z^q})\right)
    \]
    as $z \to \infty$ and $z \to -\infty$, respectively, where $p,q > 1$ are distinct odd numbers. Finally, this theorem actually characterizes the parameters which yield linearly independent neurons in an inductive way. As we can see, requirement (a) and (b) of $w_j^{(L-1)}$'s and $b_j^{(L-1)}$'s depend on $m$, which is determined by the parameters of $H^{(L-2}$. Thus, to fully characterize the set of parameters which yield linearly independent neurons, we must then work with $H^{(L-2)}$, then $H^{(L-3)}$, ..., by applying this theorem for $L-1$ times. It is then clear that for both NN with and without bias, the structure of this set of parameters depends on activation function $\sigma$ and network structure $\{m_l\}_{l=1}^L$. One important and special case is $\sigma(0) = 0$, as the set stabilizes over generic choices of such activations. To further our investigation of these topics, we introduce some notations below.\\
\end{remark}

\begin{defn}[Minimal zero set]\label{Defn Minimal zero set}
    Given $d, L \in \bN$ and $m_1, ..., m_L \in \bN$. Consider $H$ (with or without bias) with network structure $\{m_l\}_{l=1}^L$. Define a subset $\calZ(\{m_l\}_{l=1}^L, \sigma)$ of $\bR^N$ such that a point $\theta = ((a_j, b)_{j=1}^{m_{L-1}}, \theta^{(L-1)}, ..., \theta^{(1)}) \in \calZ(\{m_l\}_{l=1}^L, \sigma)$ if and only if
    \begin{itemize}
        \item [(a)] For any $1 \le l \le L-2$, $\theta^{(l)}$ is arbitrary. 

        \item [(b)] For $l = L-1$, either $H_k^{(L-2)}(\theta, \cdot) = \text{const.}$ for all $1 \le k \le m_{L-2}$, or $\theta^{(L-1)} = (w_j^{(L-1)}, b_j^{(L-1)})_{j=1}^{m_{L-1}}$ does not satisfy requirement (a) or (b) of Theorem \ref{Thm Lin ind of full-connected neurons I}. 

        \item [(c)] By rearranging the indices of $H_j^{(L-1)}(\theta, \cdot)$'s if necessary, we can find $0=n_0 < n_1 < ... < n_r \le m_{L-1}$ and group them as 
        \begin{align*}
            &H_{n_{j-1}+1}^{(L-1)}(\theta, \cdot) = ... = H_{n_j}^{(L-1)}(\theta, \cdot) \ne \text{const.}, &\forall\, 1 \le j \le r, \\ 
            &H_{n_j}^{(L-1)}(\theta, \cdot) \ne H_{n_{j+1}}^{(L-1)}(\theta, \cdot),                         &\forall\, 1\le j < r, \\
            &H_{n_r+1}^{(L-1)}(\theta, \cdot), ..., H_{m_{L-1}}^{(L-1)}(\theta, \cdot) = \text{const.}. 
        \end{align*}
        Then, after rearrangement of indices, $\sum_{t=n_{j-1}+1}^{n_j} a_t = 0$ for all $1 \le j \le r$, and if $H$ is an NN with bias we require 
        \[
            \sum_{t = n_r+1}^{m_{L-1}} a_t H_t^{(L-1)}(\theta, \cdot) + b = 0; 
        \]
        if $H$ is an NN without bias we require 
        \[
            \sum_{t = n_r+1}^{m_{L-1}} a_t H_t^{(L-1)}(\theta, \cdot) = 0. 
        \]
    \end{itemize}
    When the network structure or/and activation function is clear from context, then we simply write $\calZ_\sigma, \calZ(\{m_l\}_{l=1}^L)$ or just $\calZ$, respectively. 
\end{defn}

\begin{defn}
    Given $d, L \in \bN$ and $m_1, ..., m_L \in \bN$. Consider $H$ (with or without bias) with network structure $\{m_l\}_{l=1}^L$. Define $\Tilde{\calZ}(\{m_l\}_{l=1}^L)$ as the set of parameters $\theta$ such that $H(\theta, \cdot) \equiv 0$, i.e., 
    \[
        \Tilde{\calZ}(\{m_l\}_{l=1}^L) = \bigcap_{z\in\bF} H(\cdot, z)^{-1}(0). 
    \]
    For simplicity, we will write $\Tilde{\calZ}_\sigma$ or just $\Tilde{\calZ}$ when the network structure or/and activation function is clear from context. 
\end{defn}

\begin{remark}
    $\calZ$ is ``minimal" in the sense that for any $\sigma$ and any $\theta \in \calZ$, $H_\sigma(\theta,\cdot)$ must be constant zero. Thus, $\calZ \siq \Tilde{\calZ}$. Theorem \ref{Thm Lin ind of full-connected neurons I} says that there is some $\sigma \in \calA(\bR)$ such that $\calZ = \Tilde{\calZ}$. 
\end{remark}

\begin{cor}[Structure of $\calZ$]\label{Cor Structure of Z I}
    Given $d, L \in \bN$ and $m_1, ..., m_L \in \bN$. Consider any NN with network structure $\{m_l\}_{l=1}^L$ and activation $\sigma \in \calA(\bR)$. Given $\bar{\theta}^{(L-2)}, ..., \bar{\theta}^{(1)}$, the intersection of the set $E$ with $\calZ$, where 
    \[
        E := \left\{\theta \in \bR^N: \theta^{(l)} = \bar{\theta}^{(l)}\, \forall\, 1\le l \le L-2\right\}, 
    \]
    is a finite union of sets of the form $\calM \times \calV$. Here $\calV$ is a finite union of linear spaces and $\calM$ a finite union of analytic submanifolds of $\bR^{m_{L-1}+1}$. In particular, $\calZ$ is closed. The same holds for NN without bias, except that we may have countably many union of such sets. 
\end{cor}
\begin{proof}
    Given $\bar{\theta}^{(L-2)}, ..., \bar{\theta}^{(1)}$, there is some $0 \le m \le L-2$ such that $H_k^{(L-2)}(\theta, \cdot)$ is non-constant for all $1 \le k \le m$ and $H_k^{(L-2)}(\theta, \cdot)$ is constant for all $m < k \le m_{L-2}$. Fix any $\theta = (\theta^{(l)})_{l=1}^L \in E \cap \calZ$. For $\theta^{(L-1)}$ we must be able to find distinct $j_1, j_2 \in \{1, ..., m_{L-1}\}$ such that 
    \begin{align*}
        &\,\,\,\,\,\,\,\left( \left(w_{j_1 k}^{(L-1)} \right)_{k=1}^m, \sum_{k > m} w_{j_1 k}^{(L-1)} H_k^{(L-2)}(\theta, z) + b_{j_1}^{(L-1)} \right) \\
        &= \left( \left(w_{j_2 k}^{(L-1)} \right)_{k=1}^m, \sum_{k > m} w_{j_2 k}^{(L-1)} H_k^{(L-2)}(\theta, z) + b_{j_2}^{(L-1)} \right)
    \end{align*}
    or  
    \[
        \left( w_{j_1 k}^{(L-1)} \right)_{k=1}^m = \left( w_{j_2 k}^{(L-1)} \right)_{k=1}^m = 0,
    \]
    Therefore, $\theta^{(L-1)} = (w_j^{(L-1)}, b_j^{(L-1)})_{j=1}^{m_{L-1}}$ must be taken from the union of $2 \binom{m_{L-1}}{2}$ linear subspaces: 
    \begin{align*}
        E' 
        = &\,\bigcup_{j_1 \ne j_2} \left[ \left\{ \theta^{(L-1)}:  \left(w_{j_1 k}^{(L-1)} - w_{j_1 k}^{(L-1)} \right)_{k=1}^m = 0, \sum_{k>m} c_k ( w_{j_1 k}^{(L-1)} - w_{j_2 k}^{(L-1)}) + b_{j_1}^{(L-1)} - b_{j_2}^{(L-1)} = 0 \right\} \right. \\
          &\, \left. \cup \left\{ \theta^{(L-1)}: (w_{j_1 k})_{k=1}^m = (w_{j_2 k})_{k=1}^m = 0 \right\} \right],  
    \end{align*}
    where $c_k = H_k^{(L-2)}(\theta, 0)$ (or actually $H_k^{(L-2)}(\theta, x)$ for any $x$, because the neuron is constant) for every $m < k \le m_{L-2}$. So in particular, $\theta^{(L-1)} \in \cup_{m=0}^{m_{L-2}} E'$. Conversely, any $\theta$ with $\theta^{(L-1)} \in \cup_{m=0}^{m_{L-2}} E'$ must be in $\calZ$. \\
    
    Now consider $l = L$. By Definition \ref{Defn Minimal zero set}, for any $\theta$ with $\theta^{(L-1)} \in E'$, there must be a partition $0=n_0 < n_1 < ... < n_r \le m_{L-1}$ and a permutation $\pi$ on $\{1, ..., m_{L-1}\}$ such that 
    \begin{equation}\label{eq 1 of Cor Structure of Z I}
        \sum_{t=n_{j-1}+1}^{n_j} a_{\pi(t)} = 0, \quad \forall\,1 \le j \le r
    \end{equation}
    and 
    \begin{equation}\label{eq 2 of Cor Structure of Z I}
        b = b\left( (a_t)_{t=1}^{m_{L-1}}, \theta^{(l)}\right) = - \sum_{t=n_r+1}^{m_{L-1}} c_t\left(\theta^{(l)}\right) a_{\pi(t)} 
    \end{equation}
    for some functions $\{c_t\}_{t=n_r+1}^{m_{L-1}}$ determined by constant neurons among $\{H_j^{(L-1)}(\theta, \cdot)\}_{j=1}^{m_{L-1}}$. Since the $c_t$'s are analytic on each linear subspaces in $E$, it follows that (for fixed $r$) equation (\ref{eq 2 of Cor Structure of Z I}) consists of a finite union $\calM_r$ of submanifolds in $\bR^{m_{L-1}+1} \times \bR^{m_{L-1}(m_{L-2}+1)}$, by taking union over all permutations $\pi$ and all appropriate (subsets of) linear spaces in $E'$. Meanwhile, equations (\ref{eq 1 of Cor Structure of Z I}) implies that $(a_j)_{j=1}^{n_r}$ consists of a finite union $\calV_r$ of linear subspaces of $\bR^{n_r}$, by taking union over all possible $(n_0, ..., n_r)$'s. As we can see, $\calM_r$ and $\calV_r$ both depend only on $r$, whence 
    \[
        E \cap \calZ = \cup_r \calM_r \times \calV_r,  
    \]
    completing the proof for NN with bias. \\ 

    The proof for NN with structure $\{m_l\}_{l=1}^{L}$ without bias is quite similar. Indeed, using the notations above, we now have
    \begin{itemize}
        \item [(a)] The set $E'$ is still a finite union of linear spaces, as 
        \begin{align*}
            E' 
            = &\,\bigcup_{j_1 \ne j_2} \left[ \left\{ \theta^{(L-1)}:  \left(w_{j_1 k}^{(L-1)} - w_{j_1 k}^{(L-1)} \right)_{k=1}^m = 0, \sum_{k>m} c_k ( w_{j_1 k}^{(L-1)} - w_{j_2 k}^{(L-1)}) = 0 \right\} \right. \\
            &\, \left. \cup \left\{ \theta^{(L-1)}: (w_{j_1 k})_{k=1}^m = (w_{j_2 k})_{k=1}^m = 0 \right\} \right].  
        \end{align*}

        \item [(b)] Equations (\ref{eq 1 of Cor Structure of Z I}) does not change. 

        \item [(c)] Equation (\ref{eq 2 of Cor Structure of Z I}) becomes $\sum_{t=n_r+1}^{m_{L-1}} c_t\left(\theta^{(l)}\right) a_t = 0$. 
    \end{itemize}
    Suggested by the proof for NN with bias, our proof for NN without bias would be completed if we can show that $((a_j, b)_{j=n_r+1}^{m_{L-1}}, \theta^{(l)})$ satisfying (c) consists of a countable union of analytic submanifolds in $\bR^{m_{L-1}} \times \bR^{m_{L-1} m_{L-2}}$. By (c), the set of these parameters lies in the zero set of the analytic function $\sum_{t=n_r+1}^{m_{L-1}} c_t\left(\theta^{(l)}\right) a_t$, which, by Lemma \ref{Lem analytic function zero set stratification}, is a countable union of analytic submanifolds in its domain, so we are done. \\ 
\end{proof}

To illustrate the structure of $\calZ$, consider a three-neuron two-layer NN $H(\theta, x) = \sum_{k=1}^3 a_k \sigma(w_k x + b_k) + b$. With the notations in Corollary \ref{Cor Structure of Z I}, we have $E' = E_1' \cup E_2'$, where 
\begin{align*}
    E_1' &= \{(w_1, b_1) = (w_2, b_2)\} \cup \{(w_1, b_1) = (w_3, b_3)\} \cup \{(w_2, b_2) = (w_3, b_3)\} \\ 
    E_2' &= \{w_1 = w_2 = 0\} \cup \{w_1 = w_3 = 0\} \cup \{w_2 = w_3 = 0\}. 
\end{align*}
Therefore, $E_1', E_2'$ and thus $E'$ are finite union of linear spaces. If $\theta \in \calZ$ and $\theta^{(1)} = (w_k, b_k)_{k=1}^3 \in \{(w_1, b_1) = (w_2, b_2)\} \cut E_2'$ then we must have 
\[
    a_1 + a_2 = 0, \quad a_3 = 0, \quad b = 0. 
\]
If $\theta^{(1)} \in \{w_1 = w_2 = 0\}$ then we must have 
\[
    b = - a_1\sigma(b_1) - a_2\sigma(b_2), \quad a_3 = 0
\]
when $w_3 \ne 0$, and 
\[
    b = - a_1\sigma(b_1) - a_2\sigma(b_2) - a_3\sigma(b_3)
\]
when $w_3 = 0$. The other cases can be analyzed in a similar way. It follows that the $\calZ$ for this network is a finite union of sets of the form $\calM \times \calV$. \\

The following result is just a restatement of Embedding Principle \cite{YZhangEBDD, YZhangEBDD2} using our notations. It gives a hierarchical description (in terms of NN width) of the parameter space, including $\calZ$ and $\Tilde{\calZ}$. 

\begin{thm}[Embedding Principle]\label{Thm Embedding Principle}
    Given $d, L \in \bN$. Let $m_1, m_1', ..., m_L, m_L' \in \bN$ be such that $m_l' \le m_l$ for all $1 \le l \le L$. Denote $N = \sum_{l=1}^L m_lm_{l-1}$ and $N' = \sum_{l=1}^L m_l' m_{l-1}'$. Denote $H, H'$ as the NN with network structure $\{m_l\}_{l=1}^L$ and $\{m_l'\}_{l=1}^L$, respectively. Then there are $M := \Pi_{l=1}^L \left[\binom{m_l-1}{m_l' - 1}!\right]$ full-rank linear maps $\varphi_1, ..., \varphi_M$ from $\bR^{N'}$ to $\bR^N$ such that for any $\sigma: \bR \to \bR$ and any $1 \le i \le M$, $H_\sigma' = H_\sigma\circ(\varphi_i, id)$. In particular, $\varphi_i(\calZ(\{m_l'\}_{l=1}^L)) \siq \varphi_i(\calZ(\{m_l\}_{l=1}^L))$ and $\varphi_i(\Tilde{\calZ}(\{m_l'\}_{l=1}^L)) \siq \varphi_i(\Tilde{\calZ}(\{m_l\})_{l=1}^L))$.  
\end{thm}

Next we investigate the set of activation functions which do not satisfy $\calZ = \Tilde{\calZ}$.

\begin{defn}
    Let $\calF$ be real vector space of functions from $\bF$ to $\bF$. Let $\calC = \calC(\bF)$ be a subset of $\calF$ such that $\sigma \in \calC$ if and only if there is some $\theta^* \in \bR^N \cut \calZ_\sigma$ with $H_\sigma (\theta^*, \cdot) \equiv 0$. 
\end{defn}

\begin{prop}\label{Prop Structure of calC I}
    The following results hold for $\calC$ as a subset of $\calA(\bR)$. 
    \begin{itemize}
        \item [(a)] For any $S \in \bN \cup \{0\}$, $\overline{\calC} \ne \calA(\bR)$, i.e., $\calC$ is not dense in $\calA(\bR)$, under the norm $\norm{\cdot}_S$. 

        \item [(b)] For any $S \in \bN \cup \{0\}$, $\calC$ does not have an interior in $\calA(\bR)$, under the local $S$-order uniform convergence. 

        \item [(c)] Given $\sigma_0 \in \calA(\bR)$ such that 
        \[
            \sigma_0(x) = o\left(\underbrace{\exp(\exp(...\exp(|x|)...))}_{\text{finite ``$\exp$"'s}}\right)
        \]
        as $x \to \pm \infty$. For any $S \in \bN$ and any $\vep, R > 0$, there is a $\tsigma \in \calA(\bR)$ satisfying $\norm{\tsigma - \sigma_0}_S < \vep$ and $H_{\tsigma}(\theta, \cdot)$ is not constant zero whenever $\theta \in B(0,R)$ satisfies $\dist{\theta}{\calZ_{\tsigma}} \ge \vep$. 
    \end{itemize}
    Similar results hold for $\calC$ as a subset of $C^{S}(\bR)$ with any $S \in \bN \cup \{0, \infty\}$. 
\end{prop}
\begin{proof}~
\begin{itemize}
    \item [(a)] This follows immediately from the proof of Theorem \ref{Thm Lin ind of full-connected neurons I}: for any activation $\sigma \in \calA(\bR)$ with $\sigma(x) = \Theta\left( \exp(e^{x^5}) \right)$ as $x \to \infty$ and $\sigma(x) = \Theta\left( \exp(e^{-x^3}) \right)$ as $x \to -\infty$, $H_\sigma(\theta, \cdot)$ is constant zero if and only if $\theta \in \calZ_\sigma$. 

    \item [(b)] Fix $\sigma_0 \in \calA(\bR)$. By Proposition \ref{Prop function conca and S-order approx}, for any $\vep > 0$ and any $a, a', b, b' \in \bR$ with $a' < a < b < b'$, there is a $\tsigma \in \calA(\bR)$ such that $\norm{\tsigma - \sigma_0}_{S, [a,b]} < \vep$ and $\norm{\tsigma - (\exp(e^{x^5}) + \exp(e^{-x^3}))}_{S, \bR\cut[a', b']} < \vep$. In particular, $\tsigma(x) = \Theta\left( \exp(e^{x^5}) \right)$ as $x \to \infty$ and $\tsigma(x) = \Theta\left( \exp(e^{-x^3}) \right)$ as $x \to -\infty$. Thus, by Theorem \ref{Thm Lin ind of full-connected neurons I}, $H_\sigma(\theta, \cdot)$ is constant zero if and only if $\theta \in \calZ_\sigma$. Equivalently, $\tsigma \notin \calC$. Since $\vep > 0$ is arbitrary, the desired result follows. 

    \item [(c)] Let $\vep' > 0$ be a small number to be determined. By Proposition \ref{Prop function conca and S-order approx}, there is some $\sigma \in \calA(\bR)$ such that $\norm{\sigma - \sigma_0}_{\infty, [0,2]} < \vep'$ and $\sigma(x) = \Theta(\exp(e^{x^5}) + \exp(e^{-x^3}))$ as $x \to \pm\infty$. By Theorem \ref{Thm Lin ind of full-connected neurons I}, $H_\sigma(\theta, \cdot)$ is constant zero if and only if $\theta \in \calZ_\sigma$. Fix this $\sigma$. By Proposition \ref{Prop function conca and S-order approx} again, by making $\norm{\sigma - \sigma_0}_{\infty, [0,2]}$ to be sufficiently small, there is some $\tsigma \in \calA(\bR)$ such that $\norm{\tsigma - \sigma}_{S, [0,1]} < \vep'$ and $\norm{\tsigma - \sigma_0}_S < \vep'$. In particular, $\norm{\tsigma - \sigma}_{\infty, [0,1]} < \vep'$. Now consider the integral function $\tau: \calA(\bR) \times \bR^N$ defined by 
    \[
        \tau(f, \theta) := \int_0^1 H_f^2(\theta, x)dx. 
    \]
    
    Clearly, $\tau$ is continuous in $\theta$, non-negative and $H_f(\theta, \cdot) \equiv 0$ if and only if $\tau(f, \theta) = 0$. Define 
    \[
        L := \inf \left\{ \tau(\sigma, \theta): \theta \in \overline{B(0,R)}, \dist{\theta}{\calZ_\sigma} \ge \frac{1}{3} \vep \right\}. 
    \]
    Similarly, $\calZ_f$ varies continuously with respect to $f$ in the sense that for any $\vep > 0$, 
    \[
        \dist{\calZ_f}{\calZ_{\Tilde{f}}} := \sup \inf_{\theta \in \calZ_f, \ttheta \in \calZ_{\Tilde{f}}} |\theta - \ttheta| < \vep
    \]
    when $\norm{f - \Tilde{f}}_{\infty, [0,1]}$ is sufficiently small. Thus, for sufficiently small $\vep' \in (0, \vep)$, we have both 
    \[
        \inf \left\{ \tau(\tsigma, \theta): \theta \in \overline{B(0,R)}, \dist{\theta}{\calZ_\sigma} \ge \frac{1}{3} \vep \right\} \ge \frac{L}{2} > 0 
    \]
    and $\dist{\calZ_\sigma}{\calZ_{\tsigma}} \le \frac{1}{3} \vep$. \\
    
    We now show that this $\tsigma$ satisfies the requirements of (c). By construction, we already have $\norm{\tsigma - \sigma_0}_S < \vep$. Let $\theta \in B(0,R)$ be such that $\dist{\theta}{\calZ_{\tsigma}} \ge \vep$. Assume that $\tau(\tsigma, \theta) = 0$, then we must have $\dist{\theta}{\calZ_{\sigma}} \le \frac{1}{3}\vep$. Thus, there is some $\theta^* \in \calZ_\sigma$ with $|\theta - \theta^*| \le \frac{1}{3} \vep$. Fix this $\theta^*$, since $\dist{\calZ_\sigma}{\calZ_{\tsigma}} \le \frac{1}{3} \vep$, there is some $\ttheta^* \in \calZ_{\tsigma}$ with $|\theta^* - \ttheta^*| \le \frac{1}{3} \vep$. But then 
    \[
        \dist{\theta}{\calZ_{\tsigma}} \le |\theta - \ttheta^*| \le |\theta - \theta^*| + |\theta^* - \ttheta^*| \le \frac{2}{3}\vep, 
    \]
    contradicting our assumption that $\dist{\theta}{\calZ_{\tsigma}} \ge \vep$. It follows that $\tau(\tsigma, \theta) = 0$. 
\end{itemize}
\end{proof}

\subsection{\texorpdfstring{$\calZ_\sigma$}{ZSigma} for Activations Vanishing at Origin}\label{Subsection Zsigma for activations vanishing at 0}

In this part we study $\calZ$ for NN without bias and with activation $\sigma \in \calA(\bR)$ such that $\sigma(0) = 0$. First we show that for definable activation functions with respect to the o-minimal structure generated by $\exp(\cdot)$ (and vanishing at the origin), $\calZ_\sigma$ is just a finite union of linear subspaces of $\bR^N$. This can be seen both as a generalization of linear independence of two-layer neurons and as a converse to Embedding Principle (Theorem \ref{Thm Embedding Principle}). Then we show in Proposition \ref{Prop Structure of calC II} that up to arbitrarily small perturbation of any $\sigma \in \calA(\bR)$ with $\sigma(0) = 0$, we have $\Tilde{\calZ}_\sigma \cap \overline{B(0,R)} = \calZ_\sigma \cap \overline{B(0,R)}$, and furthermore $\calZ_\sigma \cap \overline{B(0,R)}$ is just a finite union of linear subspaces of $\bR^N$ intersecting $\overline{B(0,R)}$. Therefore, by these two results, we can see that the requirement $\sigma(0) = 0$ generically yields the simplest possible geometry of $\calZ_\sigma$ (and $\Tilde{\calZ}_\sigma$), and thus gives the simplest possible characterization of parameters $\theta$ such that the neurons are linearly independent. \\

\begin{lemma}\label{Lem Fixed point of NN}
    Let $H$ be an NN with network structure $\{m_l\}_{l=1}^L$ and without bias. For $\sigma: \bR \to \bR$ such that $\sigma(0) = 0$, we have $H_\sigma^{(l)}(\theta, 0) = 0$ for all $\theta \in \bR^N$. In particular, $H_\sigma^{(l)}(\theta, \cdot)$ is constant if and only if it is constant zero (note that in the formulas above we are abusing the notation $H^{(l)}(\theta^{(l)}, ..., \theta^{(1)}, \cdot) = H^{(l)}(\theta, \cdot)$, see remark \ref{Rmk Comments on Defn NN}). 
\end{lemma}
\begin{proof}
    For the first part, we prove by induction. Indeed, this is obvious for $l = 1$. Let $l \ge 2$ and assume that this holds for $H_1^{(l-1)}(\theta, \cdot), ..., H_{m_{l-1}}^{(l-1)}(\theta, \cdot)$. Then 
    \[
        H_k^{(l)}(\theta, 0) = \sigma\left( w_k^{(l)} H^{(l-1)}(\theta, 0) \right) = \sigma(0) = 0,  
    \]
    completing the induction step. It follows that for any $1 \le l \le L-1$ and any $1 \le k \le m_l$, $H_k^{(l)}$ is constant if and only if it is constant zero. \\

    Now assume that $H_\sigma^{(l)}(\theta, \cdot)$ is constant. Then for any $x \in \bR^d$, $H_\sigma^{(l)}(\theta, x) = H_\sigma^{(l)}(\theta, 0) = 0$. Therefore, $H_\sigma^{(l)}(\theta, \cdot)$ is constant if and only if it is constant zero.\\ 
\end{proof}

\begin{lemma}\label{Lem Lin ind of full-connected neurons II}
    There is some $\sigma \in \calA(\bR)$ such that the following results hold. 
    \begin{itemize}
    \item [(a)] $\sigma(0) = 0$. 
    \item [(b)] Consider an NN with network structure $\{m_l\}_{l=1}^L$ without bias. For any parameter $\theta = (m_l)_{l=1}^L \in \bR^N$ such that $\{H_k^{(L-2)}(\theta, \cdot) \}_{k=1}^{m_{L-2}}$ are distinct and not constant zero, then $\{H^{(L-1)}_j(\theta, \cdot) \}_{j=1}^{m_{L-1}}$ are linearly independent if and only if $w_j^{(L-1)} \ne 0$ for all $j \in \{1, ..., m\}$, and $w_{j_1}^{(L-1)} \ne w_{j_2}^{(L-1)}$ for all distinct $j_1, j_2 \in \{1, ..., m_{L-1}\}$. 
    \item [(c)] Consider $H$ with network structure $\{m_l\}_{l=1}^L$ without bias. For any $\theta \in \bR^N$, there are $\Tilde{m}_1, ..., \Tilde{m}_L \in \bN$, each $\Tilde{m}_l \le m_l$, together with a unique $\ttheta \in \bR^{\sum_{l=1}^L \Tilde{m}_l \Tilde{m}_{l-1}}$,  such that $\Tilde{H}^{(l)}(\ttheta, \cdot) = H^{(l)}(\ttheta, \cdot)$ for all $l \in \{1, ..., L\}$.  
    \end{itemize}
    
\end{lemma}
\begin{proof}
    Let $\tsigma \in \calA(\bR)$ satisfy the requirements in Theorem \ref{Thm Lin ind of full-connected neurons I}. Define 
    \[
        \sigma:\bR \to \bR, \quad \sigma(z) = \tsigma(z) - \sigma(0). 
    \]
    Then clearly $\sigma(0) = 0$ and $\sigma(z) = \Theta(\tsigma(z))$ as $z \to \pm\infty$. In particular, $\limftyz \sigma(z) = \lim_{z\to-\infty} \sigma(z) = \infty$ at hyper-exponential rate (actually hyper-polynomial rate growth is enough, but hyper-exponential rate growth implies hyper-polynomial rate growth), $\sigma(z) = \Theta(\exp(e^{z^5}))$ as $z \to \infty$ and $\sigma(z) = \Theta(\exp(e^{-z^3}))$ as $z \to -\infty$. \\
    
    First we prove (b). Note that (with the notation in Theorem \ref{Thm Lin ind of full-connected neurons I}) $m = m_{L-2}$, because Lemma \ref{Lem Fixed point of NN} says $H_k^{(L-2)}(\theta, \cdot)$'s are non-zero implies that they are not constant. By the proof of Theorem \ref{Thm Lin ind of full-connected neurons I}, $\{H^{(L-1)}_j(\theta, \cdot) \}_{j=1}^{m_{L-1}}$ are linearly independent if and only if 
    \begin{itemize}
        \item [(a)] The vectors $w_j^{(L-1)}$, $1 \le j \le m_{L-1}$ are distinct. 
        \item [(b)] We have $w_j^{(L-1)} = 0$ for at most one $j$. Moreover, for this $j$, $H_j^{(L-1)}(\theta, \cdot)$ is not constant zero. 
    \end{itemize}
    However, when $w_j^{(L-1)} = 0$, $H_j^{(L-1)}(\theta, \cdot)$ is constant; by Lemma \ref{Lem Fixed point of NN} it must be constant zero. Thus, $\{H^{(L-1)}_j(\theta, \cdot) \}_{j=1}^{m_{L-1}}$ are linearly independent if and only if $w_j^{(L-1)}$ are distinct and non-zero. 
\end{proof}

\begin{thm}\label{Thm Lin Ind of full-connected neurons III}
    There is some $\sigma \in \calA(\bR)$ such that for any NN without bias, with network structure $\{m_l\}_{l=1}^L$ and activation $\sigma$, the minimal zero set coincides with $\Tilde{\calZ}$, i.e., $\calZ = \Tilde{\calZ}$. Moreover, $\calZ$ is a finite union of linear subspaces of $\bR^N$. 
\end{thm}
\begin{proof}
    Let $\sigma$ be any activation satisfying the requirements in Lemma \ref{Lem Lin ind of full-connected neurons II}. In particular, $\sigma(0) = 0$, $\limftyz \sigma(z) = \lim_{z\to-\infty} \sigma(z) = \infty$ at hyper-exponential rate, $\sigma(z) = \Theta(\exp(e^{z^5}))$ as $z \to \infty$ and $\sigma(z) = \Theta(\exp(e^{-z^3}))$ as $z \to -\infty$. Since $\sigma$ satisfies Theorem \ref{Thm Lin ind of full-connected neurons I}, we immediately have $\calZ = \Tilde{\calZ}$. 

    Then we show that for any $1 \le l \le L-1$ and $0 = n_0 < n_1 < ... < n_r \le m_{L-1}$, the set of parameters $\theta^{(l)}, ..., \theta^{(1)}$ such that up to a rearrangement of the indices, 
    \begin{equation}\label{eq 1 of Thm Lin ind of full-connected neurons III}
    \begin{aligned}
        &H_{n_{j-1}+1}^{(l)}(\theta, \cdot) = ... = H_{n_j}^{(l)}(\theta, \cdot) \ne \text{const.} &\forall\, 1 \le j \le r, \\
        &H_{n_j}^{(l)}(\theta, \cdot) \ne H_{n_j+1}^{(l)}(\theta, \cdot) &\forall\, 1 \le j < r, \\
        &H_{n_r+1}^{(l)}(\theta, \cdot) = ... = H_{m_{l-1}}^{(l)}(\theta, \cdot) \equiv 0
    \end{aligned}
    \end{equation}
    is a finite union of linear subspaces of $\bR^{\sum_{l'=1}^l m_{l'} m_{l'-1}}$. To prove this, first recall that by Theorem \ref{Thm Lin ind of full-connected neurons I} and our proof above, the neurons $\{\sigma(w_k^{(1)} x)\}_{k=1}^{m_1}$ are linearly independent if and only if the $w_k^{(1)}$'s are distinct and non-zero. Thus, they can be classified as in (\ref{eq 1 of Thm Lin ind of full-connected neurons III}) if and only if there are exactly $r$ distinct, non-zero weights $w_{n_1}^{(1)}, ..., w_{n_r}^{(1)}$, each one repeating for a fixed number of times, and there are exactly $m_1 - n_r$ weights equal to zero. Thus, for fixed $0 = n_0 < n_1 < ... < n_r \le m_1$, the closure of the set of $\theta^{(1)} = (w_k^{(1)})_{k=1}^{m_1}$ satisfying (\ref{eq 1 of Thm Lin ind of full-connected neurons III}) is a finite union of linear subspaces of $\bR^{m_1 d}$. \\ 

    Now consider $2 \le l \le L-1$. Assume that this holds for $l-1$. Fix $\theta^{(l-1)}, ..., \theta^{(1)}$. Let $0 = n_0 < ... < n_r \le m_{l-1}$ work for the $H_k^{(l-1)}(\theta, \cdot)$'s. Up to a rearrangement of indices, we can write each $H_j^{(l)}$ as 
    \[
        H_j^{(l)}(\theta, z) = \sigma\left( \sum_{t=1}^r \left( \sum_{k=n_{t-1}+1}^{n_t} w_{jk}^{(l)} \right) H_{n_k}^{(l-1)}(\theta, z) \right). 
    \]
    By Theorem \ref{Thm Lin ind of full-connected neurons I} again, these generalized neurons can be classified as in (\ref{eq 1 of Thm Lin ind of full-connected neurons III}) if and only if for some $r' \le m_l$, there are exactly $r'$ distinct, non-zero ``combined weights" 
    \[
        \left(\sum_{k=n_{t-1}+1}^{n_t} w_{n_1' k}^{(l)}\right)_{t=1}^r, ..., \left(\sum_{k=n_{t-1}+1}^{n_t} w_{n_{r'}' k}^{(l)} \right)_{t=1}^r, 
    \]
    each one repeating for a fixed number of times, and there are exactly $m_l - n_{r'}'$ ``combined weights" equal to zero. Thus, for fixed $0 = n_0' < n_1' < ... < n_{r'}' \le m_1$, the closure of the set of $\theta^{(l)} = (w_j^{(l)})_{j=1}^{m_l}$ satisfying (\ref{eq 1 of Thm Lin ind of full-connected neurons III}) is a finite union of linear subspaces of $\bR^{m_l m_{l-1}}$. Finally, by Definition \ref{Defn Minimal zero set}, given $\theta^{(L-1)}, ..., \theta^{(1)}$, the closure of the set of $\theta^{(L)} = (a_j)_{j=1}^{m_{L-1}}$ is a finite union of linear subspaces of $\bR^{m_{L-1}}$. \\

    From this we can see that $\calZ$ is a finite union of some sets $V_1, ..., V_n$ whose closure are linear subspaces of $\bR^N$. Since $\calZ$ is closed, it follows that 
    \[
        \calZ = \overline{\calZ} = \overline{V_1} \cup ... \overline{V_n},  
    \]
    so that $\calZ$ is a finite union of linear subspaces of $\bR^N$. 
\end{proof}
\begin{remark}
    As we can see from the proof of Theorem \ref{Thm Lin Ind of full-connected neurons III}, when $\sigma$ satisfies the hypotheses in this theorem, $\calZ(\{m_l\}_{l=1}^L, \sigma)$ is independent of $\sigma$ -- it depends only on the network structure $\{m_l\}_{l=1}^L$. This is different from the general case proved in Theorem \ref{Thm Lin ind of full-connected neurons I}. \\
\end{remark}

\begin{prop}\label{Prop Structure of calC II}
    Consider $L \ge 2$ and the network structure $\{m_l\}_{l=1}^L$. Given $S \in \bN$ and $\vep, M > 0$. For any $\sigma_0 \in C^S(\bR)$ such that $\sigma_0(0) = 0$ and
    \[
        \sigma_0(x) = o\left(\underbrace{\exp(\exp(...\exp(|x|)...))}_{\text{finite ``$\exp$"'s}}\right) 
    \]
    as $x \to \pm\infty$, there is some $\tsigma \in \calA(\bR)$ with $\norm{\tsigma - \sigma_0}_S < \vep$, and $\Tilde{\calZ}_{\tsigma} = \calZ_{\tsigma}$. 
\end{prop}
\begin{proof}
    Let $\mu > S$ be an integer to be determined later. Given $\sigma \in \calA(\bR)$ with the following two properties: 
    \begin{itemize}
        \item [(a)] It satisfies the hypotheses in Theorem \ref{Thm Lin Ind of full-connected neurons III}. 

        \item [(b)] It is obtained by algebraic operations of exponential function and polynomials. 
    \end{itemize}
    Therefore, we can apply Proposition \ref{Prop function conca and S-order approx} and/or \ref{Prop Structure of calC I} to find some $\tsigma \in \calA(\bR)$ such that $\norm{\tsigma - \sigma}_{\mu, [0,1]} < \vep$, $\norm{\tsigma - \sigma}_\mu < \vep$ and $\tsigma(x) - \sigma(x) = O(x^\mu)$ as $x \to 0$. We will prove that for sufficiently large $\mu$, we have $\Tilde{\calZ}_{\tsigma} \cap \overline{B(0,R)} = \calZ_{\tsigma} \cap \overline{B(0,R)} = \calZ_\sigma \cap \overline{B(0,R)}$. \\

    By property (b) of $\sigma$, $\sigma$ and its derivatives $\sigma$, and thus $H_\sigma$ and its derivatives are all definable functions. By Corollary \ref{Cor Limiting asymp of dim-1 parameterized function near 0}, for any $\theta \in \bR^N \cut \calZ_{\sigma}$, there is some $M \in \bN$ such that for any $\theta \in \bR^N \cut \calZ$, there are some $s \in \{0, 1, ..., M\}$ and $c_s(\theta) \ne 0$ with $H_\sigma(\theta, x) \sim c_s(\theta) x^s$ as $x \to 0$. Thus, if $\theta^* \in \bR^N \cut \calZ_\sigma$ satisfies $H_{\tsigma}(\theta^*, \cdot) \equiv 0$, we must have 
    \begin{equation}\label{eq 1 of Prop Structure of calC II}
\begin{aligned}
        |H_{\tsigma}(\theta^*, x) - H_\sigma(\theta^*, x)| 
        &= |H_\sigma(\theta^*, x)| \\
        &\sim |c_s(\theta^*)||x|^s \\
        &= \Omega(x^M), 
    \end{aligned}    
    \end{equation}
    where $s \in \{0, 1, ..., M\}$ and $c_s(\theta^*) \ne 0$. \\ 

    We then show by induction that under our construction of $\tsigma$, we have $|H_{\tsigma}(\theta, x) - H_\sigma(\theta, x)| = O(x^\mu)$ as $x \to 0$. Indeed, for any $w \in \bR$ we have 
    \[
        |\tsigma(wx) - \sigma(wx)| \le C|wx|^\mu = C|w|^\mu |x|^\mu, 
    \]
   where $C \ge 0$ is determined by $\tsigma$, $\sigma$, but not $w$, because we can always make $x$ small enough so that $wx$ is bounded by a given constant, say $|wx| \in [0,1]$. Thus, for $L=2$ we obtain the following estimate for $\theta = (a_k, w_k)_{k=1}^{m_1} \in \bR^{(d+1)m_1}$: 
    \begin{align*}
        \left|\sum_{k=1}^{m_1} a_k \tsigma(w_k x) - \sum_{k=1}^{m_1} a_k \sigma(w_k x)\right| 
        &\le \sum_{k=1}^{m_1} |a_k| |\tsigma(w_k x) - \sigma(w_k x)| \\ 
        &\le C\left( \sum_{k=1}^{m_1} |a_k||w_k|^\mu \right) |x|^\mu \\ 
        &\le C m_1 R^{\mu+1} |x|^\mu. 
    \end{align*}
    Thus, the desired result holds for two-layer networks $H_{\tsigma}^{(2)}$ and $H_\sigma^{(2)}$. Now suppose that the result holds for each component of $H_{\tsigma}^{(L-2)}$ and $H_\sigma^{(L-2)}$, so given $\theta$ there is some $C' > 0$ with $|H_{\tsigma}^{(L-2)}(\theta, x) - H_\sigma^{(L-2)}(\theta, x)| \le C'|x|^\mu$ for all sufficiently small $x \in \bR$. For the components of $H_{\tsigma}^{(L-1)}(\theta, \cdot)$ and $H_\sigma^{(L-1)}(\theta, \cdot)$ we do the following estimate: 
    \begin{align*}
        |H_{\tsigma, j}^{(L-1)}(\theta, x) - H_{\sigma, j}^{(L-1)}(\theta, x)| 
        &= \left|\tsigma\left(w_j^{(L-1)}H_{\tsigma}^{(L-2)}(\theta, x) \right) - \sigma\left( w_j^{(L-1)}H_\sigma^{(L-2)}(\theta, x) \right)\right| \\ 
        &\le \left| \tsigma\left(w_j^{(L-1)}H_{\tsigma}^{(L-2)}(\theta, x) \right) - \tsigma\left(w_j^{(L-1)}H_\sigma^{(L-2)}(\theta, x) \right)\right| \\
        &\,\,\,\,\,\,\,+\left|\tsigma\left(w_j^{(L-1)}H_\sigma^{(L-2)}(\theta, x) \right) - \sigma\left( w_j^{(L-1)}H_\sigma^{(L-2)}(\theta, x) \right)\right| \\
        &\le \norm{\Tilde{\sigma}'}_{\infty, [0,1]} |w_j^{(L-1)}||H_{\tsigma}^{(L-2)}(\theta, x) - H_\sigma^{(L-2)}(\theta, x)| \\
        &\,\,\,\,\,\,\,+ C|w_j^{(L-1)}|^\mu|H_\sigma^{(L-2)}(\theta, x)|^\mu \\
        &\le \norm{\Tilde{\sigma}'}_{\infty, [0,1]} R C'|x|^\mu + CR \sum_{k=1}^{m_{L-2}}\norm{\frac{d}{dx} H_{\sigma, k}^{(L-2)}(\theta, \cdot)}_{\infty, [0,1]}^\mu |x|^\mu. 
    \end{align*}
    Here $x$ is so small that it satisfies 
    \[
        x\in [0,1], \quad w_j^{(L-1)}H_{\tsigma}^{(L-2)}(\theta, x) \in [0,1], \quad w_j^{(L-1)}H_\sigma^{(L-2)}(\theta, x) \in [0,1]. 
    \]
    Moreover, in the estimate above, for the second inequality we use the induction hypothesis on $H_{\tsigma}^{(L-2)}$ and $H_\sigma^{(L-2)}$, and for the second and third inequality we use Lipschitz property of $\sigma$ and the $H_{\sigma,k}^{(L-2)}$'s, respectively. Thus, we have $|H_{\tsigma, j}^{(L-1)}(\theta, x) - H_{\sigma, j}^{(L-1)}(\theta, x)| = O(x^\mu)$ as $x \to 0$ for every $1 \le j \le m_{L-1}$.  It follows that 
    \[
        \left|H_{\tsigma}(\theta, x) - H_\sigma(\theta, x) \right| \le \sum_{j=1}^{m_{L-1}} |a_k| |H_{\tsigma, j}^{(L-1)}(\theta, x) - H_{\sigma, j}^{(L-1)}(\theta, x)| = O(x^\mu). 
    \]
   as $x \to 0$, showing the claim above. \\
   
    Combining this estimate with (\ref{eq 1 of Prop Structure of calC II}), we can see that 
    \[
    \Omega(x^M) = |H_{\tsigma}(\theta^*, x) - H_\sigma(\theta^*, x)| = O(x^\mu)
    \]
    as $x \to 0$. But this gives a contradiction whenever $\mu > M$. Thus, $H_{\tsigma}(\theta^*, \cdot)$ cannot be constant zero. Since this holds for any $\theta^* \in \bR^N \cut  \calZ_\sigma$, we must have $\Tilde{\calZ}_{\tsigma} =\calZ_{\tsigma} = \calZ_\sigma$, completing the proof. \\
\end{proof}

\section{Misc}\label{Section Misc} 

The purpose of this section is to discuss linear independence of neurons with special structures, especially those two- and three-layer neurons. As the methods we use here are often different from that for proving general cases in Section \ref{Section Theory of general neurons}, one may read this part by skipping Section \ref{Section Theory of general neurons}. \\

The structure for this section is as follows. First, in Section \ref{Subsection Generic activations are not enough}, we show that while generic activation functions (see Definition \ref{Defn Generic activation}) implies that for two-layer neurons without bias, $\calZ$ is completely characterized by the simplest possible cases (\ref{simplest case of linear dependence}) in Section \ref{Section Intro}. In particular, for any $m, d \in \bN$, 
\begin{equation}
    \{\sigma(w_k x): w_k, x \in \bR^d \}_{k=1}^m 
\end{equation}
are linearly independent if and only if $w_k$'s are distinct. Unfortunately, it does not provide such simple structure of $\calZ$ for general neurons (e.g., multi-layer ones and even two-layer neurons with bias). Thus, it is natural to ask for what choice of $\sigma$ do we have this, or at least similar results. For neurons without bias, this is addressed by perturbing any $\sigma \in \calA(\bR)$ with $\sigma(0) = 0$, as shown in Section \ref{Subsection Zsigma for activations vanishing at 0}. Section \ref{Subsection Two-layer neurons with bias} discusses this problem for two-layer neurons with bias. In particular, we identify three cases in which $\Tilde{\calZ}$ is completely characterized by the simplest possible cases (\ref{simplest case of linear dependence}). \\

Finally, we discuss the linear independence of neurons with Sigmoid and Tanh activations for three-layer neurons with bias, i.e., activation functions $\sigma(z) = \frac{1}{1 + e^{-x}}$ or $\tanh(z) = \frac{e^{x} - e^{-x}}{1 + e^{-x}}$. 

\subsection{Generic Activations are not Enough}\label{Subsection Generic activations are not enough}

Recall the definition of generic activations often used in the analysis of two-layer neural networks. 

\begin{defn}[Generic activation]\label{Defn Generic activation}
    A generic activation $\sigma: \bR \to \bR$ is a smooth function ($\sigma \in C^\infty (\bR)$) such that $\sigma(0) \ne 0$ and its derivatives satisfy $\sigma^{(s)}(0) \ne 0$ for infinitely many even and odd numbers $s \in \bN$. 
\end{defn}

The definition of generic activation differs slightly among different works. The definition here follows Simsek's version \citep{BSimsek}. Alternatively, some works further require that $\sigma$ is analytic, such as \cite{LZhangGlobal, RSun}. However, their proof still only uses $C^\infty$ property; the analyticity is mainly used in the analysis of loss landscapes. \\

The following results give a complete characterization of two-layer neurons without bias and with generic activation. A proof can be found in \citet{LZhangGlobal}. 

\begin{prop}[Corollary 3.1 from \cite{LZhangGlobal}]
    Let $m, d \in \bN$. Given a generic activation $\sigma: \bR \to \bR$, then for any $w_1, ..., w_m \in \bR^d$, the (two-layer) neurons $\{\sigma(w_k z)\}_{k=1}^m$ are linearly independent if and only if $w_j \ne w_k$ for all distinct $j,k \in \{1, ..., m\}$. 
\end{prop}

\textbf{Example. } We now give examples to show that with generic activations, $\calZ$ for either two-layer neurons with bias or for multi-layer neurons without bias could still have very complicated structures. In the following example, $\sigma$ always denotes a generic activation. 
\begin{itemize}
    \item [(a)] (two-layer neurons with bias) Consider $\sigma(z) = e^z$. Then for any $w \in \bR$ and $b, b' \in \bR$, $\sigma(wz + b)$ and $\sigma(wz + b')$ are linearly dependent. 
    
    \item [(b)] (two-layer neurons with bias) Consider $\sigma(z) = \tanh(z+1)$. Then for any $w \in \bR$, $\sigma(wz - 1)$ and $\sigma(-wz - 1)$ are linearly dependent. 

    \item [(c)] (two-layer neurons with bias) Consider $\sigma(z) = \cos(z) + \sin(z)$. Then for any $w \in \bR$ and any $k \in \bZ$, $\sigma(wz)$ and $\sigma(wz + 2k\pi)$ are linearly dependent. Similarly, it is not difficult to construct a smooth function $\sigma$ such that $\sigma(wz) = \sigma(wz + b) + b'$ for given $w, b, b'\ne 0$ and all $z \in \bR$, but then $\sigma(wz) - \sigma(wz + b) - b'$ is constant zero. 

    \item [(d)] (multi-layer neurons without bias) Consider $\sigma(z) = \tanh(z+1)$ again. Then for real numbers $w_1^{(2)}, w_2^{(2)}$ and $w^{(1)}$, the three-layer neurons 
    \begin{align*}
        \sigma\left(w_1^{(2)} \sigma(w^{(1)}z) + w_2^{(2)}\sigma(0)\right), \sigma\left(-w_1^{(2)} \sigma(w^{(1)}z) + w_2^{(2)}\sigma(0) \right)
    \end{align*}
    are linearly dependent provided that $w_2^{(2)} = -\frac{1}{\tanh(1)}$. Using the same idea we can construct similar examples for general multi-layer neural networks without bias, when $\sigma(z) = \tanh(z+1)$.
\end{itemize}

\subsection{Two-layer Neurons With Bias}\label{Subsection Two-layer neurons with bias}

\begin{defn}[Fourier transform]
    Let $\sF$ denote the Schwartz class of functions on $\bR$, namely, $f \in \sF$ if for any $s, s' \in \bN \cup \{0\}$ we have $\sup_{z \in \bR} |x^{s'} f^{(s)}(z)| < \infty$. For any $f \in \sF$ the Fourier transform of $f$ is defined as 
    \[
        \hat{f}(\xi) = \int_{-\infty}^\infty f(z) e^{-i\xi z} dz. 
    \]
    In particular, we will use $\xi$ to denote the frequency-domain variable. 
\end{defn}

Given $w_1, ..., w_m > 0$, the following result can be used to distinguish Schwartz functions $\{f(w_k z)\}_{k=1}^\infty$ by investigating the decay rates of their Fourier transforms.  

\begin{lemma}\label{Lem FT decay}
    Let $f \in \sF$ be non-constant zero. Then its Fourier transform $\hat{f}$ has the property that for any $0 < \tw < w$, $\hat{f}(w \xi) = o(\hat{f}(\tw \xi))$ as $\xi \to \infty$. 
\end{lemma}
\begin{proof}
    Note that $f \in \sF$ implies $\norm{\xi^s f^{\hat{}}}_{\infty} < \infty$ for all $s \in \bN$ (\red{cite Proposition 2.2.11 in Grafakos' book}). In other words, for any $s \in \bN$, there is some $T > 0$ such that $|f^{\hat{}}(\xi)| = O(\xi^{-s})$ as $\xi \to \infty$. We then show that this implies 
    \begin{equation}\label{eq 1 of Lem FT decay}
        \frac{1}{\hat{f}(\tw \xi)} = o\left(\frac{1}{\hat{f}(w \xi)}\right), 
    \end{equation}
    so that the desired result follows. \\

    For simplicity, let $g := \frac{1}{\hat{f}}$ and by a substitution of variables if necessary, we can assume that $\tw = 1$ and $w > 1$. Assume that (\ref{eq 1 of Lem FT decay}) does not hold. Then there is a sequence $\{\xi_n\}_{n=1}^\infty$ with $\limftyn \xi_n = \infty$, and there is some $M > 0$ such that $0 < |g(w \xi_n)| \le M |g(\xi_n)|$ for all $n \in \bN$. It follows that for any $n \in \bN$ with $\xi_n \in [w^{k-1}, w^k)$ for some $k$ we have 
    \begin{align*}
        0 < |g(w \xi_n)| \le M |g(\xi_n)| \le ... \le M^k \left|g\left(w^{-k} \xi_n \right)\right|. 
    \end{align*}
    Note that $w^{-k} \xi_n \in [1, w)$, whence 
    \begin{align*}
        |g(w \xi_n)| 
        &\le M^k \norm{g}_{\infty, [1,w]} \\ 
        &\le \norm{g}_{\infty, [1,w]} (w^{\log_w M})^k \\ 
        &= \norm{g}_{\infty, [1,w]} (w^k)^{\log_w M} \\ 
        &\le \norm{g}_{\infty, [1,w]} (w\xi_n)^{\log_w M}, 
    \end{align*}
    Let $s:= \log_w M$. Then for any $s' > s$, 
    \begin{align*}
        \left|\hat{f}(w \xi_n) (w\xi_n)^{s'}\right| 
        &= \left| \frac{(w\xi_n)^{s + (s'-s)}}{g(w \xi_n)} \right| \\
        &\ge \frac{|w^k|^s}{\norm{g}}_{\infty, [1,w]} |w \xi_n|^{s'-s} \\
        &= \frac{M^k}{\norm{g(w\xi_n)}_{\infty, [1,w]}} |w\xi_n|^{s'-s}, 
    \end{align*}
    and thus $\limftyn |\hat{f}(w \xi_n) (w\xi_n)^{s'}| = \infty$, contradicting the fact that $\norm{x^{s'} \hat{f}}_{\infty} < \infty$. 
\end{proof}
\begin{remark}
    As a result, given $m \in \bN$ and (strictly) positive numbers $w_1, ..., w_m$, the functions $\{\hat{f}(w_k\xi)\}_{k=1}^m$ have ordered growth if and only if $w_k \ne w_j$ for all distinct $k, j \in \{1, ..., m\}$. This result will be used in Proposition \ref{Prop Lin ind two-layer neurons with bias} (c) to deduce the linear independence of two-layer neurons with bias. \\
\end{remark}

\begin{lemma}\label{Lem Lin Comb of Tri Func are bounded below}
    Let $m \in \bN$. Given distinct $b_1, ..., b_m > 0$ and non-zero constants $a_1, ..., a_m \in \bR$, we have 
    \[
        \varlimsup_{z\to\infty} \left| \sum_{k=1}^m a_k e^{ib_k z} \right| > 0. 
    \]
\end{lemma}
\begin{proof}
    Let $f(z) := \sum_{k=1}^m a_k \cos(b_k z)$. It suffices to show that $\varlimsup_{z \to \infty} |f(z)| > 0$. Note that for any $s \in \bN$, $\norm{f^{(s)}}_{\infty} < \infty$ because 
    \[
        |f^{(s)}(z)| = \left| \sum_{k=1}^m a_k b_k^s \cos(b_k z) \right| \le \sum_{k=1}^m |a_k| |b_k|^s. 
    \]
    Without loss of generality, we may assume that $b_1 > b_2 > ... > b_m$. Then there are $S \in \bN$, $\delta_S > 0$ and a sequence $\{z_n^{(S)}\}_{n=1}^\infty$ with $\limftyn z_n^{(S)} = \infty$, such that for any $z \in (z_n, z_n + \delta_S)$, $|f^{(S)}(z)| \ge \frac{|a_1||b_1|^S}{2}$. Fix this $S$. We would like to show that for any $0 \le s \le S$, there are $\delta_s > 0$ and a sequence $\{z_n^{(s)}\}_{n=1}^\infty$ with $\limftyn z_n^{(s)} = \infty$, such that $|f^{(s)}|$ is bounded below by some $C_s > 0$ on $(z_n^{(s)}, z_n^{(s)} + \delta_s)$, for all $n \in \bN$. This immediately gives us the desired result. \\
    
    Indeed, our proof above implies that this holds for $s = S$. Suppose it holds for some $s \in \{1, ..., S\}$. Let $\delta_s, \{z_n^{(s)}\}_{n=1}^\infty$ and $C_s$ be defined as above. If $\varlimsup_{z\to\infty} |f^{(s-1)}(z)| = 0$, then there is some $n \in \bN$ with $|f^{(s-1)}(z_n^{(s)})| \le \frac{C_s \delta_s}{3}$. But this gives a contradiction, because 
    \begin{align*}
        |f^{(s-1)}(z_n^{(s)})| 
        &= \left| f^{(s-1)}(z_n^{(s)}) + \int_{z_n^{(s)}}^{z_n^{(s)}+\delta_s} f^{(s)}(z) dz\right| \\ 
        &\ge \left|\int_{z_n^{(s)}}^{z_n^{(s)}+\delta_s} f^{(s)}(z) dz \right| - \left| f^{(s-1)}(z_n^{(s)})\right| \\ 
        &\ge C_s \delta_s - \frac{C_s \delta_s}{3} \\
        &= \frac{2 C_s \delta_s}{3}. 
    \end{align*}
    Therefore, we can set $C_{s-1} := \frac{1}{2} \varlimsup_{z \to \infty} |f^{(s-1)}(z)|$, which is strictly positive according to our proof above. Now $f^{(s)}$ is (uniformly) bounded, so there must be some $\delta_{s-1} > 0$ and a sequence $\{z_n^{(s-1)}\}_{n=1}^\infty$ diverging to $\infty$, such that $|f^{(s-1)}(z)| > C_{s-1}$ whenever $z \in (z_n^{(s-1)}, z_n^{(s-1)} + \delta_{s-1})$. This completes the induction step. \\
\end{proof}

\begin{prop}\label{Prop Lin ind two-layer neurons with bias}
    Given $d, m \in \bN$ and a continuous $\sigma: \bR \to \bR$. The following results hold.  
    \begin{itemize}
        \item [(a)] Suppose there are some $s \in \bN$, $p \in \bR \cut \{0\}$ and $c \in \bR$ with $\sigma^{(s)} = f^p + c$, where $\limftyz f(z) = \lim_{z\to-\infty} f(z) = \infty$ both at hyper-exponential rate, and for any $w > 0$ and $\tw < 0$ either $f(wz) = o(f(\tw z))$ as $z \to \infty$ or $f(\tw z) = o(f(wz))$ as $z \to \infty$. Then $\{\sigma(w_kz + b_k)\}_{k=1}^m$ are linearly independent if and only if $(w_k, b_k) \ne (w_j, b_j)$ for all distinct $k, j \in \{1, ..., m\}$. 

        \item [(b)] Suppose there are some $s \in \bN$, $p \in \bR \cut \{0\}$ and $c \in \bR$ with $\sigma^{(s)} = f^p + c$, where $\limftyz f(z) = \lim_{z\to-\infty} f(z) = \infty$ both at hyper-exponential rate, and for a sequence $\{z_n\}_{n=1}^\infty$ with $\limftyn z_n = \infty$ we have $\limftyn \frac{f(z_n)}{f(-z_n)} \in \bR \cut \{0\}$. Then $\{\sigma(w_kz + b_k)\}_{k=1}^m$ are linearly independent if $(w_k, b_k) \pm (w_j, b_j) \ne 0$ for all distinct $k, j \in \{1, ..., m\}$. 

        \item [(c)] Suppose there is some $s \in \bN$ such that $\sigma^{(s)}$ is an even function with no zeros, $\sigma^{(s)} \in \sF$, and its Fourier transform has countably many zeros. Then $\{\sigma(w_kz+b_k)\}_{k=1}^m$ are linearly independent if and only if $(w_k, b_k) \pm (w_j, b_j) \ne 0$ for all distinct $k, j \in \{1, ..., m\}$, and when $\sigma(0) \ne 0$, we have $w_k = 0$ for at most one $k$. 
    \end{itemize}
\end{prop}
\begin{proof}
    By Lemma \ref{Lem Dimension reduction}, it suffices to work with $d = 1$. Moreover, when $d = 1$, if $a_1, ..., a_m$ are constants such that $\sum_{k=1}^m a_k \sigma(w_kz+b_k) \equiv 0$, then clearly $\sum_{k=1}^m [a_k w_k^s] \sigma^{(s)}(w_kz + b_k) \equiv 0$. Thus, we only need to show that $\{\sigma^{(s)}(w_kz + b_k)\}_{k=1}^m$ with $d = 1$ are linearly independent. 
    
    \begin{itemize}
        \item [(a)] Let $\{(w_k, b_k)\}_{k=1}^m$ be distinct. First assume that $p > 0$. Then $\sigma^{(s)}(z) = \Theta(f^p(z))$ as $z \to \pm\infty$. It follows that 
        \[
            \limftyz \sigma^{(s)}(z) = \lim_{z\to-\infty} \sigma^{(s)}(z) = \infty 
        \]
        at hyper-exponential rate and for any $w > 0$ and $\tw < 0$, either $\sigma^{(s)}(wz) = o(\sigma^{(s)}(\tw z))$ or $\sigma^{(s)}(\tw z) = o(\sigma^{(s)}(wz))$ as $z \to \infty$. Thus, applying Proposition \ref{Prop Functions of ordered growth II} to $\sigma^{(s)}$ and the identity map, we can see that $\{\sigma^{(s)}(w_kz+b_k)\}_{k=1}^m$ have ordered growth, whence by Proposition \ref{Prop Ordered growth implies linear independence}, they are linearly independent. \\

        Then assume that $p < 0$. So 
        \[
            \limftyz \sigma^{(s)}(z) = \lim_{z\to-\infty} \sigma^{(s)}(z) = c
        \]
        and $\frac{1}{\sigma^{(s)}(z) - c} = \Theta(f^{-p})$ as $z \to \pm\infty$ (the hyper-exponential growth of $f$ implies that $\frac{1}{\sigma^{(s)}(z) - c}$ is well-defined whenever $|z|$ is large). Again, by Proposition \ref{Prop Ordered growth implies linear independence}, $\{\sigma(w_kz + b_k) - c\}_{k=1}^m$ are linearly independent. On the other hand, if $a_1, ..., a_m \in \bR$ are constants such that $\sum_{k=1}^m a_k \sigma^{(s)}(w_kz + b_k) \equiv 0$, then 
        $\limftyz \sum_{k=1}^m a_k \sigma^{(s)}(w_kz + b_k) = \sum_{k=1}^m a_k c$ and thus we must have 
        \[
            \sum_{k=1}^m a_k \left( \sigma^{(s)}(w_kz + b_k) - c \right) \equiv 0. 
        \]
        By linear independence of $\{\sigma^{(s)}(w_kz + b_k)\}_{k=1}^m$, $a_1 = ... = a_m = 0$. It follows that the functions $\{\sigma^{(s)}(w_kz + b_k)\}_{k=1}^m$ are linearly independent. \\ 

        Conversely, if $(w_k, b_j) = (w_j, b_j)$ for some distinct $k, j \in \{1, ..., m\}$, then $\sigma(w_kz + b_k) = \sigma(w_jz + b_j)$, so $\{\sigma(w_kz + b_k)\}_{k=1}^m$ cannot be linearly independent. 

        \item [(b)] The proof is almost identical to that for (a); so we only show the case for $p > 0$. Let $\{(w_k, b_k)\}_{k=1}^m$ be distinct. Again, we have $\sigma^{(s)}(z) = \Theta(f^p(z))$ as $z \to \pm\infty$, whence 
        \[
            \limftyz \sigma^{(s)}(z) = \lim_{z\to-\infty} \sigma^{(s)}(z) = \infty 
        \]
        at hyper-exponential rate and 
        \[
            L := \limftyn \frac{\sigma^{(s)}(z_n)}{\sigma^{(s)}(-z_n)} \in \bR\cut\{0\} 
        \]
        Fix distinct $(w, b), (\tw, \tb) \in \bR^2$. Consider the following cases. 
        \begin{itemize}
            \item [i)] $w, \tw$ are both positive or negative. Then clearly $wz + b$ and $\tw z + \tb$ diverge to $\infty$ or $-\infty$ simultaneously. If $|w| > |\tw|$, then $\limftyz |(wz + b) - (\tw z + \tb)| = \infty$. Therefore, $\sigma^{(s)}(\tw z + \tb) = o(\sigma^{(s)}(wz + b))$ as $z \to \infty$. Similarly, if $|w| < |\tw|$ then we have $\sigma^{(s)}(wz + b) = o(\sigma^{(s)}(\tw z + \tb))$ as $z \to \infty$. If $w = \tw$ then we must have $b \ne \tb$, which means $|(wz + b) - (\tw z + \tb)| = |b - \tb|$ for all (large) $z$. Again, $\sigma^{(s)}(wz + b)$ and $\sigma^{(s)}(\tw z + \tb)$ have ordered growth. 

            \item [ii)] $w > 0$ and $\tw < 0$. By i), the functions $\sigma^{(s)}(wz + b)$ and $\sigma^{(s)}(-\tw z - \tb)$ have ordered growth, say $\sigma^{(s)}(-\tw z - \tb) = o(\sigma^{(s)}(wz + b))$ as $z \to \infty$. Define a sequence $\{\tz_n\}_{n=1}^\infty$ such that for each $n \in \bN$ we have $-\tw \tz_n - \tb = z_n$. Then $\limftyn \tz_n = \infty$ and  
            \[
                \limftyn \frac{\sigma^{(s)}(-\tw \tz_n - \tb)}{\sigma^{(s)}(\tw \tz_n + \tb)} = L. 
            \]
            This, together with the ordered growth of $\sigma^{(s)}(-\tw z - \tb)$ and $\sigma^{(s)}(wz + b)$, we can see that for sufficiently large $n$, 
            \[
                \left|\sigma^{(s)}(\tw \tz_n + \tb)\right| \le \frac{1}{L} \left|\sigma^{(s)}(-\tw \tz_n - \tb) \right| = o\left(\sigma^{(s)}(w\tz_n + b)\right)
            \]
            as $n \to \infty$ (or equivalently, as $z_n \to \infty$). 

            \item [iii)]$w < 0$ and $\tw > 0$. Argue in the same way we can show that $\sigma^{(s)}(wz+b)$ and $\sigma^{(s)}(\tw z + \tb)$ have ordered growth. 
        \end{itemize}
        Since $(w, b)$ and $(\tw, \tb)$ are arbitrary, the functions $\{\sigma^{(s)}(w_kz + b_k)\}_{k=1}^m$ have ordered growth. Therefore, by Proposition \ref{Prop Ordered growth implies linear independence}, they are linearly independent. Conversely, if $(w_k, b_j) = (w_j, b_j)$ for some distinct $k, j \in \{1, ..., m\}$, then $\sigma(w_kz + b_k) = \sigma(w_jz + b_j)$, so the functions $\{\sigma(w_kz + b_k)\}_{k=1}^m$ cannot be linearly independent. 

        \item [(c)] We will use a different way to prove this result. First, if $(w_k, b_k) \pm (w_j, b_j) = 0$ for some distinct $k, j \in \{1, ..., m\}$, then because $\sigma^{(s)}$ is even, we must have $\sigma^{(s)}(w_kz + b_k) = \sigma^{(s)}(w_j z + b_j)$. Similarly, if $w_k = w_j = 0$ for some distinct $k, j \in \{1, ..., m\}$, we obtain two constant functions $\sigma^{(s)}(b_k)$ and $\sigma^{(s)}(b_j)$. In either case, the functions are not linearly independent. \\ 
        
        For the converse, assume that we are given such $(w_k, b_k)$'s. Since $\sigma^{(s)}$ is an even function, $\sigma^{(s)}(wz + b) = \sigma^{(s)}(-wz - b)$ for all $(w, b) \in \bR^2$, thus we only need to consider the $(w_k, b_k)$'s with the following structures:
        \begin{itemize}
            \item [i)]  There are $0 = m_0 < m_1 < ... < m_r = m$ such that $w_{m_t} > w_{m_{t+1}} > 0$ for any $1 \le t < r$ and $w_k = w_{m_t}$ for any $m_{t-1} < k \le m_t$. 
            \item [ii)] We have $m_r = m$ or $m_r = m-1$. 
            \item [iii)] For any $1 \le t \le r$, $b_{m_{t-1}+1} > ... > b_{m_t}$. 
        \end{itemize}
        Since $\sigma^{(s)} \in \sF$, its Fourier transform is well-defined. In particular, for any $1 \le k \le m$ with $w_k \ne 0$ the Fourier transform of $\sigma^{(s)}(w_kz + b_k)$ can be computed by  
        \[
            \int_{-\infty}^\infty \sigma^{(s)}(w_kz + b_k) e^{-i\xi z} dz = e^{i \frac{b_k}{w_k} \xi} \left( \sigma^{(s)} \right)\hat{}\left(\frac{\xi}{w_k} \right). 
        \]
        Let $f_k(\xi) := \left(\sigma^{(s)}\right)\hat{}\left(\frac{\xi}{w_k}\right)$ for each $1 \le k \le m$. We now show that $\sum_{k=1}^{m_r} a_k \sigma^{(s)}(w_kz + b_k)$ is constant (which must be zero from the proof of Lemma \ref{Lem FT decay}) if and only $a_1 = ... = a_{m_r} = 0$. In this case, its Fourier transform takes the form 
        \[
            \sum_{t=1}^{r} \left( \sum_{k=m_{t-1}+1}^{m_t} a_k e^{i\frac{b_k}{w_{m_t}} \xi} \right) f_k(\xi) \equiv 0. 
        \]
        Without loss of generality, we may assume that $a_1, ..., a_{m_r}$ are non-zero (otherwise simply reduce the number of terms). By ii) above, for any $1 \le t \le r$, 
        \[
        \frac{b_{m_{t-1}+1}}{w_{m_t}}, ..., \frac{b_{m_t}}{w_{m_t}}
        \]
        are distinct. By hypothesis on $\sigma^{(s)}$, $f_{m_r}^{-1}(0)$ is countable. These, together with Lemma \ref{Lem Lin Comb of Tri Func are bounded below} , give us some constant $C > 0$ and a sequence $\{\xi_n\}_{n=1}^\infty$ with $\limftyn \xi_n = \infty$, $f_{m_r}(\xi_n) \ne 0$, and
        \[
            \left|\sum_{k=m_{r-1}+1}^{m_r} a_k e^{i \frac{b_k}{w_{m_r}} \xi_n }\right| \ge C \quad \forall\, n \in \bN. 
        \]
        By Lemma \ref{Lem FT decay}, the functions $f_{m_1}, ..., f_{m_r}$ satisfy $f_{m_t}(\xi) = o(f_{m_r}(\xi))$ as $\xi \to \infty$, for all $1 \le t < r$. But then 
        \[
            a_{m_r} = - \limftyn \sum_{t=1}^{r-1} \left( \frac{\sum_{k=m_{t-1}+1}^{m_t} a_k e^{i \frac{b_k}{w_{m_t}} \xi_n }}{\sum_{k=m_{r-1}+1}^{m_r} a_k e^{i \frac{b_k}{w_{m_t}} \xi_n }} \right) \frac{f_{m_t}(\xi_n)}{f_{m_r}(\xi_n)} = 0, 
        \]
        contradicting $a_1, ..., a_{m_r} \ne 0$. \\

        By our proof above, if $m_r = m$, we immediately see that $\{\sigma^{(s)}(w_kz + b_k)\}_{k=1}^m$ are linearly independent. If $m_r = m - 1$, then $w_m = 0$ and we have two cases. First, if $\sigma(w_mz - b_m) \equiv 0$ then the neurons are trivially linearly dependent. Otherwise, assume there are $a_1, ..., a_m \in \bR$ with $\sum_{k=1}^m a_k \sigma(w_kz + b_k) \equiv 0$. Then $\sum_{k=1}^m a_k \sigma^{(s)}(w_kz + b_k) \equiv 0$ and thus $\sum_{k=1}^{m_r} a_k \sigma^{(s)}(w_kz + b_k)$ must be constant, whence $a_1 = ... = a_{m_r} = 0$ by our proof above. But then $a_m = 0$ as well, showing that $\{\sigma(w_kz + b_k)\}_{k=1}^m$ are linearly independent. 
    \end{itemize}
\end{proof}

Below we give examples of activation functions that satisfy the requirements of Proposition \ref{Prop Lin ind two-layer neurons with bias}. 
\begin{itemize}
    \item [(a)] Consider $f(z) = e^{z^q} + e^{-z^r}$ for some distinct odd numbers $q,r \in \bN\cut \{1\}$. It is clear that $\limftyz f(z) = \lim_{z\to-\infty} f(z) = \infty$ at hyper-exponential rate. Thus, by Proposition \ref{Prop Lin ind two-layer neurons with bias} (a), $\{f(w_k z + b_k)\}_{k=1}^m$ are linearly independent if and only if $(w_k, b_k) \ne (w_j, b_j)$ for all distinct $k, j \in \{1, ..., m\}$. By Proposition \ref{Prop Lin ind two-layer neurons with bias} (a) again, this also holds for $\sigma(z) = f^p + c$, where $p \in \bR\cut\{0\}$ and $c \in \bR$. 

    \item [(b)] Consider $f(z) = e^{z^r}$ for some even number $r \in \bN$. Clearly, $\limftyz f(z) = \lim_{z\to-\infty} f(z) = \infty$ at hyper-exponential rate. Since $f$ is an even function and has no real roots, $\frac{f(z)}{f(-z)} = 1$ for all $z \in \bR$. Thus, by Proposition \ref{Prop Lin ind two-layer neurons with bias} (b), $\{f(w_k z + b_k)\}_{k=1}^m$ are linearly independent if $(w_k, b_k) \pm (w_j, b_j) \ne 0$ for all distinct $k, j \in \{1, ..., m\}$ (in fact this is also a necessary condition, as $f(wz + b) = f(-wz - b)$). By Proposition \ref{Prop Lin ind two-layer neurons with bias} (b) again, this also holds for $\sigma(z) = f^p + c$, where $p \in \bR\cut\{0\}$ and $c \in \bR$. 
    
    \item [(c)] Let $\sigma \in \calA(\bR)$ and there is some $s \in \bN$ such that $\sigma^{(s)}$ is an even function, $\sigma^{(s)} \in \sF$ and $\sigma^{(s)}(z) = O(e^{-\lambda |z|})$ as $z \to \pm\infty$ for some $\lambda > 0$. Then for any $\xi, \xi_0 \in \bR$ with $|\xi - \xi_0| < \lambda$, we have 
    \begin{align*}
        |\Hat{\sigma^{(s)}}(\xi)|
        &= \left| \int_{-\infty}^\infty \sigma^{(s)}(z) e^{-i(\xi - \xi_0)z} e^{-i\xi_0 z} dz \right| \\
        &\le \left| \int_{-\infty}^\infty \sigma^{(s)}(z) \sum_{s=0}^\infty \frac{1}{s!}(-i(\xi - \xi_0) z)^s e^{-i\xi_0 z} dz \right| \\
        &\le \sum_{s=0}^\infty \frac{1}{s!} |\xi - \xi_0|^s \int_{-\infty}^\infty |\sigma^{(s)}(z)||z|^s dz \\ 
        &\le C \sum_{s=0}^\infty \frac{1}{s!} |\xi - \xi_0|^s \int_{-\infty}^\infty e^{-\lambda |z|} |z|^s dz \\ 
        &= C\lambda \sum_{s=0}^\infty \frac{1}{s!} (\lambda^s s!) |\xi - \xi_0|^s \\
        &< \infty. 
    \end{align*}
    Here $C > 0$ is any constant such that $|\sigma^{(s)}(z)| \le C e^{-\lambda |z|}$ for all $z \in \bR$. Thus, for any $\xi_0 \in \bR$, $\Hat{\sigma^{(s)}}$ has a series expansion at $\xi_0$ given by 
    \[
        \Hat{\sigma^{(s)}}(\xi) = \sum_{s=0}^\infty \frac{(\xi-\xi_0)^s}{s!} \int_{-\infty}^\infty \sigma^{(s)}(z) (-iz)^s e^{-i\xi_0 z}dz, \quad \xi \in \left(\xi_0 - \frac{1}{\lambda}, \xi_0 + \frac{1}{\lambda}\right). 
    \]
    It follows that $\Hat{\sigma^{(s)}}$ is analytic, so it has countable zeros and thus $\sigma$ satisfies the requirements in Proposition \ref{Prop Lin ind two-layer neurons with bias} (c). 

    \item [(d)] We now use (c) to show that a class of commonly used activation functions all satisfy the requirements in Proposition \ref{Prop Lin ind two-layer neurons with bias} (c), whence the results in (c) holds. Indeed, if $f$ is Sigmoid function, then $f' = f(1-f)$ is an even function, $f' \in \sF$ and $f'(z) = O(e^{-|z|})$ as $z \to \pm\infty$; in fact, $f^{(s)}(z) = O(e^{-|z|})$ for all $s \in \bN$. Thus, any ``$f'$-related" activation functions satisfies the requirements for Proposition \ref{Prop Lin ind two-layer neurons with bias}. This include: $\sigma(z) = \frac{1}{1 + e^{-x}}$ (Sigmoid), $\sigma(z) = \frac{e^x - e^{-x}}{e^x + e^{-x}}$ (Tanh), $\sigma(x) = \log\left( 1 + e^x \right)$, $\sigma(x) = \frac{x}{1 + e^{-x}}$ (Swish). Indeed, in the first three examples we have $\sigma' = f'$ or $\sigma'' = f'$, and in the last example we have 
    \[
        \frac{d^2}{dx^2} \left(\frac{x}{1 + e^{-x}}\right) = x f'' + f', 
    \]
    so that, by properties of $f'$, we also have $\sigma''$ is even, $\sigma'' \in \sF$ and $\sigma''(z) = O(e^{-|z|})$ as $z \to \pm\infty$. 
\end{itemize}

\subsection{Three-layer Neurons with Sigmoid and Tanh Activation} 

\begin{lemma}\label{Lem Blow-up of two-layer Sigmoid neurons}
    Let $\sigma$ be Sigmoid activation. Given $m \in \bN$ and dstinct $\{(w_k, b_k) \in \bR^2 \}_{k=1}^m$ such that $w_1 \ge ... \ge w_m > 0$, there are curves $\gamma_1^+, ..., \gamma_m^+: [0, \infty) \to \bC$ with the following properties. 
    \begin{itemize}
        \item [(a)] For each $1 \le k \le m$, $\limftyt \gamma_k^+(t) \in \bC$. 
        \item [(b)] For each $1 \le k \le m$, $\sigma\left( w_k \gamma_k^+(t) + b_k \right) \in \bR$ for all $t \ge 0$. 
        \item [(c)] For each $1 \le k \le m$, we have $\limftyt \sigma\left( w_k \gamma_k^+ (t) + b_k \right) = \infty$ and $\sigma\left( w_{k'} z + b_{k'} \right)$ is analytic on a neighborhood of $\limftyt \gamma^+(t)$ whenever $k' > k$. 
    \end{itemize}
    Similarly, there are curves $\gamma_1^-, ..., \gamma_m^-: [0, \infty) \to \bC$ with properties (a), (b), and $\limftyt \sigma(w_k \gamma_k^-(t) + b_k) = -\infty$ and  $\sigma\left( w_{k'} z + b_{k'} \right)$ is analytic on a neighborhood of $\limftyt \gamma^-(t)$ whenever $k' > k$. 
\end{lemma}
\begin{proof}
    We only show how to find $\gamma_1^+$ and $\gamma_1^-$; the construction of the other $\gamma_k^+, \gamma_k^-$'s are similar. Note that we can view $\sigma$ as a meromorphic function on $\bC$; in particular, it has  (simple) poles at $\{i(2q + 1)\pi: q \in \bZ\}$. This means for each $1 \le k \le m$, the neuron $\sigma(w_k z + b_k)$ is also meromorphic on $\bC$ and has poles at $\{\frac{
    i(2q + 1)\pi - b_k}{w_k}: q \in \bZ\}$. \\ 
    
    Since $w_1 \ge w_k$ for all $k$, $\im \frac{i\pi}{w_1}$ achieves the minimum of the set $\{ \im \frac{i(2q + 1)\pi}{w_k}: q \in \bN \cup \{0\}, 1 \le k \le m\}$. Meanwhile, if $w_1 = w_k$ for some $k$, we must have $b_1 \ne b_k$ (because $(w_1, b_1)$ and $(w_k, b_k)$ are distinct); thus the real parts of the poles $\frac{i\pi - b_1}{w_1}$ and $\frac{i(2q + 1)\pi - b_1}{w_1}$ are distinct. Consider the curve
    \[
        \Tilde{\gamma}_1^+: [0, \infty) \to \bC, \quad \Tilde{\gamma}_1^+(t) = \frac{1}{w_1} \left( -\frac{1}{t + 1} - b_1 + i\pi \right)
    \]
    and the curve 
    \[
        \Tilde{\gamma}_1^-: [0, \infty) \to \bC, \quad \Tilde{\gamma}_1^-(t) = \frac{1}{w_1} \left( \frac{1}{t + 1} - b_1 + i\pi \right)
    \]
    It is not difficult to deduce that 
    \begin{itemize}
        \item [(a)] $\limftyt \Tilde{\gamma}_1^+(t) = \frac{i\pi - b_1}{w_1} \in \bC$ and $\limftyt \Tilde{\gamma}_1^-(t) = \frac{i\pi - b_1}{w_1} \in \bC$
        \item [(b)] $\sigma(w_1 \Tilde{\gamma}_1^+(t) + b_1) = \frac{1}{1 - e^{-1/(t+1)}} \in \bR$ and $\sigma(w_1 \Tilde{\gamma}_1^-(t) + b_1) = \frac{1}{1 - e^{1/(t+1)}} \in \bR$. 
        \item [(c)] $\limftyt \sigma(w_1 \Tilde{\gamma}_1^+(t) + b_1) = \infty$, and by the minimality of $\im \frac{i\pi}{w_1}$ mentioned above, there is some $T > 0$ such that for any $k > 1$, $\sigma(w_k \Tilde{\gamma}_1^+(t) + b_k)$ is analytic on the open ball 
        \[
            B\left( \frac{i\pi - b_1}{w_1}, \frac{1}{T} \right) \siq \bC. 
        \]
        Similarly, $\limftyt \sigma(w_1 \Tilde{\gamma}_1^-(t) + b_1) = \infty$ and there is some $T > 0$ such that for any $k > 1$, $\sigma(w_k \Tilde{\gamma}_1^+(t) + b_k)$ is analytic on this open ball. 
    \end{itemize}
    It follows that we can define $\gamma_1^+, \gamma_1^-: [0, \infty) \to \bC$ by $\gamma_1^+(t) = \Tilde{\gamma}_1^+(t + T)$ and $\gamma_1^-(t) = \Tilde{\gamma}_1^-(t + T)$. \\
\end{proof}

Note that while we can use Proposition \ref{Prop Lin ind two-layer neurons with bias} to deduce the linear independence of $\{\frac{1}{1 + e^{w_kz + b_k}}\}$, we can also directly use Lemma \ref{Lem Blow-up of two-layer Sigmoid neurons} to prove it, using the growth orders of these functions along $\gamma_k^+$ and $\gamma_k^-$. We omit the proof here. \\

\begin{lemma}\label{Lem Exp compose Sigmoid}
     Let $\Omega \siq \bF$ be open and $f_1, ..., f_m: \Omega \to \bF$ be meromorphic functions. Let $\gamma: [0, \infty) \to \Omega$ be a curve such that for $z^* := \limftyt \gamma(t)$, we have $f_1(\gamma(t)) \ge C|\gamma(t) - z^*|^{-q}$ for some $C, q > 0$ and $f_2, ..., f_m$ are analytic on a neighborhood of $\limftyt \gamma(t)$. Then for any $n \in \bN$ and any $w_1, ..., w_n \in \bR^m$, if $\sum_{j=1}^n a_j \exp(\sum_{k=1}^m w_{jk} f_k(z))$ is constant zero, the functions 
     \[
        \sum_{j': w_{j'1} = w_{j1}} a_{j'} \exp\left( \sum_{k=2}^m w_{j'k} f_k(z) \right), \quad 1 \le j \le n
     \]
     must all be constant zero. 
\end{lemma}
\begin{proof}
    By rearranging the indices (in $j$) and multiplying $\exp(-(\min_{1\le j \le n} w_{j1}) f_1)$ to each term in the sum, we can assume that 
    \[
        w_{11} \ge w_{12} \ge ... \ge w_{1n} = 0. 
    \]
    Also, denote $z^* := \limftyt \gamma(t)$. We shall prove the desired result by induction on number of different $w_{j1}$'s. The result clearly holds when $w_{11} = ... = w_{1n}$. Now assume that there is some $n' \in \{1, ..., n\}$ with $w_{11} = ... = w_{1n'} > w_{1(n'+1)} \ge ... \ge w_{1n} = 0$. Rewrite the sum as 
    \begin{equation}\label{eq 1 of Lem Exp compose Sigmoid}
    \begin{aligned}
        &\sum_{j=1}^n a_j \exp\left( \sum_{k=1}^m w_{jk} f_k(z) \right) \\ 
        =&\,\left[ \sum_{j=1}^{n'} a_j \exp\left( \sum_{k=2}^m w_{jk} f_k(z) \right) \right] e^{w_{11} f_1(z)} + \sum_{j>n'} a_j \exp\left( \sum_{k=1}^m w_{jk} f_k(z) \right). 
    \end{aligned}
    \end{equation}
    Assume that $\sum_{j=1}^{n'} a_j \exp\left( \sum_{k=2}^m w_{jk} f_k(z) \right)$ is not constant zero. Since $z^* \in \Omega$ and $f_2, ..., f_m$ are analytic near $z^*$, this sum is also analytic near $z^*$. Thus, there is some $s \in \bN$ with
    \[
        \sum_{j=1}^{n'} a_j \exp\left( \sum_{k=2}^m w_{jk} f_k(z) \right) = \Theta((z - z^*)^s)
    \]
    as $z \to z^*$, whence 
    \begin{align*}
        \left[ \sum_{j=1}^{n'} a_j \exp\left( \sum_{k=2}^m w_{jk} f_k(\gamma(t)) \right) \right] e^{w_{11} f_1(\gamma(t))}
        &= \Theta\left( (\gamma(t) - z^*)^s e^{w_{11} f_1(\gamma(t))} \right) \\
        &= \Omega\left( (\gamma(t) - z^*)^s \exp\left(\frac{2w_{11} + w_{(n'+1)1}}{3} f_1(\gamma(t))\right)\right) \\
        &= \Omega \left( \exp\left(\frac{w_{11} + w_{(n'+1)1}}{2} f_1(\gamma(t))\right) \right) \\
    \end{align*}
    as $t \to \infty$. Note that in the estimation above, the third equality holds because $(\gamma(t) - z^*)^s = o(\exp(w f_1(\gamma(t))))$ as $t \to \infty$ for any $w > 0$, which can be seen from 
    \[
        \delta(t)^s \ll e^{\frac{w}{\delta(t)^p}} = O\left( e^{w f_1(\gamma(t))} \right), \quad \delta(t) = |\gamma(t) - z^*|
    \]
    as $t \to \infty$. On the other hand, since $w_{j1} < w_{11}$ for all $j > n'$, we also have 
    \begin{align*}
        \sum_{j>n'} a_j \exp\left( \sum_{k=1}^m w_{jk} f_k(\gamma(t)) \right) 
        &= O\left( \sum_{j > n'} |a_j| e^{w_{j1} f_1(\gamma(t))} \right) \\ 
        &= O\left( \exp\left(w_{(n'+1)1} f_1(\gamma(t))\right) \right) \\ 
        &= o\left( \exp\left(\frac{w_{11} + w_{(n'+1)1}}{2} f_1(\gamma(t))\right)\right) 
    \end{align*}
    as $t \to \infty$. Combining the estimates for the two parts, we can see from equation (\ref{eq 1 of Lem Exp compose Sigmoid}) that $\sum_{j=1}^n a_j \exp(\sum_{k=1}^m w_{jk} f_k(\gamma(t))) \to \infty$ as $t \to \infty$, contradicting the hypothesis that it is constant zero. Therefore, we must have 
    \[
        \sum_{j=1}^{n'} a_j \exp\left( \sum_{k=2}^m w_{jk} f_k(z) \right) \equiv 0. 
    \]
    Then equation (\ref{eq 1 of Lem Exp compose Sigmoid}) becomes 
    \[
        \sum_{j=1}^n a_j \exp\left( \sum_{k=1}^m w_{jk} f_k(z) \right) = \sum_{j > n'} a_j \exp\left( \sum_{k=1}^m w_{jk} f_k(z) \right). 
    \]
    Thus, we can perform an induction to derive the final result. 
\end{proof}
\begin{remark}
    Suppose we have $m = 1$, and $f: \Omega \to \bF$ and $\gamma: [0, \infty) \to \Omega$ are given as in Lemma \ref{Lem Exp compose Sigmoid}. Then the functions $\{e^{w_j f(z)}\}_{j=1}^n$ are linearly independent if and only if $w_1, ..., w_n \in \bR$ are distinct. Now assume that there are $\gamma_1, ..., \gamma_m: [0, \infty) \to \Omega$ such that for each $k$, $f_k(\gamma_k(t)) \ge C |\gamma_k(t) - \limftyt \gamma_k(t)|^{-q_k}$ for some $C,q_k > 0$, and $f_{k'}$'s are analytic on a neighborhood of $\limftyt \gamma_k(t)$ for all $k' \ne k$. By arguing inductively using Lemma \ref{Lem Exp compose Sigmoid}, we can see that the function $\{\exp(\sum_{k=1}^m w_{jk} f_k(z))\}_{j=1}^n$ are linearly independent if and only if the vectors $\{w_j = (w_{j1}, ..., w_{jm})\}_{j=1}^n$ are distinct. 
\end{remark}

\begin{lemma}\label{Lem Main lem for Prop Lin ind of three-layer Sigmoid neurons}
    Let $n \in \bN$ and $\Omega \siq \bF$ be open. Let $f, \lambda_1, ..., \lambda_n: \Omega \to \bR$ be meromorphic functions such that there are two curves $\gamma^+, \gamma^-: [0, \infty) \to \Omega$ with the following properties:
    \begin{itemize}
        \item [(a)] $z_+^* := \limftyt \gamma^+(t), z_-^* := \limftyt \gamma^-(t) \in \Omega$. Moreover, there are constants $q_+, q_- > 0$ such that $f(\gamma^+(t)) = \Omega(|\gamma^+(t) - z_+^*|^{-q_+})$ and $f(\gamma^-(t)) = \Omega(|\gamma^-(t) - z_-^*|^{-q_-})$ as $t \to \infty$. 
        \item [(b)] $\lambda_1, ..., \lambda_n$ are analytic, non-zero near $\limftyt \gamma^+(t)$ and $\limftyt \gamma^-(t)$, and the functions $\{\Pi_{j'\ne j} \lambda_{j'}\}_{j=1}^n$ are linearly independent. 
    \end{itemize}
    Then for any $w \in \bR\cut\{0\}$, the function
    \begin{equation}\label{eq 1 of Lem Main lem for Prop Lin ind of three-layer Sigmoid neurons}
        \sum_{j=1}^n \frac{a_j}{1 + \lambda_j(z) \exp( w f(z))} 
    \end{equation}
    satisfies 
    \[
        \sum_{j=1}^n \frac{a_j}{1 + \lambda_j(\gamma(t)) \exp( w f(\gamma(t)))} = o\left( \exp(-|\tw f(\gamma(t))| \right)
    \]
    for some $\gamma \in \{\gamma^+, \gamma^-\}$ depending only on $w$ and any $|\tw| < |w|$. 
\end{lemma}
\begin{proof}
    The result obviously holds for $a_1 = ... = a_n = 0$. So suppose that the $a_j$'s are not all zero. Let $\gamma \in \{\gamma^+, \gamma^-\}$ be such that $w \gamma(t) \to \infty$; then let $z^* = \limftyt \gamma(t)$, so in particular $z^* \in \{z_+^*, z_-^*\}$. Rewrite (\ref{eq 1 of Lem Main lem for Prop Lin ind of three-layer Sigmoid neurons}) as 
    \begin{align*}
        \sum_{j=1}^n \frac{a_j}{1 + \lambda_j(z) \exp( w f(z))} 
        &= \frac{\sum_{j=1}^n a_j \Pi_{j'\ne j} \left(1 + \lambda_{j'}(z) e^{w f(z)}\right)}{\Pi_{j=1}^n \left(1 + \lambda_j(z) e^{w f(z)}\right)} \\ 
        &= \frac{\left[ \sum_{j=1}^n a_j \Pi_{j'\ne j} \lambda_{j'}(z) \right] e^{(n-1)w f(z)} + g_1(z)}{\left[\Pi_{j=1}^n \lambda_j(z)\right] e^{n w f(z)} + g_2(z)}. 
    \end{align*}
    where $g_1, g_2: \Omega \to \bF$ are functions which satisfy $g_1(\gamma(t)) = O(e^{(n-2)w f(\gamma(t)}))$ and $g_2(\gamma(t)) = O(e^{(n-1)w f(\gamma(t))})$ as $t \to \infty$. Since none of the $\lambda_j$'s vanish near $z^*$, so is $\Pi_{j=1}^n \lambda_j$. Since the functions $\{\Pi_{j'\ne j} \lambda_{j'}\}_{j=1}^n$ are linearly independent and since the $a_j$'s are not all zero, there is some $s \in \bN$ with 
    \[
        \sum_{j=1}^n a_j \Pi_{j'\ne j} \lambda_{j'}(z) = \Theta((z - z^*)^s)
    \]
    as $z \to z^*$. Meanwhile, by hypothesis (a), for any $\tw \in \bR$ with $|\tw| < |w|$ we have, as $t \to \infty$, $|\gamma(t) - z^*|^s \ll e^{(|w|-|\tw|)|\gamma(t) - z^*|^{-q}}$ and thus 
    \[
        |\gamma(t) - z^*|^s = o\left(e^{(|w|-|\tw|)|f(\gamma(t))|}\right). 
    \]
    It follows that 
    \begin{align*}
        \sum_{j=1}^n \frac{a_j}{1 + \lambda_j(\gamma(t)) \exp( w f(z))} 
        &= \frac{\Theta((\gamma(t) - z^*)^s) e^{(n-1)w f(\gamma(t))} + O(e^{(n-2)w f(\gamma(t))})}{\Theta(e^{nw f(\gamma(t))}) + O(e^{(n-1)w f(\gamma(t))})} \\ 
        &= \Theta\left( \frac{\Theta((\gamma(t) - z^*)^s) e^{(n-1)w f(\gamma(t))}}{\Theta(e^{nw f(\gamma(t))})} \right) \\ 
        &= \Theta\left(\frac{o\left(e^{(|w|-|\tw|)|f(\gamma(t))|}\right)}{\Theta(e^{|w||f(\gamma(t))|}}\right) \\
        &= o\left( e^{-|\tw||f(\gamma(t)|}\right) = o\left( e^{-|\tw f(\gamma(t))|}\right), 
    \end{align*}
    where the last equality holds because by our construction of $\gamma$, $|w||f(\gamma(t))| = |w f(\gamma(t))|$ as $t \to \infty$. \\
\end{proof}

\begin{prop}\label{Prop Lin ind of three-layer Sigmoid neurons}
    Given $d, m, n \in \bN$. Let $\{(w_k^{(1)}, b_k^{(1)})\}_{k=1}^m \siq \bR^{md + m}$ be such that $(w_{k_1}^{(1)}, b_{k_1}^{(1)}) \pm (w_{k_2}^{(1)}, b_{k_2}^{(1)}) \neq 0$ for all distinct $k_1, k_2 \in \{1, ..., m\}$ and $w_k^{(1)} \ne 0$ for all $k \in \{1, ..., m\}$. Let $\{(w_j^{(2)}, b_j^{(2)})\}_{j=1}^m \siq \bR^{nm + n}$ be such that $(w_{j_1}^{(2)}, b_{j_1}^{(2)}) \pm (w_{j_2}^{(2)}, b_{j_2}^{(2)}) \ne 0$ for all distinct $j_1, j_2 \in \{1, ..., n\}$ and $w_j^{(2)} \ne 0$ for all $j \in \{1, ..., n\}$. Then for $\sigma$ being a Sigmoid or Tanh activation function, the three-layer neurons 
    \[
        \left\{\sigma\left( \sum_{k=1}^m w_{jk}^{(2)} \sigma\left( w_k^{(1)}z + b_k^{(1)} \right) + b_j^{(2)} \right)\right\}_{j=1}^n
    \]
    are linearly independent. 
\end{prop}
\begin{proof}
     Because $\tanh(z) = \frac{2}{1 + e^{-2z}} - 1$, we only need to prove the desired result for Sigmoid activation function. By Lemma \ref{Lem Dimension reduction}, assume that $\{w_k^{(1)}\}_{k=1}^n \subseteq \bR$. Define for each $1 \le k \le m$ 
     \[
        f_k(z) := \sigma\left( w_k^{(1)}z + b_k^{(1)} \right)
     \]
     and for each $1 \le j \le n$ 
     \begin{align*}
        H(\theta_j, z) 
        &= \sigma\left( \sum_{k=1}^m w_{jk}^{(2)} f_k(z) + b_j^{(2)} \right) \\ 
        &= \frac{1}{1 + \exp\left( \sum_{k=1}^m w_{jk}^{(2)} f_k(z) + b_j^{(2)} \right)}, \quad \theta_j = (w_j^{(2)}, b_j^{(2)}). 
     \end{align*}
    By rearranging the indices of the $w_{j1}$'s, we may assume that $w_{11}^{(2)} \ge ... \ge w_{n1}^{(2)}$. Let $\gamma_1: [0, \infty) \to \bC$ be a curve with limit $z^* \in \bC$ such that the requirements (a), (b) and (c) in Lemma \ref{Lem Blow-up of two-layer Sigmoid neurons}, and moreover $\limftyt w_{11}^{(2)} \gamma_1(t) = \infty$. \\

    With these notations and assumptions, we then show that $\sum_{j=1}^n a_j H(\theta_j, \cdot) \equiv 0$ implies $a_1 = ... = a_n = 0$, thus showing that $\{H(\theta_j, \cdot)\}_{j=1}^n$ are linearly independent. First consider $n = 1$ and any $m \in \bN$. Since 
    \[
        \limftyt \exp\left( \sum_{k=1}^m w_{1k}^{(2)} f_k(\gamma(t)) + b_1^{(2)} \right) = \exp\left( \sum_{k=2}^m w_{1k}^{(2)}f_k(z^*) + b_1^{(2)} \right) \limftyt e^{w_{11}^{(2)} f_1(\gamma(t))} = \infty, 
    \]
    $H(\theta_1, \cdot)$ is clearly not a constant (zero) function. Fix $n \ge 2$. Suppose that the desired result holds for all $n' < n$ and all $m \in \bN$. There is a (unique) $n' \in \{1, ..., n\}$ with $w_{11}^{(2)} = ... = w_{n'1}^{(2)} > w_{(n'+1)1}^{(2)} \ge ... \ge w_{n1}^{(2)}$. Thus, we can rewrite $\sum_{j=1}^n a_j H(\theta_j, z)$ as 
    \begin{align*}
        \sum_{j=1}^n a_j H(\theta_j, z) 
        &= \sum_{j=1}^{n'} \frac{a_j}{1 + \exp\left( \sum_{k=2}^m w_{jk}^{(2)} f_k(z) + b_j^{(2)}\right) e^{w_{j1}^{(2)} f_1(z)}} + \sum_{j > n'} a_j H(\theta_j, z) \\ 
        &= \sum_{j=1}^{n'} \frac{a_j}{1 + \lambda_j(z) e^{w_{j1}^{(2)} f_1(z)}} + \sum_{j > n'} a_j H(\theta_j, z), 
    \end{align*}
    where we define $\lambda_j(z) = \exp( \sum_{k=2}^m w_{jk}^{(2)} f_k(z) + b_j^{(2)})$ for each $1 \le j \le n'$. By Lemma \ref{Lem Exp compose Sigmoid} and its remark, the functions $\{\Pi_{j'\ne j} \lambda_j\}_{j=1}^n$ are linearly independent. Thus, when $a_1, ..., a_{n'}$ are not all-zero we can use Lemma \ref{Lem Main lem for Prop Lin ind of three-layer Sigmoid neurons} to obtain
    \[
        \sum_{j=1}^{n'} \frac{a_j}{1 + \lambda_j(\gamma(t)) e^{w_{j1}^{(2)} f_1(\gamma(t))}} = o\left( e^{-|\tw f_1(\gamma(t))|} \right)
    \]
    as $t \to \infty$, where $\tw \in \bR$ satisfies $\min_{n' < j \le n} |w_{j1}^{(2)}| < |\tw| < |w_{11}^{(2)}|$. On the other hand, we have 
    \[
        \sum_{j > n'} a_j H(\theta_j, \gamma(t)) = \Omega\left( e^{-|\tw f_1(\gamma(t))|} \right)
    \]
    as $t \to \infty$. It follows that when $a_1, ..., a_{n'}$ are not all-zero, the function $\sum_{j=1}^n a_j H(\theta_j, \cdot)$ cannot be constant zero. Thus, $a_1 = ... = a_{n'} = 0$ and we can rewrite $\sum_{j=1}^n a_j H(\theta_j, z) = \sum_{j > n'} a_j H(\theta_j, z) = 0$ for all $z$. By induction hypothesis, $a_{n'+1} = ... = a_n = 0$, showing that $\{H(\theta_j, \cdot)\}_{j=1}^n$ are linearly independent. \\
\end{proof}

\begin{cor}[Necessary and sufficient conditions for three-layer Tanh neurons]
    Given $d, m, n \in \bN$. Fix any $(w_k^{(1)})_{k=1}^m \in \bR^{dm + m}$. When $\sigma$ is Tanh activation, the three-layer neurons without bias 
    \[
        \left\{\sigma\left( \sum_{k=1}^m w_{jk}^{(2)} \sigma\left( w_k^{(1)}z \right) \right)\right\}_{j=1}^n,  
    \]
    are linearly independent if and only if for all distinct $j_1, j_2 \in \{1, ..., n\}$, 
    \[
        \left( \sum_{k':w_{k'}^{(1)} \pm w_k^{(1)} = 0} \text{Sgn}\left(\frac{w_{k'}}{w_k}\right) w_{j_1 k'}^{(2)} \right)_{k=1}^m \pm \left( \sum_{k':w_{k'}^{(1)} \pm w_k^{(1)} = 0} \text{Sgn}\left(\frac{w_{k'}}{w_k}\right) w_{j_2 k'}^{(2)}\right)_{k=1}^m \ne 0
    \]
    and for all $j \in \{1, ..., n\}$, 
    \[
        \left( \sum_{k':w_{k'}^{(1)} \pm w_k^{(1)} = 0} \text{Sgn}\left(\frac{w_{k'}}{w_k}\right) w_{j k'}^{(2)} \right)_{k=1}^m \ne 0. 
    \]
    Here we define $\text{Sgn}(u) = 1$ if $u > 0$ and $\text{Sgn}(u) = -1$ otherwise, and for two vectors $w, \tw \in \bR^d$ such that $\Spans{w} = \Spans{\tw}$ we define $\frac{w}{\tw}$ as the number $\lambda \in \bR$ with $w = \lambda \tw$. 
\end{cor}

\section{Discussion and Conclusion}

The article delves into the intricate relationship between the linear independence of neurons and their parameters, offering insightful contributions that shed light on this fundamental aspect. We make a comprehensive exploration of neurons with diverse structures. First, in Section \ref{Section Theory of general neurons} we study the structure of the set $\tilde{\calZ}(\{m_l\}_{l=1}^L)$ of parameters making them linearly dependent, and in particular, to what extent can Embedding Principle characterize it. Specifically, we find specific activation function for which $\tilde{\calZ}(\{m_l\}_{l=1}^L)$ is minimal, and Embedding Principle fully characterize it.  Based on this, we extend our investigation to neurons without bias, demonstrating that for any activation $\sigma$ vanishing at the origin, we can slightly ``perturb it" so that the linear independence of corresponding neurons can still be locally described by the Embedding Principle. This nuanced analysis highlights the adaptability and versatility of neural networks in accommodating different activation functions while maintaining linear independence. Further research may focus on extending this result to the whole parameter space. \\ 

Besides, in Section \ref{Section Misc} we provide a thorough treatment of the linear independence for both two-layer and three-layer neurons. For two-layer neurons, we identify three classes of activation function and fully describe how the parameters of these neurons affect their linear independence. These classes of activations include commonly used ones such as Sigmoid, Tanh, etc. Meanwhile, we show by example that for two-layer neurons with bias or three-layer neurons, and with generic activations (Definition \ref{Defn Generic activation}), the relationship between linear independence of neurons and parameters could be very complicated, contrary to the previous studies about such activations. For three-layer neurons, we focus on ``Sigmoid-type" activation and also establish complete description of the relationship between linear independence of neurons and parameters. \\

Finally, we note that our paper introduces innovative methodologies to deep learning research, leveraging tools from analytic function theory, complex analysis, and Fourier analysis and developing novel, systematic proof strategies. We hope that such tools and ideas would give fresh insights and advancements for future studies of neural networks (and in general other machine learning models).

\clearpage
\bibliography{main}

\end{document}